\documentclass[twoside,11pt]{article}

\PassOptionsToPackage{numbers,square,comma,sort&compress}{natbib}

\usepackage[abbrvbib, preprint]{jmlr2e}

\usepackage{amsmath,amsfonts,bm}

\def\eqref#1{equation~\ref{#1}}
\def\Eqref#1{Equation~\ref{#1}}

\def\1{\bm{1}}

\DeclareMathAlphabet{\mathsfit}{\encodingdefault}{\sfdefault}{m}{sl}
\SetMathAlphabet{\mathsfit}{bold}{\encodingdefault}{\sfdefault}{bx}{n}

\usepackage[capitalise,noabbrev]{cleveref}
\crefname{assumption}{Assumption}{Assumptions}
\usepackage{url}
\usepackage{booktabs}

\crefalias{lemmaapp}{lemma}

\usepackage{enumitem}
\usepackage{float}

\usepackage{thmtools}
\usepackage{thm-restate}

\newtheorem{assumption}{Assumption}

\usepackage{algorithm}
\usepackage{algpseudocode}
\usepackage{xcolor}
\algrenewcommand\algorithmicrequire{\textbf{Input:}}
\algrenewcommand\algorithmicensure{\textbf{Output:}}
\algrenewcommand\algorithmiccomment[1]{\hfill{\color{gray}$\triangleright$~#1}}

\usepackage{adjustbox}
\newcommand{\res}[2]{$#1\,{\scriptstyle \pm #2}$}
\newcommand{\resb}[2]{$\mathbf{#1}\,{\scriptstyle \pm #2}$}
\newcommand{\resu}[2]{$\underline{#1}\,{\scriptstyle \pm #2}$}
\newcommand{\resbu}[2]{$\underline{\mathbf{#1}}\,{\scriptstyle \pm #2}$}
\newcommand{\gain}[2]{$#1$\,{\scriptsize$(#2\%)$}}
\newcommand{\gainb}[2]{$\mathbf{#1}$\,{\scriptsize$(\mathbf{#2\%})$}}

\ShortHeadings{Two-Sample Testing via Path-based Inference}{Two-Sample Testing via Path-based Inference}
\firstpageno{1}

\begin{document}

\title{Two-Sample Testing via Path-based Inference}

\author{\name{Eshant English\textsuperscript{1,3,*}\thanks{Equal contribution, sorted in alphabetical order. Contact email: eshant.english@outlook.co.nz, wei-cheng.lai@hpi.de}}, \name{Wei-Cheng Lai\textsuperscript{3,*}}, \name{Yanfeng Yang\textsuperscript{2,4}}, \name{Kenji Fukumizu\textsuperscript{2}}, \name{Taiji Suzuki}\textsuperscript{1,5}, \name{Christoph Lippert\textsuperscript{3,6}} \\
  \\
  \addr{\textsuperscript{1}The University of Tokyo, Tokyo, Japan} \\
  \addr{\textsuperscript{2}Institute of Statistical Mathematics, Japan} \\
  \addr{\textsuperscript{3}Hasso Plattner Institute for Digital Engineering, Germany} \\
  \addr{\textsuperscript{4}The Graduate University for Advanced Studies, SOKENDAI, Japan} \\
\addr{\textsuperscript{5}RIKEN Centre for Advanced Intelligence Project (AIP), Tokyo, Japan} \\
  \addr{\textsuperscript{6}Hasso Plattner Institute for Digital Health at the Icahn School of Medicine at Mount Sinai, USA} 
}


\maketitle

\begin{abstract}
Modern deep generative models are primarily studied for their ability to generate realistic samples, yet the generative dynamics they learn can also serve as objects of statistical inference.
We develop this idea for two-sample testing, the problem of deciding whether the same distribution generated two finite datasets.
Using stochastic interpolants, we connect both distributions to a shared Gaussian bottleneck, so that each half of the resulting path is a Gaussian channel acting on a single population.
We prove that the null hypothesis holds if and only if the population denoiser, or equivalently, the velocity fields of the two halves, coincide at any single noise level, which amounts to a reflection symmetry of the path about the bottleneck.
Deviations from this symmetry yield a continuum of two-sample witnesses, which we estimate via held-out regression risks on learned denoisers and velocities and aggregate along the path; under an information-theoretic weighting, the aggregated discrepancy equals the Jeffreys divergence between the noise-smoothed distributions.
Calibrating the resulting statistics by permutation yields tests that are valid in finite samples for any trained networks and consistent when the fields are learned accurately.
On a synthetic benchmark and three image benchmarks, the proposed tests improve power over the strongest baseline by up to 33 percentage points at an equal total sample budget, with the best choice of regression representation and path weighting depending on the data modality.
These results show that generative paths provide a principled representation for statistical testing, extending stochastic-interpolant models beyond generation.
\end{abstract}
\vspace{1cm}

\newpage
\tableofcontents
\newpage

\section{Introduction}
\label{sec:intro}

Dynamical-transport generative models \citep{grathwohl_ffjord_2018, bunne_schrodinger_2023}, including diffusion models \citep{song_score-based_2020}, flow matching \citep{lipman_flow_2022} and stochastic interpolants \citep{albergo_stochastic_2025, albergo_building_2022}, are the state of the art for data synthesis.
The underlying idea is simple: given two distributions $\rho_0$ and $\rho_1$, one constructs a continuous dynamical process $(x_t)_{t \in [0,1]}$ such that $x_t \sim \rho_0$ at $t=0$ and $x_t \sim \rho_1$ at $t=1$. The choice of $t \in [0,1]$ is purely a unifying convention; diffusion models, for instance, instead run a fixed noising process over an extended or infinite time horizon to reach a Gaussian. 
Conventions differ across implementations, but in each case, a deep network is trained to approximate the field that generates this process: a velocity field $b(t,x)$ in the deterministic case, or a drift and score in the stochastic case.
Combining this simple modelling principle with the expressivity of deep networks has yielded strong empirical results and enabled applications in high-fidelity computer vision \citep{rombach_high-resolution_2022}, audio and speech synthesis \citep{kong_diffwave_2020}, and computational biology, such as protein design \citep{watson_novo_2023}.

Despite this abundance of generative applications, the direct use of these models for statistical inference remains largely unexplored.
The continuous transformation carrying a sample from $\rho_0$ to $\rho_1$ traces out a trajectory, which we call a \emph{path}, describing how a sample evolves from $\rho_0$ into a sample from $\rho_1$.
This path may follow stochastic or deterministic dynamics; we restrict our attention to deterministic dynamics, under which each sample follows a unique, identifiable path.
We argue that the path itself, not merely its endpoints, is a rich statistical object in its own right.
At the population level, suppose we want to test whether the distribution of collected data has changed over time.
We can construct a dynamical transport model that carries samples from $\rho_0$ to $\rho_1$.
If the dynamics realised optimal transport, then under $\rho_0 = \rho_1$ the optimal map would be the identity so that every sample would remain at rest, and any movement along the path would be evidence of a distributional change.
Learned dynamics are rarely optimal, however, so this picture serves only as intuition.
In \cref{sec:method}, we show that an analogous property holds exactly for an interpolant that routes both distributions through a shared Gaussian bottleneck \citep{albergo_stochastic_2025}: the null hypothesis holds if and only if the path is reflection symmetric about the bottleneck (\cref{prop:invariance,fig:path_invariance}).
We refer to this broader line of reasoning as \emph{path-based inference}.

Paths are particularly well-suited to this role because of the stability properties of the models that generate them.
\citet{kadkhodaie_generalization_2024, mlodozeniec_influence_2025} showed that these models are stable in the sense that sampling from the same seed and the same latent vector yields visually similar outputs across different trained models, reflecting a stability of the learned score, or velocity, field whenever the target distribution is held fixed.
\citet{briq_amazing_2025} showed an analogous stability of the path dynamics itself under ODE-based sampling.
This might suggest that it suffices to test only the initial and final points of the path, but doing so discards information.
\citet{english_path_2026} showed that paths can differ substantially even when their final mappings are visually near-identical, and, more importantly, such endpoint similarity can itself be an artefact of modelling bias rather than genuine distributional closeness.
Additionally, these models are known to lock onto low-dimensional data manifolds in high-dimensional settings \citep{pidstrigach_score-based_2022, bortoli_convergence_2022}, making them well-suited to high-dimensional inference problems.
Intuitively, path-based inference lets us base statistical conclusions on the entire generative process connecting two datasets, rather than on their endpoints alone.

\begin{figure}[!t]

    \centering
    \includegraphics[width=0.8\textwidth]{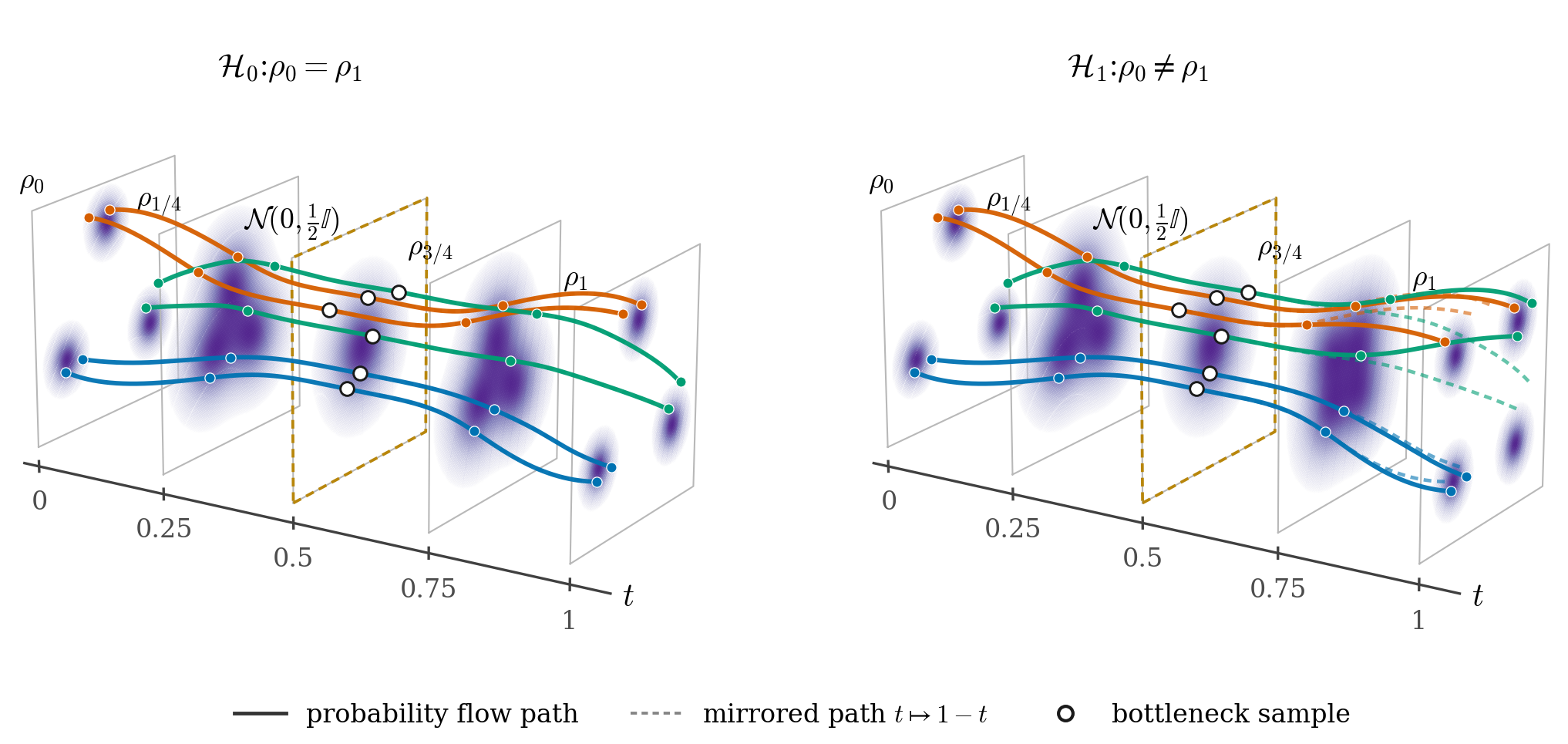}
    \caption{Path-based two-sample testing with a Gaussian-bottleneck interpolant.
    Samples from $\rho_0$ are transported by the probability flow (solid curves) to the shared bottleneck $\mathcal{N}(0, \tfrac{1}{2}\mathbb{I}_d)$ at $t = \tfrac{1}{2}$ (open circles) and onwards to $\rho_1$, while dashed curves show the first half of each path mirrored under $t \mapsto 1 - t$.
    \emph{Left:} under $\mathcal{H}_0$, the path is reflection symmetric about the bottleneck, so matched slices such as $\rho_{1/4}$ and $\rho_{3/4}$ coincide, each trajectory retraces its first half, and the mirrored curves are hidden beneath the flow.
    \emph{Right:} under $\mathcal{H}_1$, the second half of each path departs from the mirror image of the first, and this asymmetry, measured through the denoiser and velocity fields at matched noise levels, is the two-sample signal.}
    \label{fig:path_invariance}
\end{figure}

We develop this idea for the classical problem of \emph{two-sample testing} \citep{lehmann_invariance_2005}: given two finite datasets, determine whether they were drawn from the same distribution.
Two-sample testing underlies many statistical and machine learning tasks, including detecting distribution shift \citep{rabanser_failing_2019}, checking whether generated data match the true data distribution \citep{salimans_improved_2016}, and testing whether a distribution differs across demographic groups \citep{barocas_fairness_2023}.
Although classical two-sample testing is well studied \citep{gretton2012kernel, anderson1991hotelling}, existing methods still struggle with structured, high-dimensional data, especially in the low-sample regimes common in resource-limited settings such as biomedicine \citep{ramdas_decreasing_2015}.
A further desideratum is \emph{interpretability}: beyond a binary reject or fail-to-reject decision, we would like to identify which observations drive a rejection, for instance, to flag and remove corrupted data from a newly collected batch while preserving the rest of the distribution.
Path-based inference is naturally suited to this, since our test statistics are differences of averages of observation-level scores, so the observations that contribute most to a rejection can be read off directly (\cref{fig:localisation,sec:interpretability}). For a comprehensive positioning with respect to related works, see \cref{sec:related}.

We build on the stochastic interpolant framework, which unifies the major dynamical-transport generative models under a single formalism, and make the following contributions: 
\textbf{(i)}~We show that an interpolant routing both distributions through a shared Gaussian bottleneck \citep{albergo_stochastic_2025} turns two-sample testing into a test of path symmetry, and we prove that the null hypothesis holds if and only if the denoiser, or equivalently velocity, fields of its two halves coincide at any single noise level (\cref{sec:path_invariance}); 
\textbf{(ii)}~we turn this characterisation into denoising-risk and velocity-risk statistics that are evaluated on held-out data and aggregated along the path, and we show that an information-theoretic weighting recovers the Jeffreys divergence between the noise-smoothed distributions (\cref{sec:aggregation}); 
\textbf{(iii)}~we prove that permutation calibration yields tests that are valid in finite samples for any trained networks and consistent when the fields are learned accurately, with an explicit account of how estimation errors affect power (\cref{sec:permutation,sec:consistency}); and 
\textbf{(iv)}~on a synthetic benchmark and three image benchmarks, the proposed tests improve power over the strongest baseline by up to 33 percentage points at an equal total sample budget (\cref{sec:experiments}).
\section{Preliminaries}
\label{sec:prelim}

\subsection{Stochastic Interpolants}
\label{sec:prelim-si}

We work within the stochastic interpolant framework of \citet{albergo_stochastic_2025}, which unifies flow-based and diffusion-based generative models under a single construction.
Let $\rho_0$ and $\rho_1$ be probability densities on $\mathbb{R}^d$, and let $(x_0, x_1) \sim \nu$ be a coupling with marginals $\rho_0$ and $\rho_1$; the simplest choice is the independent coupling $\nu = \rho_0 \otimes \rho_1$.
Let $z \sim \mathcal{N}(0, \mathbb{I}_d)$ be a latent variable independent of $(x_0, x_1)$.

The stochastic interpolant is the process
\begin{equation}
\label{eq:interpolant}
    x_t = \alpha_t\, x_0 + \beta_t\, x_1 + \gamma_t\, z, \qquad t \in [0,1],
\end{equation}
where the coefficients $\alpha_t, \beta_t, \gamma_t$ are sufficiently smooth and satisfy the boundary conditions
\begin{equation}
\label{eq:boundary}
    \alpha_0 = \beta_1 = 1, \qquad \alpha_1 = \beta_0 = 0, \qquad \gamma_0 = \gamma_1 = 0, \qquad \gamma_t > 0 \;\; \text{for } t \in (0,1).
\end{equation}
These conditions guarantee that $x_{t=0} = x_0$ and $x_{t=1} = x_1$, so the time-dependent density $\rho_t$ of $x_t$ interpolates between $\rho_0$ and $\rho_1$.
A standard choice is $\alpha_t = 1 - t$, $\beta_t = t$ and $\gamma_t = \sqrt{2t(1-t)}$, and the same latent coefficient $\gamma_t$ is used in our construction in \cref{sec:method}.
The latent term $\gamma_t z$ smooths $\rho_t$ at intermediate times; setting $\gamma_t \equiv 0$ recovers the deterministic interpolants underlying flow matching \citep{lipman_flow_2022} and rectified flow \citep{liu_flow_2022}.

\Eqref{eq:interpolant} specifies the path at the level of individual samples: conditional on $(x_0, x_1, z)$, the trajectory $t \mapsto x_t$ is fully determined, and its law defines the modelling path $(\rho_t)_{t \in [0,1]}$.
The conditional velocity along this path is available in closed form,
\begin{equation}
\label{eq:cond-velocity}
    \dot{x}_t = \dot{\alpha}_t\, x_0 + \dot{\beta}_t\, x_1 + \dot{\gamma}_t\, z,
\end{equation}
where dots denote time derivatives.
However, $\dot{x}_t$ depends on the endpoints and the latent, which are not accessible from $x_t$ alone.
The generative dynamics are therefore obtained by averaging over them. 
\paragraph{Velocity field:}
The marginal velocity field is the conditional expectation of~\eqref{eq:cond-velocity} given the current position,
\begin{equation}
\label{eq:velocity}
    b(t, x) = \mathbb{E}\!\left[\dot{x}_t \,\middle|\, x_t = x\right].
\end{equation}
\citet{albergo_building_2022} show that $\rho_t$ satisfies the transport equation
\begin{equation}
\label{eq:transport}
    \partial_t \rho_t + \nabla \cdot \big(b\, \rho_t\big) = 0, \qquad \rho_{t=0} = \rho_0.
\end{equation}
Consequently, solutions of the probability flow ODE
\begin{equation}
\label{eq:ode}
    \dot{X}_t = b(t, X_t), \qquad X_0 \sim \rho_0,
\end{equation}
satisfy $X_t \sim \rho_t$ for all $t \in [0,1]$, and in particular $X_1 \sim \rho_1$.
Integrating~\eqref{eq:ode} backward in time from $X_1 \sim \rho_1$ likewise transports $\rho_1$ to $\rho_0$.
Although the interpolant~\eqref{eq:interpolant} involves the latent $z$, the transport~\eqref{eq:ode} is deterministic: stochasticity enters only through the construction of $\rho_t$, and each initial condition $X_0$ is mapped to a unique trajectory.

\paragraph{Denoisers and score:}
Equivalently, the dynamics can be parameterised through the denoisers
\begin{equation}
\label{eq:denoisers}
    \eta_0(t,x) = \mathbb{E}\!\left[x_0 \,\middle|\, x_t = x\right], \qquad
    \eta_1(t,x) = \mathbb{E}\!\left[x_1 \,\middle|\, x_t = x\right], \qquad
    \eta_z(t,x) = \mathbb{E}\!\left[z \,\middle|\, x_t = x\right].
\end{equation}
By linearity of conditional expectation,
\begin{equation}
\label{eq:velocity-denoiser}
    b(t,x) = \dot{\alpha}_t\, \eta_0(t,x) + \dot{\beta}_t\, \eta_1(t,x) + \dot{\gamma}_t\, \eta_z(t,x),
    \qquad
    x = \alpha_t\, \eta_0(t,x) + \beta_t\, \eta_1(t,x) + \gamma_t\, \eta_z(t,x),
\end{equation}
so that only two of the three denoisers need to be learned.
Moreover, for Gaussian $z$, the latent denoiser determines the score of $\rho_t$,
\begin{equation}
\label{eq:score}
    \nabla \log \rho_t(x) = -\gamma_t^{-1}\, \eta_z(t,x), \qquad t \in (0,1),
\end{equation}
which connects the framework to score-based diffusion models.
For any $\epsilon \geq 0$, the SDE $\mathrm{d}X_t = \big[b(t,X_t) + \epsilon\, \nabla \log \rho_t(X_t)\big]\,\mathrm{d}t + \sqrt{2\epsilon}\,\mathrm{d}W_t$ shares the marginals $\rho_t$; throughout this work, we use only the deterministic case $\epsilon = 0$.

\paragraph{Learning:}
Each conditional expectation above is the unique minimiser of a quadratic regression objective.
The velocity field minimises
\begin{equation}
\label{eq:loss-b}
    \mathcal{L}_b[\widehat{b}] = \int_0^1 \mathbb{E}\,\big|\widehat{b}(t, x_t) - \dot{x}_t\big|^2 \,\mathrm{d}t,
\end{equation}
and the latent denoiser minimises
\begin{equation}
\label{eq:loss-eta}
    \mathcal{L}_{\eta_z}[\widehat{\eta}_z] = \int_0^1 \mathbb{E}\,\big|\widehat{\eta}_z(t, x_t) - z\big|^2 \,\mathrm{d}t,
\end{equation}
with analogous objectives for $\eta_0$ and $\eta_1$ obtained by replacing $z$ with $x_0$ or $x_1$.
Here, the expectation is taken over $(x_0, x_1) \sim \nu$ and $z \sim \mathcal{N}(0, \mathbb{I}_d)$, with $x_t$ and $\dot{x}_t$ given by~\eqref{eq:interpolant} and~\eqref{eq:cond-velocity}.
These objectives coincide, up to additive constants independent of $\widehat{b}$ and $\widehat{\eta}_z$, with those of \citet{albergo_building_2022}.
In practice, $\widehat{b}$ and $\widehat{\eta}_z$ are parameterised by neural networks, and the objectives are estimated by Monte Carlo sampling of $t \sim \mathcal{U}[0,1]$, the endpoints and the latent.

\subsection{Two-Sample Testing}
\label{sec:prelim-testing}

Consider two independent datasets $\mathcal{D}_0 = \{x_0^{(i)}\}_{i=1}^{n_0}$ and $\mathcal{D}_1 = \{x_1^{(j)}\}_{j=1}^{n_1}$ in $\mathbb{R}^d$, whose members are drawn i.i.d.\ from densities $\rho_0$ and $\rho_1$, respectively.
The densities themselves are unknown, and we only have access to the two finite datasets.
The two-sample problem asks whether these datasets are consistent with a single underlying density, which we phrase as a test of the null hypothesis against its complement,
\begin{equation}
\label{eq:hypotheses}
    \mathcal{H}_0 \colon \rho_0 = \rho_1
    \qquad \text{versus} \qquad
    \mathcal{H}_1 \colon \rho_0 \neq \rho_1 .
\end{equation}
Our validity guarantee is non-asymptotic, holding at finite $n_0$ and $n_1$, whereas power is analysed as the sample sizes grow.

A test is a binary decision rule $\Delta(\mathcal{D}_0, \mathcal{D}_1) \in \{0, 1\}$, where $\Delta = 1$ corresponds to rejecting $\mathcal{H}_0$.
We allow $\Delta$ to be randomised, that is, to depend on internal randomness $\xi$ drawn independently of the data, such as the sample splitting, network training, Monte Carlo draws, and permutations used in \cref{sec:method}.
We write $\mathbb{P}_{\rho_0, \rho_1}$ for the joint law of $(\mathcal{D}_0, \mathcal{D}_1, \xi)$ when the datasets are drawn from $\rho_0$ and $\rho_1$; the law of $\xi$ is fixed by the test and is therefore left implicit.
More generally, subscripts on $\mathbb{P}$ and $\mathbb{E}$ indicate the distributions of the random variables involved, as in $\mathbb{E}_{x \sim \rho}[f(x)]$, and unsubscripted probabilities and expectations are taken over all random variables in the expression.

A false rejection, or type I error, occurs when $\Delta = 1$ although $\rho_0 = \rho_1$.
We require its probability to be controlled uniformly over the set $\mathcal{P}(\mathbb{R}^d)$ of probability densities on $\mathbb{R}^d$,
\begin{equation}
\label{eq:level}
    \sup_{\rho_0 \in \mathcal{P}(\mathbb{R}^d)} \mathbb{P}_{\rho_0, \rho_0}\big[\Delta(\mathcal{D}_0, \mathcal{D}_1) = 1\big] \leq \kappa ,
\end{equation}
for a prescribed $\kappa \in (0, 1)$, in which case $\Delta$ is said to have level $\kappa$.
Conversely, a type II error occurs when $\Delta = 0$ although $\rho_0 \neq \rho_1$.
For a fixed alternative, if
\begin{equation}
\label{eq:type-ii}
    \mathbb{P}_{\rho_0, \rho_1}\big[\Delta(\mathcal{D}_0, \mathcal{D}_1) = 0\big] \leq \delta
\end{equation}
for some $\delta \in (0, 1)$, the test is said to have power $1 - \delta$ against that alternative.
The goal is to construct tests that attain level $\kappa$ while maximising power against a broad class of alternatives.
\section{Path-Based Two-Sample Testing}
\label{sec:method}

We work in the setting of \cref{sec:prelim-testing}.
Our construction connects each of $\rho_0$ and $\rho_1$ to a common Gaussian bottleneck that depends on neither population.
Rather than comparing $\rho_0$ and $\rho_1$ only at their endpoints, we compare the regression fields induced along the two halves of this path.
Under $\mathcal{H}_0$, the two fields coincide at every matched noise level. In contrast, under $\mathcal{H}_1$, they differ at every matched noise level, so their discrepancy supplies a continuum of two-sample witnesses indexed by position along the path.
Throughout, population fields are distinguished from their neural estimates, which we denote with hats.

\subsection{A Gaussian-Bottleneck Interpolant}
\label{sec:gaussian_path}

Following \citet{albergo_stochastic_2025}, we specialise the interpolant~\eqref{eq:interpolant} to the schedule
\begin{equation}
\label{eq:bottleneck_schedule}
    \alpha_t = \cos^2(\pi t)\, \mathbf{1}_{\{t \leq 1/2\}},
    \qquad
    \beta_t = \cos^2(\pi t)\, \mathbf{1}_{\{t \geq 1/2\}},
    \qquad
    \gamma_t = \sqrt{2t(1-t)} .
\end{equation}
Its structure is most transparent in the radial coordinate $s = |t - \tfrac{1}{2}| \in [0, \tfrac{1}{2}]$, which measures distance from the midpoint.
Each radius $s$ corresponds to the two mirrored times
\begin{equation}
\label{eq:mirrored_times}
    t_s^0 = \tfrac{1}{2} - s,
    \qquad
    t_s^1 = \tfrac{1}{2} + s,
\end{equation}
and we define the radial schedule
\begin{equation}
\label{eq:radial_schedule}
    a_s = \sin^2(\pi s),
    \qquad
    \sigma_s = \sqrt{\tfrac{1}{2} - 2s^2}.
\end{equation}
For $s \in (0, \tfrac{1}{2})$, this schedule satisfies $a_s > 0$ and $\sigma_s > 0$, and its derivatives $a_s'$ and $\sigma_s'$ are finite; we use these facts without further mention.
The following lemma records the properties of the schedule that we need.

\begin{lemma}[Gaussian-bottleneck interpolant]
\label{lem:bottleneck}
The schedule~\eqref{eq:bottleneck_schedule} satisfies the boundary conditions~\eqref{eq:boundary}, with $\alpha_t, \beta_t \in C^1([0,1])$ and $\gamma_t^2 \in C^\infty([0,1])$.
Moreover, for every $s \in [0, \tfrac{1}{2}]$ and $\ell \in \{0, 1\}$,
\begin{equation}
\label{eq:channels}
    x_{t_s^\ell} = a_s\, x_\ell + \sigma_s\, z .
\end{equation}
In particular, the law of $x_t$ does not depend on the coupling $\nu$, and the midpoint density is the pivot $\rho_{1/2} = \mathcal{N}\big(0, \tfrac{1}{2}\mathbb{I}_d\big)$.
\end{lemma}

\begin{proof}
The proof consists of a direct evaluation of the schedule.
For the boundary values, $\alpha_0 = \cos^2 0 = 1$, $\alpha_1 = 0$, $\beta_0 = 0$, $\beta_1 = \cos^2 \pi = 1$, and $\gamma_0 = \gamma_1 = 0$, while $\gamma_t^2 = 2t(1-t) > 0$ for $t \in (0,1)$.
The function $\gamma_t^2$ is a polynomial and hence lies in $C^\infty([0,1])$.
On $[0, \tfrac{1}{2}]$, $\alpha_t = \cos^2(\pi t)$ has derivative $-\pi \sin(2\pi t)$, and both $\alpha_t$ and its derivative vanish at $t = \tfrac{1}{2}$, matching the identically zero branch on $(\tfrac{1}{2}, 1]$.
Hence $\alpha_t \in C^1([0,1])$, and the same argument applies to $\beta_t$.

For the channel representation, let $s \in [0, \tfrac{1}{2}]$.
At $t = t_s^0 = \tfrac{1}{2} - s$, we have $\alpha_t = \cos^2(\tfrac{\pi}{2} - \pi s) = \sin^2(\pi s) = a_s$ and $\beta_t = 0$, since either $t < \tfrac{1}{2}$ or $t = \tfrac{1}{2}$, where $\cos^2(\tfrac{\pi}{2}) = 0$.
Moreover, $\gamma_t^2 = 2(\tfrac{1}{2} - s)(\tfrac{1}{2} + s) = \tfrac{1}{2} - 2s^2 = \sigma_s^2$.
Substituting into~\eqref{eq:interpolant} gives $x_{t_s^0} = a_s x_0 + \sigma_s z$.
At $t = t_s^1$, the identity $\cos^2(\tfrac{\pi}{2} + \pi s) = \sin^2(\pi s)$ gives $x_{t_s^1} = a_s x_1 + \sigma_s z$ in the same way.
Since $z$ is independent of $(x_0, x_1)$, the law of $x_{t_s^\ell}$ depends only on the marginal $\rho_\ell$ and not on the coupling $\nu$.
At $s = 0$, we have $a_0 = 0$ and $\sigma_0^2 = \tfrac{1}{2}$, so $x_{1/2} = \sigma_0 z \sim \mathcal{N}(0, \tfrac{1}{2}\mathbb{I}_d)$.
\end{proof}

\Cref{lem:bottleneck} shows that each half of the path is a Gaussian channel acting on a single population.
We write
\begin{equation}
\label{eq:channel_densities}
    \zeta_s^\ell = \rho_{t_s^\ell},
    \qquad
    \ell \in \{0, 1\},
\end{equation}
for the density of channel $\ell$ at radius $s$, so that $\zeta_s^0$ and $\zeta_s^1$ are the interpolant densities at the two mirrored times.
As $s \uparrow \tfrac{1}{2}$, we have $a_s \to 1$ and $\sigma_s \to 0$, so the channels recover the data densities, $\zeta_{1/2}^\ell = \rho_\ell$, while at $s = 0$ both reduce to the pivot, $\zeta_0^0 = \zeta_0^1 = \mathcal{N}\big(0, \tfrac{1}{2}\mathbb{I}_d\big)$.
In the original time coordinate, the construction therefore realises the transport
\begin{equation}
    \rho_0 \longrightarrow \mathcal{N}\big(0, \tfrac{1}{2}\mathbb{I}_d\big) \longleftarrow \rho_1 ,
\end{equation}
while in the radial coordinate, it consists of two outward paths emanating from the same pivot.
The coefficients $\alpha_t$ and $\beta_t$ are only once differentiable at the bottleneck.
This suffices for all results below, which we establish directly for the schedule~\eqref{eq:bottleneck_schedule} rather than through the general theory of \citet{albergo_stochastic_2025}.

Throughout, we impose the following condition on the two populations.

\begin{assumption}[Moments]
\label{ass:moments}
The densities $\rho_0$ and $\rho_1$ have finite second moments.
\end{assumption}

\Cref{ass:moments} guarantees that the fields of \cref{sec:path_fields} are well defined and that the squared-error risks of \cref{sec:witnesses} are finite.
The Jeffreys identity of \cref{prop:jeffreys} additionally relies on a technical regularity condition, \cref{ass:regularity}, which we state in \cref{sec:aggregation}.

\subsection{Denoiser and Velocity Fields}
\label{sec:path_fields}

For $\ell \in \{0, 1\}$ and $s \in (0, \tfrac{1}{2})$, we define the posterior-mean denoiser and the outward velocity of channel $\ell$ as
\begin{equation}
\label{eq:denoiser}
    \eta_s^\ell(x) = \mathbb{E}\!\left[x_\ell \,\middle|\, x_{t_s^\ell} = x\right],
    \qquad
    u_s^\ell(x) = \mathbb{E}\!\left[\partial_s x_{t_s^\ell} \,\middle|\, x_{t_s^\ell} = x\right],
\end{equation}
where $\partial_s x_{t_s^\ell} = a_s' x_\ell + \sigma_s' z$ and primes denote derivatives in $s$.
These are the fields of \cref{sec:prelim-si} expressed in the radial coordinate: $\eta_s^0 = \eta_0(t_s^0, \cdot)$ and $\eta_s^1 = \eta_1(t_s^1, \cdot)$, while $u_s^0 = -b(t_s^0, \cdot)$ and $u_s^1 = b(t_s^1, \cdot)$, the sign reflecting that $s$ runs against $t$ on the first half.
Under \cref{ass:moments}, both conditional expectations exist and admit continuous versions on $\mathbb{R}^d$, constructed for the denoiser in \cref{lem:smoothing} and for the velocity in the proof of \cref{lem:relations}.
We use these versions throughout, so that identities between fields hold at every point.

Write $\varphi_\sigma(x) = (2\pi\sigma^2)^{-d/2} \exp\!\big(-|x|^2 / (2\sigma^2)\big)$ for the density of $\mathcal{N}(0, \sigma^2 \mathbb{I}_d)$.
The following lemma collects the properties of the channel densities used throughout.

\begin{lemma}[Gaussian smoothing]
\label{lem:smoothing}
Let $s \in [0, \tfrac{1}{2})$ and $\ell \in \{0, 1\}$.
The channel density is
\begin{equation}
\label{eq:smoothed_density}
    \zeta_s^\ell(x) = \int_{\mathbb{R}^d} \varphi_{\sigma_s}(x - a_s y)\, \rho_\ell(y)\, \mathrm{d}y ,
\end{equation}
which is infinitely differentiable and strictly positive on $\mathbb{R}^d$.
If, in addition, $s > 0$ and \cref{ass:moments} holds, then the posterior density of $x_\ell$ given $x_{t_s^\ell} = x$ is $\varphi_{\sigma_s}(x - a_s y)\, \rho_\ell(y) / \zeta_s^\ell(x)$, and
\begin{equation}
\label{eq:denoiser_version}
    \eta_s^\ell(x) = \frac{1}{\zeta_s^\ell(x)} \int_{\mathbb{R}^d} y\, \varphi_{\sigma_s}(x - a_s y)\, \rho_\ell(y)\, \mathrm{d}y
\end{equation}
is a continuous version of the conditional expectation in~\eqref{eq:denoiser}.
\end{lemma}

\begin{proof}
By \cref{lem:bottleneck}, $x_{t_s^\ell} = a_s x_\ell + \sigma_s z$ with $z \sim \mathcal{N}(0, \mathbb{I}_d)$ independent of $x_\ell$.
Conditionally on $x_\ell = y$, the variable $x_{t_s^\ell}$ therefore has density $\varphi_{\sigma_s}(\cdot - a_s y)$, and integrating over $y \sim \rho_\ell$ gives~\eqref{eq:smoothed_density}.
Every partial derivative of $\varphi_{\sigma_s}$, of any order, is bounded on $\mathbb{R}^d$.
Hence, the corresponding derivative of the integrand in~\eqref{eq:smoothed_density} is dominated by a constant multiple of $\rho_\ell(y)$, uniformly in $x$, and dominated convergence allows differentiation under the integral sign to all orders.
This shows $\zeta_s^\ell \in C^\infty(\mathbb{R}^d)$.
Strict positivity follows because $\varphi_{\sigma_s} > 0$ everywhere and $\rho_\ell$ integrates to one.
The posterior density follows from Bayes' rule.
Finally, the integrand in the numerator of~\eqref{eq:denoiser_version} is continuous in $x$ and bounded by $(2\pi\sigma_s^2)^{-d/2}\, |y|\, \rho_\ell(y)$, which is integrable under \cref{ass:moments}.
The numerator is therefore continuous by dominated convergence, and so is the ratio, since the denominator is continuous and strictly positive.
\end{proof}

\begin{lemma}[Denoiser, velocity and score]
\label{lem:relations}
Under \cref{ass:moments}, for every $s \in (0, \tfrac{1}{2})$ and $\ell \in \{0, 1\}$: 
\textbf{(i)}~the velocity satisfies $u_s^\ell(x) = \lambda_s\, \eta_s^\ell(x) + \frac{\sigma_s'}{\sigma_s}\, x$ with $\lambda_s = a_s' - a_s \frac{\sigma_s'}{\sigma_s} = \pi \sin(2\pi s) + \frac{2s\, a_s}{\sigma_s^2}$; 
\textbf{(ii)}~the denoiser satisfies Tweedie's identity, $\eta_s^\ell(x) = \frac{1}{a_s}\big(x + \sigma_s^2\, \nabla \log \zeta_s^\ell(x)\big)$; and 
\textbf{(iii)}~the contrasts satisfy $u_s^1 - u_s^0 = \lambda_s ( \eta_s^1 - \eta_s^0 ) = \chi_s ( \nabla \log \zeta_s^1 - \nabla \log \zeta_s^0 )$ with $\chi_s = \frac{\lambda_s \sigma_s^2}{a_s} > 0$.
\end{lemma}

\begin{proof}
Parts~(i) and~(ii) are computed directly, and part~(iii) follows by subtraction.
Fix $s \in (0, \tfrac{1}{2})$ and $\ell \in \{0, 1\}$, and abbreviate $x_t = x_{t_s^\ell}$.
Both $\partial_s x_t = a_s' x_\ell + \sigma_s' z$ and $x_\ell$ are integrable under \cref{ass:moments}, so all conditional expectations below exist.

For~(i), linearity of conditional expectation gives $u_s^\ell(x) = a_s'\, \eta_s^\ell(x) + \sigma_s'\, \mathbb{E}[z \mid x_t = x]$.
Since $z = (x_t - a_s x_\ell)/\sigma_s$, we have $\mathbb{E}[z \mid x_t = x] = (x - a_s \eta_s^\ell(x))/\sigma_s$.
Substituting and collecting terms yields~(i).
The right-hand side of~(i) is continuous by \cref{lem:smoothing}, and it provides the continuous version of $u_s^\ell$ used throughout.
The explicit form of $\lambda_s$ follows from $a_s' = 2\pi \sin(\pi s) \cos(\pi s) = \pi \sin(2\pi s)$ and from differentiating $\sigma_s^2 = \tfrac{1}{2} - 2s^2$, which gives $\sigma_s'/\sigma_s = -2s/\sigma_s^2$.

For~(ii), differentiating~\eqref{eq:smoothed_density} under the integral sign, as justified in \cref{lem:smoothing}, gives
\begin{equation}
    \nabla \zeta_s^\ell(x)
    =
    -\frac{1}{\sigma_s^2} \int_{\mathbb{R}^d} (x - a_s y)\, \varphi_{\sigma_s}(x - a_s y)\, \rho_\ell(y)\, \mathrm{d}y
    =
    -\frac{x}{\sigma_s^2}\, \zeta_s^\ell(x) + \frac{a_s}{\sigma_s^2}\, \zeta_s^\ell(x)\, \eta_s^\ell(x),
\end{equation}
where the second equality uses~\eqref{eq:denoiser_version}.
Dividing by $\zeta_s^\ell(x) > 0$ and rearranging yields~(ii).

For~(iii), subtracting~(i) for $\ell = 0$ from~(i) for $\ell = 1$ cancels the common affine term and gives the first equality in~(iii).
Subtracting~(ii) for the two channels gives $\eta_s^1 - \eta_s^0 = (\sigma_s^2 / a_s)(\nabla \log \zeta_s^1 - \nabla \log \zeta_s^0)$, and multiplying by $\lambda_s$ gives the second equality.
Finally, on $(0, \tfrac{1}{2})$ we have $\sin(2\pi s) > 0$, $a_s > 0$ and $\sigma_s^2 > 0$, so $\lambda_s > 0$ and hence $\chi_s > 0$.
\end{proof}

The affine term in \Cref{lem:relations}~(i) is known and common to both channels, so it cancels in every contrast.
Denoiser, velocity and score contrasts therefore encode the same population difference, up to known positive scalar functions of $s$.

\subsection{Path Invariance Characterises the Null}
\label{sec:path_invariance}

\begin{proposition}[Path invariance]
\label{prop:invariance}
Suppose \cref{ass:moments} holds, and fix any $s \in (0, \tfrac{1}{2})$.
Then the following are equivalent:
\textbf{(i)}~$\mathcal{H}_0$ holds ($\rho_0 = \rho_1$); \ 
\textbf{(ii)}~$\zeta_s^0 = \zeta_s^1$; \ 
\textbf{(iii)}~$\eta_s^0 = \eta_s^1$; \ 
\textbf{(iv)}~$u_s^0 = u_s^1$.
\end{proposition}

\begin{proof}
We prove the cycle (i)~$\Rightarrow$~(ii)~$\Rightarrow$~(iii)~$\Rightarrow$~(i), and separately the equivalence (iii)~$\Leftrightarrow$~(iv).
The key step is (iii)~$\Rightarrow$~(i), which inverts the Gaussian channel.

The implication (i)~$\Rightarrow$~(ii) holds because, by \cref{lem:bottleneck}, the law of $x_{t_s^\ell}$ is determined by $\rho_\ell$ alone, so $\rho_0 = \rho_1$ implies $\zeta_s^0 = \zeta_s^1$.
For (ii)~$\Rightarrow$~(iii), the right-hand side of \Cref{lem:relations}~(ii) depends on channel $\ell$ only through $\zeta_s^\ell$, so equal densities give equal denoisers.

For (iii)~$\Rightarrow$~(i), \Cref{lem:relations}~(ii) shows that equal denoisers give $\nabla \log \zeta_s^0 = \nabla \log \zeta_s^1$ on $\mathbb{R}^d$.
Both densities are smooth and strictly positive by \cref{lem:smoothing}, so $\log(\zeta_s^0 / \zeta_s^1)$ is constant on the connected set $\mathbb{R}^d$.
Since both densities integrate to one, the constant is zero, and hence $\zeta_s^0 = \zeta_s^1$.
Writing $\hat{f}_\ell(k) = \mathbb{E}_{x \sim \rho_\ell}\big[e^{\mathrm{i} k \cdot x}\big]$ for the characteristic function of $\rho_\ell$, the representation~\eqref{eq:channels} gives the characteristic function of $\zeta_s^\ell$ as $\hat{f}_\ell(a_s k)\, e^{-\sigma_s^2 |k|^2 / 2}$.
The Gaussian factor never vanishes, so $\hat{f}_0(a_s k) = \hat{f}_1(a_s k)$ for all $k \in \mathbb{R}^d$.
Since $a_s > 0$, this gives $\hat{f}_0 = \hat{f}_1$, and L\'evy's uniqueness theorem yields $\rho_0 = \rho_1$.

Finally, (iii)~$\Leftrightarrow$~(iv) follows from \Cref{lem:relations}~(iii): since $u_s^1 - u_s^0 = \lambda_s (\eta_s^1 - \eta_s^0)$ with $\lambda_s > 0$, one difference vanishes if and only if the other does.
\end{proof}

In the original time coordinate, condition~(iv) is a reflection antisymmetry of the velocity field about the bottleneck,
\begin{equation}
\label{eq:mirror_cancellation}
    b(\tfrac{1}{2} - s, x) + b(\tfrac{1}{2} + s, x) = 0 ,
\end{equation}
and condition~(ii), which reads $\rho_{1/2 - s} = \rho_{1/2 + s}$, is the corresponding reflection symmetry of the interpolant density.

Consequently, whenever the probability flow ODE~\eqref{eq:ode} is well posed, every trajectory satisfies $X_{1-t} = X_t$ under $\mathcal{H}_0$, since $t \mapsto X_{1-t}$ solves the same ODE by~\eqref{eq:mirror_cancellation} and agrees with $X_t$ at $t = \tfrac{1}{2}$.
The transport from $\rho_0$ to $\rho_1$ is then the identity map, which makes the optimal-transport intuition of \cref{sec:intro} exact for this construction.

\Cref{prop:invariance} has two consequences.
Under $\mathcal{H}_0$, the path fields coincide at every radius.
Under $\mathcal{H}_1$, they differ at every radius, so any single matched noise level already carries two-sample evidence.
This motivates testing $\mathcal{H}_0$ through discrepancies between the fields of the two channels.

\subsection{Held-Out Risk Witnesses}
\label{sec:witnesses}

Suppose that $\widehat{\eta}_s^0, \widehat{u}_s^0$ and $\widehat{\eta}_s^1, \widehat{u}_s^1$ are estimated on a training split that is independent of the observations used for testing.
For a held-out observation $\omega \in \mathbb{R}^d$, we draw an independent latent $z \sim \mathcal{N}(0, \mathbb{I}_d)$ and form the corrupted point and its exact outward derivative,
\begin{equation}
\label{eq:heldout_corruption}
    \omega_s = a_s\, \omega + \sigma_s\, z,
    \qquad
    U_s(\omega, z) = a_s'\, \omega + \sigma_s'\, z .
\end{equation}
We define the denoising-risk and velocity-risk contrasts
\begin{align}
    \psi_s^{\mathrm{den}}(\omega, z)
    &=
    \big| \omega - \widehat{\eta}_s^1(\omega_s) \big|^2
    -
    \big| \omega - \widehat{\eta}_s^0(\omega_s) \big|^2 ,
    \label{eq:denoising_local_witness} \\
    \psi_s^{\mathrm{vel}}(\omega, z)
    &=
    \big| U_s(\omega, z) - \widehat{u}_s^1(\omega_s) \big|^2
    -
    \big| U_s(\omega, z) - \widehat{u}_s^0(\omega_s) \big|^2 .
    \label{eq:velocity_local_witness}
\end{align}
Each contrast is a single scalar function evaluated identically on held-out observations from both datasets.
The group label of $\omega$ is never supplied to the witness; the labels $0$ and $1$ in~\eqref{eq:denoising_local_witness} and~\eqref{eq:velocity_local_witness} refer only to the two learned fields.
Throughout, the latent $z \sim \mathcal{N}(0, \mathbb{I}_d)$ is drawn independently of all other variables. We write, for instance, $\mathbb{E}_{\omega \sim \rho_0,\, z}$ for the expectation over a held-out observation $\omega \sim \rho_0$ and the latent.

\begin{lemma}[Population risk contrasts]
\label{lem:risk_contrasts}
Suppose \cref{ass:moments} holds and that the learned fields equal their population counterparts.
Then for every $s \in (0, \tfrac{1}{2})$,
\begin{align}
    \mathbb{E}_{\omega \sim \rho_0,\, z}\big[\psi_s^{\mathrm{den}}(\omega, z)\big]
    &=
    \mathbb{E}_{x \sim \zeta_s^0}\big[| \eta_s^0(x) - \eta_s^1(x) |^2\big],
    &
    \mathbb{E}_{\omega \sim \rho_1,\, z}\big[\psi_s^{\mathrm{den}}(\omega, z)\big]
    &=
    -\mathbb{E}_{x \sim \zeta_s^1}\big[| \eta_s^0(x) - \eta_s^1(x) |^2\big],
    \label{eq:den_identities} \\
    \mathbb{E}_{\omega \sim \rho_0,\, z}\big[\psi_s^{\mathrm{vel}}(\omega, z)\big]
    &=
    \mathbb{E}_{x \sim \zeta_s^0}\big[| u_s^0(x) - u_s^1(x) |^2\big],
    &
    \mathbb{E}_{\omega \sim \rho_1,\, z}\big[\psi_s^{\mathrm{vel}}(\omega, z)\big]
    &=
    -\mathbb{E}_{x \sim \zeta_s^1}\big[| u_s^0(x) - u_s^1(x) |^2\big],
    \label{eq:vel_identities}
\end{align}
where both sides may be infinite.
Consequently, the population differences of held-out means,
\begin{equation}
\label{eq:signals}
    D_s^{\mathrm{rep}}
    =
    \mathbb{E}_{\omega \sim \rho_0,\, z}\big[\psi_s^{\mathrm{rep}}(\omega, z)\big]
    -
    \mathbb{E}_{\omega \sim \rho_1,\, z}\big[\psi_s^{\mathrm{rep}}(\omega, z)\big],
    \qquad
    \mathrm{rep} \in \{\mathrm{den}, \mathrm{vel}\},
\end{equation}
take values in $[0, \infty]$ and satisfy $D_s^{\mathrm{vel}} = \lambda_s^2 D_s^{\mathrm{den}}$.
Moreover, $D_s^{\mathrm{rep}} = 0$ if and only if $\mathcal{H}_0$ holds.
\end{lemma}

\begin{proof}
The identities follow from the orthogonality of regression residuals, and the final claim follows from \cref{prop:invariance}.
Fix $s \in (0, \tfrac{1}{2})$, and consider first the denoising contrast for $\omega \sim \rho_0$; in this paragraph, $\mathbb{E}$ denotes the expectation over $\omega \sim \rho_0$ and $z$.
By \cref{lem:bottleneck}, the pair $(\omega, \omega_s)$ has the same joint law as $(x_0, x_{t_s^0})$, so $\omega_s \sim \zeta_s^0$ and $\eta_s^0(\omega_s) = \mathbb{E}[\omega \mid \omega_s]$.
Write $e = \omega - \eta_s^0(\omega_s)$ and $g = \eta_s^0(\omega_s) - \eta_s^1(\omega_s)$, so that $\omega - \eta_s^1(\omega_s) = e + g$ and $\psi_s^{\mathrm{den}}(\omega, z) = |e + g|^2 - |e|^2$.
Since $e$ is the residual of a conditional expectation, $\mathbb{E}|e|^2 = \mathbb{E}|\omega|^2 - \mathbb{E}|\eta_s^0(\omega_s)|^2 \leq \mathbb{E}_{\omega \sim \rho_0}|\omega|^2 < \infty$ under \cref{ass:moments}.
The negative part of $\psi_s^{\mathrm{den}}$ is therefore integrable, and its expectation is well defined in $(-\infty, \infty]$.
If $\mathbb{E}|g|^2 < \infty$, then the cross term $\mathbb{E}[e \cdot g]$ is integrable by the Cauchy--Schwarz inequality and vanishes by the tower property, because $g$ is a function of $\omega_s$ and $\mathbb{E}[e \mid \omega_s] = 0$.
Expanding the square then gives
\begin{equation}
    \mathbb{E}\big[\psi_s^{\mathrm{den}}(\omega, z)\big]
    =
    \mathbb{E}|g|^2
    =
    \mathbb{E}_{x \sim \zeta_s^0}\big[|\eta_s^0(x) - \eta_s^1(x)|^2\big],
\end{equation}
which is the first identity in~\eqref{eq:den_identities}.
If instead $\mathbb{E}|g|^2 = \infty$, then $e + g$ is not square integrable, because $e$ is, and both sides equal $+\infty$.

For $\omega \sim \rho_1$, the pair $(\omega, \omega_s)$ has the same law as $(x_1, x_{t_s^1})$, so $\omega_s \sim \zeta_s^1$.
The same argument with the roles of the two denoisers exchanged gives $\mathbb{E}_{\omega \sim \rho_1,\, z}[\psi_s^{\mathrm{den}}(\omega, z)] = -\mathbb{E}_{x \sim \zeta_s^1}[|\eta_s^0(x) - \eta_s^1(x)|^2] \in [-\infty, 0]$, which is the second identity in~\eqref{eq:den_identities}.
The velocity identities~\eqref{eq:vel_identities} follow identically, replacing $\omega$ by the target $U_s(\omega, z)$.
This target is square integrable under \cref{ass:moments}, and $u_s^\ell(\omega_s) = \mathbb{E}[U_s(\omega, z) \mid \omega_s]$ when $\omega \sim \rho_\ell$.

Since the first identity in each row is non-negative and the second is non-positive, their difference $D_s^{\mathrm{rep}}$ is well defined in $[0, \infty]$.
By \Cref{lem:relations}~(iii), $|u_s^0 - u_s^1|^2 = \lambda_s^2 |\eta_s^0 - \eta_s^1|^2$ pointwise, so $D_s^{\mathrm{vel}} = \lambda_s^2 D_s^{\mathrm{den}}$.
If $\mathcal{H}_0$ holds, then $\eta_s^0 = \eta_s^1$ by \cref{prop:invariance}, and $D_s^{\mathrm{den}} = 0$.
Conversely, if $D_s^{\mathrm{den}} = 0$, then $\eta_s^0 = \eta_s^1$ holds $\zeta_s^0$-almost everywhere.
This equality then holds Lebesgue-almost everywhere because $\zeta_s^0 > 0$, and everywhere because both fields are continuous.
\Cref{prop:invariance} then gives $\mathcal{H}_0$.
The same conclusion holds for $D_s^{\mathrm{vel}}$, since $\lambda_s > 0$.
\end{proof}

\subsection{Path Aggregation}
\label{sec:aggregation}

We aggregate the local contrasts over the truncated interval $\mathcal{I}_\tau = [\tau, \tfrac{1}{2} - \tau]$, for some $\tau \in (0, \tfrac{1}{4})$, using a positive weight $w(s)$,
\begin{equation}
\label{eq:aggregated_witness}
    \Phi_w^{\mathrm{rep}}(\omega)
    =
    \mathbb{E}_{z}\!\left[ \int_{\mathcal{I}_\tau} w(s)\, \psi_s^{\mathrm{rep}}(\omega, z)\, \mathrm{d}s \right],
    \qquad
    \mathrm{rep} \in \{\mathrm{den}, \mathrm{vel}\}.
\end{equation}
We consider two weighting schemes.
The uniform scheme sets $w_{\mathrm{U}}(s) = 1$ and assigns equal path measure to all retained radii.
The Jeffreys scheme uses the information-theoretically canonical weights
\begin{equation}
\label{eq:jeff_weights}
    w_{\mathrm{J}}^{\mathrm{den}}(s) = \frac{\lambda_s^2}{\chi_s} = \frac{\lambda_s\, a_s}{\sigma_s^2},
    \qquad
    w_{\mathrm{J}}^{\mathrm{vel}}(s) = \frac{1}{\chi_s},
\end{equation}
whose origin is made precise in \cref{prop:jeffreys}.
Given held-out subsets of $\mathcal{D}_0$ and $\mathcal{D}_1$ of sizes $n_0^{\mathrm{te}}$ and $n_1^{\mathrm{te}}$, the empirical statistic is the difference of held-out group means,
\begin{equation}
\label{eq:T_generic_weighted}
    \widehat{T}_w^{\mathrm{rep}}
    =
    \frac{1}{n_0^{\mathrm{te}}} \sum_{i=1}^{n_0^{\mathrm{te}}} \widehat{\Phi}_w^{\mathrm{rep}}\big(x_0^{(i)}\big)
    -
    \frac{1}{n_1^{\mathrm{te}}} \sum_{j=1}^{n_1^{\mathrm{te}}} \widehat{\Phi}_w^{\mathrm{rep}}\big(x_1^{(j)}\big),
\end{equation}
where $\widehat{\Phi}_w^{\mathrm{rep}}$ denotes the Monte Carlo estimate of \cref{sec:path_mc}.
Normalising $w$ by a constant is immaterial, since a deterministic rescaling does not affect the permutation test of \cref{sec:permutation}.

\begin{corollary}[Population target]
\label{cor:target}
Suppose the conditions of \cref{lem:risk_contrasts} hold, that the sampling density $\mu$ of \cref{sec:path_mc} is strictly positive on $\mathcal{I}_\tau$, and that $\int_{\mathcal{I}_\tau} w(s)\, \mathbb{E}_{\omega \sim \rho_\ell,\, z}\big[|\psi_s^{\mathrm{rep}}(\omega, z)|\big]\, \mathrm{d}s < \infty$ for $\ell \in \{0, 1\}$.
Then the population mean of $\widehat{T}_w^{\mathrm{rep}}$ is
\begin{equation}
\label{eq:population_target}
    T_w^{\mathrm{rep}} = \int_{\mathcal{I}_\tau} w(s)\, D_s^{\mathrm{rep}}\, \mathrm{d}s \geq 0,
\end{equation}
with equality if and only if $\mathcal{H}_0$ holds.
\end{corollary}

\begin{proof}
We compute the mean of a single Monte Carlo term, then apply linearity.
Let $\omega \sim \rho_\ell$, $s \sim \mu$ and $z \sim \mathcal{N}(0, \mathbb{I}_d)$ be independent.
By Tonelli's theorem, the integrability condition in the statement gives
\begin{equation}
    \mathbb{E}_{\omega \sim \rho_\ell,\, s \sim \mu,\, z}\!\left[ \frac{w(s)}{\mu(s)}\, \big|\psi_s^{\mathrm{rep}}(\omega, z)\big| \right]
    =
    \int_{\mathcal{I}_\tau} w(s)\, \mathbb{E}_{\omega \sim \rho_\ell,\, z}\big[|\psi_s^{\mathrm{rep}}(\omega, z)|\big]\, \mathrm{d}s
    < \infty ,
\end{equation}
where we used $\mu > 0$ on $\mathcal{I}_\tau$.
Fubini's theorem then gives
\begin{equation}
    \mathbb{E}_{\omega \sim \rho_\ell,\, s \sim \mu,\, z}\!\left[ \frac{w(s)}{\mu(s)}\, \psi_s^{\mathrm{rep}}(\omega, z) \right]
    =
    \int_{\mathcal{I}_\tau} w(s)\, \mathbb{E}_{\omega \sim \rho_\ell,\, z}\big[\psi_s^{\mathrm{rep}}(\omega, z)\big]\, \mathrm{d}s ,
\end{equation}
and the inner expectations are finite for almost every $s \in \mathcal{I}_\tau$.
Each term of~\eqref{eq:mc_witness} has this mean for an observation from $\rho_\ell$, so by linearity the mean of $\widehat{T}_w^{\mathrm{rep}}$ is
\begin{equation}
    \int_{\mathcal{I}_\tau} w(s) \Big( \mathbb{E}_{\omega \sim \rho_0,\, z}\big[\psi_s^{\mathrm{rep}}(\omega, z)\big] - \mathbb{E}_{\omega \sim \rho_1,\, z}\big[\psi_s^{\mathrm{rep}}(\omega, z)\big] \Big)\, \mathrm{d}s
    =
    \int_{\mathcal{I}_\tau} w(s)\, D_s^{\mathrm{rep}}\, \mathrm{d}s .
\end{equation}
The integrand is non-negative by \cref{lem:risk_contrasts}, and $w > 0$, so the integral vanishes if and only if $D_s^{\mathrm{rep}} = 0$ for almost every $s \in \mathcal{I}_\tau$.
If $\mathcal{H}_0$ holds, then $D_s^{\mathrm{rep}} = 0$ for every $s$.
Conversely, since $\tau < \tfrac{1}{4}$, the interval $\mathcal{I}_\tau$ has positive length, so if $D_s^{\mathrm{rep}} = 0$ for almost every $s \in \mathcal{I}_\tau$, it vanishes for at least one $s$, and \cref{lem:risk_contrasts} gives $\mathcal{H}_0$.
\end{proof}

We now make precise the origin of the Jeffreys weights~\eqref{eq:jeff_weights}.
The identity below relies on the following technical condition, which is used nowhere else in the paper.

\begin{assumption}[Regularity of the divergences]
\label{ass:regularity}
For every $S \in (0, \tfrac{1}{2})$, the maps $s \mapsto \mathrm{KL}(\zeta_s^0 \,\|\, \zeta_s^1)$ and $s \mapsto \mathrm{KL}(\zeta_s^1 \,\|\, \zeta_s^0)$ are absolutely continuous on every compact subinterval of $(0, S]$. Their derivatives may be computed by differentiating under the integral sign and integrating by parts with vanishing boundary terms.
\end{assumption}

For a density $\zeta$ on $\mathbb{R}^d$ and a measurable map $f \colon \mathbb{R}^d \to \mathbb{R}^d$, we write
\begin{equation}
\label{eq:l2_norm}
    \| f \|_{L^2(\zeta)} = \Big( \mathbb{E}_{x \sim \zeta}\big[|f(x)|^2\big] \Big)^{1/2}.
\end{equation}

\begin{proposition}[Jeffreys identity]
\label{prop:jeffreys}
Under \cref{ass:moments,ass:regularity}, for every $S \in (0, \tfrac{1}{2})$,
\begin{equation}
\label{eq:jeffreys_identity}
    \boxed{
    \mathcal{J}_S
    :=
    \int_0^S \frac{D_s^{\mathrm{vel}}}{\chi_s}\, \mathrm{d}s
    =
    \int_0^S w_{\mathrm{J}}^{\mathrm{den}}(s)\, D_s^{\mathrm{den}}\, \mathrm{d}s
    =
    \mathrm{KL}\big(\zeta_S^0 \,\big\|\, \zeta_S^1\big)
    +
    \mathrm{KL}\big(\zeta_S^1 \,\big\|\, \zeta_S^0\big).
    }
\end{equation}
\end{proposition}

\begin{proof}
The middle equality in~\eqref{eq:jeffreys_identity} follows from \cref{lem:risk_contrasts}, since $w_{\mathrm{J}}^{\mathrm{den}}(s)\, D_s^{\mathrm{den}} = (\lambda_s^2 / \chi_s)\, D_s^{\mathrm{den}} = D_s^{\mathrm{vel}} / \chi_s$.
For the remaining equality, we derive continuity equations for the two channels, differentiate the resulting divergences, and integrate from the pivot.
Within this proof, we write $F_s = \mathrm{KL}(\zeta_s^0 \| \zeta_s^1)$ and $G_s = \mathrm{KL}(\zeta_s^1 \| \zeta_s^0)$.

We begin with the continuity equations.
Let $f \in C_c^\infty(\mathbb{R}^d)$, and let $s$ range over a compact subinterval of $(0, \tfrac{1}{2})$.
The map $s \mapsto x_{t_s^\ell}$ is differentiable with derivative $\partial_s x_{t_s^\ell} = a_s' x_\ell + \sigma_s' z$.
This derivative is dominated by an integrable variable, uniformly in $s$, because $a_s'$ and $\sigma_s'$ are bounded on the subinterval and $x_\ell$ and $z$ are integrable.
Since $\nabla f$ is bounded, we may differentiate under the expectation, and conditioning on $x_{t_s^\ell}$ gives
\begin{equation}
    \frac{\mathrm{d}}{\mathrm{d}s}\, \mathbb{E}_{x \sim \zeta_s^\ell}\big[f(x)\big]
    =
    \mathbb{E}\big[\nabla f(x_{t_s^\ell}) \cdot \partial_s x_{t_s^\ell}\big]
    =
    \mathbb{E}_{x \sim \zeta_s^\ell}\big[\nabla f(x) \cdot u_s^\ell(x)\big] .
\end{equation}
Hence $\partial_s \zeta_s^0 + \nabla \cdot (u_s^0\, \zeta_s^0) = 0$ and $\partial_s \zeta_s^1 + \nabla \cdot (u_s^1\, \zeta_s^1) = 0$ in the weak sense.

Next, we differentiate the divergences.
Fix $s \in (0, \tfrac{1}{2})$.
\Cref{ass:regularity} permits differentiating $F_s = \int \zeta_s^0 \log(\zeta_s^0 / \zeta_s^1)\, \mathrm{d}x$ under the integral sign and testing the continuity equations above against the functions $\log(\zeta_s^0 / \zeta_s^1)$ and $\zeta_s^0 / \zeta_s^1$, which are not compactly supported.
Differentiating gives
\begin{equation}
    \partial_s F_s
    =
    \int \partial_s \zeta_s^0\, \log \frac{\zeta_s^0}{\zeta_s^1}\, \mathrm{d}x
    +
    \int \partial_s \zeta_s^0\, \mathrm{d}x
    -
    \int \frac{\zeta_s^0}{\zeta_s^1}\, \partial_s \zeta_s^1\, \mathrm{d}x .
\end{equation}
The middle term vanishes because $\zeta_s^0$ integrates to one for all $s$.
Substituting the continuity equations and integrating by parts, the first term becomes $\mathbb{E}_{x \sim \zeta_s^0}[u_s^0(x) \cdot \nabla \log(\zeta_s^0 / \zeta_s^1)(x)]$.
Similarly, the third term becomes $-\int \zeta_s^1\, u_s^1 \cdot \nabla(\zeta_s^0 / \zeta_s^1)\, \mathrm{d}x = -\mathbb{E}_{x \sim \zeta_s^0}[u_s^1(x) \cdot \nabla \log(\zeta_s^0 / \zeta_s^1)(x)]$, using $\zeta_s^1 \nabla(\zeta_s^0 / \zeta_s^1) = \zeta_s^0 \nabla \log(\zeta_s^0 / \zeta_s^1)$.
Hence
\begin{equation}
    \partial_s F_s = \mathbb{E}_{x \sim \zeta_s^0}\Big[\big(u_s^0(x) - u_s^1(x)\big) \cdot \nabla \log \frac{\zeta_s^0}{\zeta_s^1}(x)\Big].
\end{equation}
By \Cref{lem:relations}~(iii), $\nabla \log(\zeta_s^0 / \zeta_s^1) = (u_s^0 - u_s^1)/\chi_s$, so $\partial_s F_s = \| u_s^0 - u_s^1 \|_{L^2(\zeta_s^0)}^2 / \chi_s$.
Exchanging the roles of the two channels gives $\partial_s G_s = \| u_s^0 - u_s^1 \|_{L^2(\zeta_s^1)}^2 / \chi_s$, and adding the two yields
\begin{equation}
\label{eq:jeffreys_derivative}
    \partial_s (F_s + G_s) = \frac{D_s^{\mathrm{vel}}}{\chi_s} \geq 0 .
\end{equation}

It remains to control the divergences near the pivot.
Write $\zeta_s^0$ and $\zeta_s^1$ as mixtures of $\mathcal{N}(a_s y, \sigma_s^2 \mathbb{I}_d)$ and $\mathcal{N}(a_s y', \sigma_s^2 \mathbb{I}_d)$, respectively, over the same mixing law $(y, y') \sim \rho_0 \otimes \rho_1$.
By joint convexity of the Kullback--Leibler divergence,
\begin{equation}
\label{eq:kl_bound}
    F_s
    \leq
    \mathbb{E}_{(y, y') \sim \rho_0 \otimes \rho_1}\Big[\mathrm{KL}\big(\mathcal{N}(a_s y, \sigma_s^2 \mathbb{I}_d) \,\big\|\, \mathcal{N}(a_s y', \sigma_s^2 \mathbb{I}_d)\big)\Big]
    =
    \frac{a_s^2}{2\sigma_s^2}\, \mathbb{E}_{(y, y') \sim \rho_0 \otimes \rho_1}\big[|y - y'|^2\big],
\end{equation}
and the same bound holds for $G_s$.
The expectation on the right is finite under \cref{ass:moments}.
Since $a_s \to 0$ and $\sigma_s^2 \to \tfrac{1}{2}$ as $s \downarrow 0$, both $F_s$ and $G_s$ tend to zero, and both are finite for every $s < \tfrac{1}{2}$.

To conclude, let $0 < \epsilon < S < \tfrac{1}{2}$.
By \cref{ass:regularity}, $F_s + G_s$ is absolutely continuous on $[\epsilon, S]$, so integrating~\eqref{eq:jeffreys_derivative} gives
\begin{equation}
    F_S + G_S - F_\epsilon - G_\epsilon = \int_\epsilon^S \frac{D_s^{\mathrm{vel}}}{\chi_s}\, \mathrm{d}s .
\end{equation}
As $\epsilon \downarrow 0$, the left-hand side converges to $F_S + G_S$ by~\eqref{eq:kl_bound}.
The right-hand side converges to $\mathcal{J}_S$ by monotone convergence, since the integrand is non-negative.
This proves~\eqref{eq:jeffreys_identity}.
\end{proof}

The Jeffreys-weighted statistics thus accumulate the local path discrepancy into a global information divergence between the two channels.
By \cref{cor:target}, their population target on $\mathcal{I}_\tau$ is $\mathcal{J}_{1/2 - \tau} - \mathcal{J}_\tau$.

\begin{remark}
\label{rem:jeffreys_limit}
Under the assumptions of \cref{prop:jeffreys}, the divergence $\mathcal{J}_S$ is non-decreasing in $S$, and as $S \uparrow \tfrac{1}{2}$ it converges in $[0, \infty]$ to the Jeffreys divergence $\mathrm{KL}(\rho_0 \| \rho_1) + \mathrm{KL}(\rho_1 \| \rho_0)$.
\end{remark}

\begin{proof}
Monotonicity is immediate, since $\mathcal{J}_S$ is the integral of a non-negative function over $(0, S]$.
For the limit, it suffices by \cref{prop:jeffreys} to show that $\mathrm{KL}(\zeta_S^0 \| \zeta_S^1) \to \mathrm{KL}(\rho_0 \| \rho_1)$, since the reverse divergence is treated identically.
The densities $\zeta_S^0$ and $\zeta_S^1$ are the images of $\rho_0$ and $\rho_1$ under the same Markov kernel $y \mapsto \mathcal{N}(a_S y, \sigma_S^2 \mathbb{I}_d)$.
The data-processing inequality therefore gives $\mathrm{KL}(\zeta_S^0 \| \zeta_S^1) \leq \mathrm{KL}(\rho_0 \| \rho_1)$ for every $S$.
As $S \uparrow \tfrac{1}{2}$, we have $a_S \to 1$ and $\sigma_S \to 0$, so $x_{t_S^\ell} \to x_\ell$ in $L^2$ and hence in distribution.
Since the Kullback--Leibler divergence is jointly lower semicontinuous with respect to weak convergence, $\liminf_{S \uparrow 1/2} \mathrm{KL}(\zeta_S^0 \| \zeta_S^1) \geq \mathrm{KL}(\rho_0 \| \rho_1)$.
Combining the two bounds gives convergence in $[0, \infty]$.
\end{proof}

\subsection{Held-Out Permutation Calibration}
\label{sec:permutation}

We split each dataset into a training part and a held-out part, $\mathcal{D}_\ell = \mathcal{D}_\ell^{\mathrm{tr}} \cup \mathcal{D}_\ell^{\mathrm{te}}$ for $\ell \in \{0, 1\}$.
The fields $\widehat{\eta}^\ell$ and $\widehat{u}^\ell$ are fitted on $\mathcal{D}_\ell^{\mathrm{tr}}$ alone, and all statistics are evaluated on $\mathcal{D}^{\mathrm{te}} = (\mathcal{D}_0^{\mathrm{te}}, \mathcal{D}_1^{\mathrm{te}})$.
Conditionally on the training data and on the randomness of training, the learned fields are fixed, and every held-out observation is mapped through the same scalar witness $\widehat{\Phi}$.
Writing $\widehat{\Phi}_i^0 = \widehat{\Phi}(x_0^{(i)})$ and $\widehat{\Phi}_j^1 = \widehat{\Phi}(x_1^{(j)})$ for the held-out scores, we use the generic statistic
\begin{equation}
\label{eq:generic_test_stat}
    \widehat{T}
    =
    \frac{1}{n_0^{\mathrm{te}}} \sum_{i=1}^{n_0^{\mathrm{te}}} \widehat{\Phi}_i^0
    -
    \frac{1}{n_1^{\mathrm{te}}} \sum_{j=1}^{n_1^{\mathrm{te}}} \widehat{\Phi}_j^1 .
\end{equation}
For each permutation $r = 1, \dots, R$, the pooled scores are randomly relabelled while preserving the group sizes, yielding $\widehat{T}^{(r)}$.
The two-sided permutation $p$-value is
\begin{equation}
\label{eq:permutation_pvalue}
    \widehat{p}
    =
    \frac{1 + \sum_{r=1}^{R} \mathbf{1}\big\{ |\widehat{T}^{(r)}| \geq |\widehat{T}| \big\}}{R + 1},
\end{equation}
and the test is $\Delta(\mathcal{D}_0, \mathcal{D}_1) = \mathbf{1}\{\widehat{p} \leq \kappa\}$.

\begin{proposition}[Finite-sample validity]
\label{prop:validity}
Let $\mathcal{G}$ denote the training data together with all randomness used in training.
Suppose that the held-out data are independent of $\mathcal{G}$, that the Monte Carlo draws are i.i.d.\ across observations and independent of the data and of the group labels, and that the permutations are drawn independently of everything else.
Then, under $\mathcal{H}_0$, for any fitted fields,
\begin{equation}
\label{eq:conditional_validity}
    \mathbb{P}\big[\widehat{p} \leq \kappa \,\big|\, \mathcal{G}\big] \leq \kappa .
\end{equation}
In particular, $\Delta$ has level $\kappa$ in the sense of~\eqref{eq:level}.
\end{proposition}

\begin{proof}
We first show that the held-out scores are conditionally exchangeable given $\mathcal{G}$, then that the observed statistic is exchangeable with its permuted copies, and we conclude with a counting argument.

Write $n^{\mathrm{te}} = n_0^{\mathrm{te}} + n_1^{\mathrm{te}}$, and let $W = (W_1, \dots, W_{n^{\mathrm{te}}})$ be the pooled vector of held-out scores, with the first $n_0^{\mathrm{te}}$ entries from $\mathcal{D}_0^{\mathrm{te}}$.
Each score is obtained by applying a $\mathcal{G}$-measurable function to one held-out observation and its own Monte Carlo draws.
Under $\mathcal{H}_0$, the held-out observations are i.i.d.\ from $\rho_0$ and independent of $\mathcal{G}$, and the Monte Carlo draws are i.i.d.\ across observations and independent of the group labels.
Hence, conditionally on $\mathcal{G}$, the entries of $W$ are i.i.d., and in particular $W$ is exchangeable.
No moment condition is needed here, since the learned fields are defined for any data.

To pass from the scores to the statistics, write $W_g = (W_{g(1)}, \dots, W_{g(n^{\mathrm{te}})})$ for a permutation $g$ of $\{1, \dots, n^{\mathrm{te}}\}$, so that $(W_g)_{g'} = W_{g \circ g'}$.
Let $h(W)$ denote the difference of group means~\eqref{eq:generic_test_stat}.
Let $g_0$ be the identity, and let $g_1, \dots, g_R$ be the i.i.d.\ uniform permutations used by the test, so that $\widehat{T}^{(r)} = h(W_{g_r})$.
Draw an auxiliary uniform permutation $g_\ast$, independent of everything else, and set $V = W_{g_\ast}$.
By exchangeability, $V$ has the same conditional law as $W$ and is independent of $(g_\ast, g_1, \dots, g_R)$.
Moreover, $W_{g_r} = V_{\tilde{g}_r}$ with $\tilde{g}_0 = g_\ast^{-1}$ and $\tilde{g}_r = g_\ast^{-1} \circ g_r$ for $r \geq 1$.
The map $(g_\ast, g_1, \dots, g_R) \mapsto (\tilde{g}_0, \dots, \tilde{g}_R)$ is a bijection, so $\tilde{g}_0, \dots, \tilde{g}_R$ are i.i.d.\ uniform and independent of $V$.
Consequently, the vector $(A_0, \dots, A_R)$ with $A_r = |h(V_{\tilde{g}_r})| = |\widehat{T}^{(r)}|$ is exchangeable conditionally on $\mathcal{G}$, where $\widehat{T}^{(0)} = \widehat{T}$.

Finally, for $j \in \{0, \dots, R\}$, define $\widehat{p}_j = (R+1)^{-1} \#\{r : A_r \geq A_j\}$, so that $\widehat{p} = \widehat{p}_0$.
Let $J = \{j : \widehat{p}_j \leq \kappa\}$.
If $J$ is non-empty, choose $j^\ast \in J$ minimising $A_{j^\ast}$.
Every $j \in J$ then satisfies $A_j \geq A_{j^\ast}$, so $|J| \leq \#\{r : A_r \geq A_{j^\ast}\} \leq \kappa (R+1)$.
By exchangeability, each $\widehat{p}_j$ has the same conditional law, and therefore
\begin{equation}
    \mathbb{P}\big[\widehat{p} \leq \kappa \,\big|\, \mathcal{G}\big]
    =
    \frac{1}{R+1} \sum_{j=0}^{R} \mathbb{P}\big[\widehat{p}_j \leq \kappa \,\big|\, \mathcal{G}\big]
    =
    \frac{\mathbb{E}\big[|J| \,\big|\, \mathcal{G}\big]}{R+1}
    \leq
    \kappa.
\end{equation}
This is~\eqref{eq:conditional_validity}.
Taking expectations over $\mathcal{G}$ shows that the unconditional rejection probability is at most $\kappa$ for every $\rho_0 \in \mathcal{P}(\mathbb{R}^d)$, which is the level condition~\eqref{eq:level}.
\end{proof}

\Cref{prop:validity} requires neither moment assumptions nor any accuracy of the learned fields.
Consistency of the learned denoisers or velocities is required for power, through \cref{cor:target}, but not for validity.

\subsection{Consistency}
\label{sec:consistency}

\cref{prop:validity} holds for arbitrary fitted fields, so the quality of the learned fields enters only through power.
To quantify this, we measure the error of each learned field under both channel densities,
\begin{equation}
\label{eq:cross_errors}
    \varepsilon_{\ell | k}(s) = \Big( \mathbb{E}_{x \sim \zeta_s^k}\big[|\widehat{\eta}_s^\ell(x) - \eta_s^\ell(x)|^2\big] \Big)^{1/2},
    \qquad
    \ell, k \in \{0, 1\},
\end{equation}
and we denote by $\widehat{D}_s^{\mathrm{rep}}$ the contrast~\eqref{eq:signals} computed with the learned fields in place of the population ones.
The diagonal errors $\varepsilon_{0|0}$ and $\varepsilon_{1|1}$ measure how well each field is learned on its own channel.
The cross errors $\varepsilon_{1|0}$ and $\varepsilon_{0|1}$ measure how well each field extrapolates to the other channel, on which the witness also evaluates it.

\begin{lemma}[Risk contrasts with learned fields]
\label{lem:learned_gap}
Suppose \cref{ass:moments} holds, fix $s \in (0, \tfrac{1}{2})$, and let the learned denoisers be fixed and square integrable under $\zeta_s^0$ and $\zeta_s^1$.
Then
\begin{equation}
\label{eq:learned_gap}
    \widehat{D}_s^{\mathrm{den}}
    =
    \Big( \mathbb{E}_{x \sim \zeta_s^0}\big[|\eta_s^0(x) - \widehat{\eta}_s^1(x)|^2\big] - \varepsilon_{0|0}(s)^2 \Big)
    +
    \Big( \mathbb{E}_{x \sim \zeta_s^1}\big[|\eta_s^1(x) - \widehat{\eta}_s^0(x)|^2\big] - \varepsilon_{1|1}(s)^2 \Big).
\end{equation}
In particular, $\widehat{D}_s^{\mathrm{den}} = 0$ under $\mathcal{H}_0$ for any learned fields, and
\begin{equation}
\label{eq:learned_gap_bound}
    \big| \widehat{D}_s^{\mathrm{den}} - D_s^{\mathrm{den}} \big|
    \leq
    2 \sqrt{D_s^{\mathrm{den}}}\, \big( \varepsilon_{1|0}(s) + \varepsilon_{0|1}(s) \big)
    +
    \sum_{\ell, k \in \{0, 1\}} \varepsilon_{\ell|k}(s)^2 .
\end{equation}
The same statements hold for the velocity contrast, with $u$ in place of $\eta$ throughout.
\end{lemma}

Each bracket in~\eqref{eq:learned_gap} is positive when, on the corresponding channel, the matching learned field is closer to the population field than the mismatched one.
The test, therefore, retains power whenever each model fits its own population better than the other, which is considerably weaker than the consistency of either model.

\begin{proof}[Proof of \cref{lem:learned_gap}]
The identity follows from the orthogonality argument of \cref{lem:risk_contrasts}, now applied to the learned fields, and the bound follows by expanding the squares around the population fields.
Fix $s \in (0, \tfrac{1}{2})$, and abbreviate $\| \cdot \|_k = \| \cdot \|_{L^2(\zeta_s^k)}$ for $k \in \{0, 1\}$.

Let $\omega \sim \rho_0$ and $z$ be independent, let $\mathbb{E}$ denote the expectation over both, and write $e = \omega - \eta_s^0(\omega_s)$.
As in the proof of \cref{lem:risk_contrasts}, $\omega_s \sim \zeta_s^0$ and $\mathbb{E}[e \mid \omega_s] = 0$, and both $e$ and $\eta_s^0(\omega_s)$ are square integrable, with second moments at most $\mathbb{E}_{\omega \sim \rho_0}|\omega|^2$.
For any measurable $f$ with $\| f \|_0 < \infty$, the variable $(\eta_s^0 - f)(\omega_s)$ is therefore square integrable and a function of $\omega_s$.
The tower property then gives
\begin{equation}
    \mathbb{E}\big|\omega - f(\omega_s)\big|^2
    =
    \mathbb{E}|e|^2 + \| \eta_s^0 - f \|_0^2 .
\end{equation}
Applying this with $f = \widehat{\eta}_s^1$ and $f = \widehat{\eta}_s^0$, which are square integrable under $\zeta_s^0$ by assumption, and subtracting, we obtain
\begin{equation}
    \mathbb{E}_{\omega \sim \rho_0,\, z}\big[\psi_s^{\mathrm{den}}(\omega, z)\big]
    =
    \| \eta_s^0 - \widehat{\eta}_s^1 \|_0^2 - \varepsilon_{0|0}(s)^2 .
\end{equation}
The same argument for $\omega \sim \rho_1$, with residual $\omega - \eta_s^1(\omega_s)$, gives
\begin{equation}
    \mathbb{E}_{\omega \sim \rho_1,\, z}\big[\psi_s^{\mathrm{den}}(\omega, z)\big]
    =
    \varepsilon_{1|1}(s)^2 - \| \eta_s^1 - \widehat{\eta}_s^0 \|_1^2 .
\end{equation}
All terms are finite, and subtracting the second identity from the first yields~\eqref{eq:learned_gap}.

Under $\mathcal{H}_0$, \cref{prop:invariance} gives $\zeta_s^0 = \zeta_s^1$ and $\eta_s^0 = \eta_s^1 =: \eta_s$, so the two norms coincide and we write $\| \cdot \|$ for both.
The identity~\eqref{eq:learned_gap} then reads
\begin{equation}
    \widehat{D}_s^{\mathrm{den}}
    =
    \big( \| \eta_s - \widehat{\eta}_s^1 \|^2 - \| \eta_s - \widehat{\eta}_s^0 \|^2 \big)
    +
    \big( \| \eta_s - \widehat{\eta}_s^0 \|^2 - \| \eta_s - \widehat{\eta}_s^1 \|^2 \big)
    =
    0 .
\end{equation}

If any of the four errors is infinite, the right-hand side of~\eqref{eq:learned_gap_bound} is infinite and there is nothing to prove, so we assume all four are finite.
Write $A = \eta_s^0 - \eta_s^1$ and $B_\ell = \widehat{\eta}_s^\ell - \eta_s^\ell$, so that $\| B_1 \|_0 = \varepsilon_{1|0}(s)$ and $\| B_0 \|_1 = \varepsilon_{0|1}(s)$.
Then $\eta_s^0 - \widehat{\eta}_s^1 = A - B_1$ and $\eta_s^1 - \widehat{\eta}_s^0 = -(A + B_0)$.
The norms of $A$ are finite, since $\| A \|_0 \leq \| \eta_s^0 - \widehat{\eta}_s^1 \|_0 + \varepsilon_{1|0}(s)$ and similarly for $\| A \|_1$, and \cref{lem:risk_contrasts} gives $D_s^{\mathrm{den}} = \| A \|_0^2 + \| A \|_1^2$.
Expanding the squares,
\begin{equation}
\begin{aligned}
    \| A - B_1 \|_0^2 - \| A \|_0^2 &= -2\, \mathbb{E}_{x \sim \zeta_s^0}\big[A(x) \cdot B_1(x)\big] + \varepsilon_{1|0}(s)^2, \\
    \| A + B_0 \|_1^2 - \| A \|_1^2 &= 2\, \mathbb{E}_{x \sim \zeta_s^1}\big[A(x) \cdot B_0(x)\big] + \varepsilon_{0|1}(s)^2 .
\end{aligned}
\end{equation}
Subtracting $D_s^{\mathrm{den}}$ from~\eqref{eq:learned_gap} therefore gives
\begin{equation}
\begin{split}
    \widehat{D}_s^{\mathrm{den}} - D_s^{\mathrm{den}}
    &=
    -2\, \mathbb{E}_{x \sim \zeta_s^0}\big[A(x) \cdot B_1(x)\big]
    + 2\, \mathbb{E}_{x \sim \zeta_s^1}\big[A(x) \cdot B_0(x)\big] \\
    &\quad
    + \varepsilon_{1|0}(s)^2 + \varepsilon_{0|1}(s)^2 - \varepsilon_{0|0}(s)^2 - \varepsilon_{1|1}(s)^2 .
\end{split}
\end{equation}
By the Cauchy--Schwarz inequality, the two cross terms are bounded in absolute value by $2 \| A \|_0\, \varepsilon_{1|0}(s)$ and $2 \| A \|_1\, \varepsilon_{0|1}(s)$.
Since $\| A \|_k \leq \sqrt{D_s^{\mathrm{den}}}$ for both $k$, the triangle inequality yields~\eqref{eq:learned_gap_bound}.

\emph{Velocity contrast:}
The same three steps apply with the target $U_s(\omega, z)$ in place of $\omega$ and the velocities in place of the denoisers.
This uses that $U_s$ is square integrable, that $u_s^\ell(\omega_s) = \mathbb{E}[U_s(\omega, z) \mid \omega_s]$ for $\omega \sim \rho_\ell$, and that $u_s^0 = u_s^1$ under $\mathcal{H}_0$ by \cref{prop:invariance}(iv).
\end{proof}

\begin{proposition}[Consistency]
\label{prop:consistency}
Suppose the sampling conditions of \cref{prop:validity} hold, that \cref{ass:moments} holds, that $\mu$ is strictly positive on $\mathcal{I}_\tau$, and that the integrability condition of \cref{cor:target} holds for the population fields.
Suppose further that $\mathcal{H}_1$ holds with $T_w^{\mathrm{rep}} < \infty$, and that the number of permutations satisfies $R \geq 1/\kappa - 1$.
For $\ell \in \{0, 1\}$, let $M_\ell = \mathbb{E}\big[\widehat{\Phi}(\omega)^2 \,\big|\, \mathcal{G}\big]$ denote the conditional second moment of the score of a held-out observation $\omega \sim \rho_\ell$.
Consider a sequence of experiments along which $n_0^{\mathrm{te}}, n_1^{\mathrm{te}} \to \infty$, and suppose that along this sequence
\begin{enumerate}
    \item[(a)] the errors~\eqref{eq:cross_errors} of the learned fields of representation $\mathrm{rep}$ satisfy $\int_{\mathcal{I}_\tau} w(s)\, \varepsilon_{\ell|k}(s)^2\, \mathrm{d}s \to 0$ in probability for all $\ell, k \in \{0, 1\}$;
    \item[(b)] for every $s \in \mathcal{I}_\tau$, the learned fields are square integrable under $\zeta_s^0$ and $\zeta_s^1$, and $M_0$ and $M_1$ are bounded in probability.
\end{enumerate}
Then $\mathbb{P}\big[\Delta(\mathcal{D}_0, \mathcal{D}_1) = 1\big] \to 1$.
\end{proposition}

Condition~(b) holds, for instance, if $\rho_0$ and $\rho_1$ have finite fourth moments, $\mu$ is bounded away from zero on $\mathcal{I}_\tau$, $w$ is bounded on $\mathcal{I}_\tau$ as are the weights~\eqref{eq:jeff_weights}, and the learned fields grow at most linearly, $|\widehat{\eta}_s^\ell(x)| \leq C(1 + |x|)$, with $C$ bounded in probability.
The same holds for the velocity representation, with $\widehat{u}_s^\ell$ in place of $\widehat{\eta}_s^\ell$.
Condition~(a) makes precise the role of the cross errors: near the pivot, the two channels overlap heavily, and extrapolation is mild, whereas near $s = \tfrac{1}{2}$ the channels separate and the cross errors are hardest to control.
This motivates both the truncation $\tau$ and the uniform weighting, which, unlike the Jeffreys weights of the denoiser representation, does not up-weight the low-noise region.

\begin{proof}[Proof of \cref{sec:consistency}]
We show that the observed statistic converges in probability to the positive population target, while every permuted statistic converges to zero.
Within this proof, we drop the superscript on the held-out sizes, write $n = n_0 + n_1$, $T = T_w^{\mathrm{rep}}$ and $M = \max(M_0, M_1)$, and let $\widetilde{T} = \mathbb{E}[\widehat{T} \mid \mathcal{G}]$ denote the conditional mean of the statistic.

Condition on $\mathcal{G}$, so that the learned fields are fixed.
For an observation $\omega \sim \rho_\ell$ with Monte Carlo draws $(s_k, z_k)_{k=1}^K$, the score is $\widehat{\Phi}(\omega) = K^{-1} \sum_{k=1}^K X_k$ with $X_k = w(s_k)\, \psi_{s_k}^{\mathrm{rep}}(\omega, z_k) / \mu(s_k)$.
Given $\omega$ and $\mathcal{G}$, the $X_k$ are i.i.d., so
\begin{equation}
    \mathbb{E}\big[\widehat{\Phi}(\omega)^2 \,\big|\, \omega, \mathcal{G}\big]
    =
    \frac{\mathrm{Var}(X_1 \mid \omega, \mathcal{G})}{K} + \mathbb{E}[X_1 \mid \omega, \mathcal{G}]^2
    \geq
    \frac{\mathbb{E}[X_1^2 \mid \omega, \mathcal{G}]}{K},
\end{equation}
and averaging over $\omega$ gives $\mathbb{E}[X_1^2 \mid \mathcal{G}] \leq K M_\ell < \infty$.
In particular, $X_1$ is conditionally integrable.
Since the learned fields are $\mathcal{G}$-measurable and $(\omega, s_1, z_1)$ is independent of $\mathcal{G}$, and since $\mu > 0$ on $\mathcal{I}_\tau$, Tonelli's theorem gives
\begin{equation}
    \mathbb{E}\big[|X_1| \,\big|\, \mathcal{G}\big]
    =
    \int_{\mathcal{I}_\tau} w(s)\, \mathbb{E}_{\omega \sim \rho_\ell,\, z}\big[|\psi_s^{\mathrm{rep}}(\omega, z)|\big]\, \mathrm{d}s ,
\end{equation}
with the learned fields inside $\psi_s^{\mathrm{rep}}$ held fixed.
Fubini's theorem then gives $\mathbb{E}[\widehat{\Phi}(\omega) \mid \mathcal{G}] = \int_{\mathcal{I}_\tau} w(s)\, \mathbb{E}_{\omega \sim \rho_\ell,\, z}[\psi_s^{\mathrm{rep}}(\omega, z)]\, \mathrm{d}s$, and by linearity,
\begin{equation}
    \widetilde{T} = \int_{\mathcal{I}_\tau} w(s)\, \widehat{D}_s^{\mathrm{rep}}\, \mathrm{d}s .
\end{equation}

Since $T < \infty$ and $w > 0$, we have $D_s^{\mathrm{rep}} < \infty$ for almost every $s \in \mathcal{I}_\tau$.
By condition~(b), \cref{lem:learned_gap} applies at every $s \in \mathcal{I}_\tau$.
Integrating~\eqref{eq:learned_gap_bound} against $w$ and applying the Cauchy--Schwarz inequality to the first term gives
\begin{equation}
    \big| \widetilde{T} - T \big|
    \leq
    2\, T^{1/2} \left( 2 \int_{\mathcal{I}_\tau} w(s) \big( \varepsilon_{1|0}(s)^2 + \varepsilon_{0|1}(s)^2 \big)\, \mathrm{d}s \right)^{1/2}
    +
    \sum_{\ell, k \in \{0, 1\}} \int_{\mathcal{I}_\tau} w(s)\, \varepsilon_{\ell|k}(s)^2\, \mathrm{d}s ,
\end{equation}
where we used $(\varepsilon_{1|0} + \varepsilon_{0|1})^2 \leq 2(\varepsilon_{1|0}^2 + \varepsilon_{0|1}^2)$.
By condition~(a), the right-hand side converges to zero in probability, so $\widetilde{T} \to T$ in probability.
Under $\mathcal{H}_1$, \cref{cor:target} gives $T > 0$.

Conditionally on $\mathcal{G}$, the held-out scores are independent, with the $n_0$ scores from $\mathcal{D}_0^{\mathrm{te}}$ identically distributed and likewise the $n_1$ scores from $\mathcal{D}_1^{\mathrm{te}}$.
Hence $\mathrm{Var}(\widehat{T} \mid \mathcal{G}) \leq M_0 / n_0 + M_1 / n_1 \leq M (1/n_0 + 1/n_1)$.
For any $\delta > 0$ and $C > 0$, Chebyshev's inequality gives
\begin{equation}
    \mathbb{P}\big[|\widehat{T} - \widetilde{T}| > \delta\big]
    \leq
    \mathbb{P}[M > C] + \frac{C}{\delta^2} \left( \frac{1}{n_0} + \frac{1}{n_1} \right).
\end{equation}
The second term vanishes as $n_0, n_1 \to \infty$, and the first can be made arbitrarily small by choosing $C$ large, since $M$ is bounded in probability.
Hence $\widehat{T} - \widetilde{T} \to 0$ in probability, and combined with Step 2, $\widehat{T} \to T$ in probability.
In particular, $\mathbb{P}[|\widehat{T}| \geq T/2] \to 1$.

Let $W$ be the pooled vector of held-out scores, with mean $\bar{W}$ and sample variance $S^2 = (n-1)^{-1} \sum_{i=1}^n (W_i - \bar{W})^2$.
For a uniform permutation $g$, independent of $W$, the permuted statistic is $m_0 - m_1$, where $m_0$ is the mean of a uniformly random subset of $n_0$ pooled scores and $m_1$ is the mean of its complement.
Since $n_0 m_0 + n_1 m_1 = n \bar{W}$, we have $m_0 - m_1 = (n / n_1)(m_0 - \bar{W})$.
By the standard moments of sampling without replacement, $\mathbb{E}[m_0 \mid W] = \bar{W}$ and $\mathrm{Var}(m_0 \mid W) = n_1 S^2 / (n\, n_0)$.
Hence, the permuted statistic has a conditional mean of zero and a conditional variance
\begin{equation}
    \frac{n^2}{n_1^2} \cdot \frac{n_1 S^2}{n\, n_0}
    =
    \left( \frac{1}{n_0} + \frac{1}{n_1} \right) S^2 .
\end{equation}
Moreover, $S^2 \leq \frac{n}{n-1} \cdot \frac{1}{n} \sum_{i=1}^n W_i^2$, and the conditional mean of $n^{-1} \sum_{i} W_i^2$ given $\mathcal{G}$ is at most $M$.
By the conditional Markov inequality, $S^2$ is therefore bounded in probability.
Chebyshev's inequality, applied conditionally on $W$, gives for each $r \in \{1, \dots, R\}$
\begin{equation}
    \mathbb{P}\big[|\widehat{T}^{(r)}| > T/4 \,\big|\, W\big]
    \leq
    \frac{16}{T^2} \left( \frac{1}{n_0} + \frac{1}{n_1} \right) S^2 ,
\end{equation}
which converges to zero in probability.
By bounded convergence, $\mathbb{P}[|\widehat{T}^{(r)}| > T/4] \to 0$, and a union bound over the fixed number $R$ of permutations gives $\mathbb{P}[\max_{1 \leq r \leq R} |\widehat{T}^{(r)}| > T/4] \to 0$.

\emph{Step 5: conclusion:}
By Steps 3 and 4, the event $\{|\widehat{T}| \geq T/2\} \cap \{\max_{r} |\widehat{T}^{(r)}| \leq T/4\}$ has probability tending to one.
On this event, every indicator in~\eqref{eq:permutation_pvalue} vanishes, so $\widehat{p} = 1/(R+1) \leq \kappa$ by the choice $R \geq 1/\kappa - 1$, and the test rejects.
\end{proof}

\begin{algorithm}[t]
\caption{Path-based two-sample test (PBI)}
\label{alg:pbi}
\begin{algorithmic}[1]
\Require datasets $\mathcal{D}_0, \mathcal{D}_1$; representation $\mathrm{rep} \in \{\mathrm{den}, \mathrm{vel}\}$; path weight $w$ (\cref{tab:variants}); truncation $\tau \in (0, \tfrac{1}{4})$; sampling density $\mu > 0$ on $\mathcal{I}_\tau$; Monte Carlo draws $K$; permutations $R \geq 1/\kappa - 1$; level $\kappa$
\Ensure decision $\Delta \in \{0, 1\}$ and $p$-value $\widehat{p}$
\State Split $\mathcal{D}_\ell = \mathcal{D}_\ell^{\mathrm{tr}} \cup \mathcal{D}_\ell^{\mathrm{te}}$ at random for $\ell \in \{0, 1\}$ \Comment{independent of the data values}
\For{$\ell \in \{0, 1\}$}
    \State $\widehat{f}^{\ell} \gets \Call{FitField}{\mathcal{D}_\ell^{\mathrm{tr}}, \mathrm{rep}}$ \Comment{\cref{alg:fit}; training data only}
\EndFor
\State Pool the held-out data as $(\omega^{(i)}, \ell_i)_{i=1}^{N}$, with $N = n_0^{\mathrm{te}} + n_1^{\mathrm{te}}$ and group labels $\ell_i \in \{0, 1\}$
\For{$i = 1, \dots, N$}
    \State $W_i \gets \Call{Score}{\omega^{(i)}; \widehat{f}^0, \widehat{f}^1}$ \Comment{\cref{alg:score}; label not used}
\EndFor
\State $\widehat{T} \gets \Call{MeanDiff}{W, \ell}$ \Comment{statistic~\eqref{eq:generic_test_stat}}
\For{$r = 1, \dots, R$}
    \State $\ell^{(r)} \gets$ uniformly random permutation of $(\ell_1, \dots, \ell_N)$
    \State $\widehat{T}^{(r)} \gets \Call{MeanDiff}{W, \ell^{(r)}}$ \Comment{reuses scores; no network evaluations}
\EndFor
\State $\widehat{p} \gets \big(1 + \sum_{r=1}^{R} \mathbf{1}\{|\widehat{T}^{(r)}| \geq |\widehat{T}|\}\big) / (R + 1)$
\State \Return $\Delta = \mathbf{1}\{\widehat{p} \leq \kappa\}$ and $\widehat{p}$
\Statex
\Function{MeanDiff}{$W, \ell$}
    \State \Return $\dfrac{1}{|\{i : \ell_i = 0\}|} \displaystyle\sum_{i : \ell_i = 0} W_i \;-\; \dfrac{1}{|\{i : \ell_i = 1\}|} \displaystyle\sum_{i : \ell_i = 1} W_i$
\EndFunction
\end{algorithmic}
\end{algorithm}

\begin{algorithm}[t]
\caption{\textsc{FitField}: fitting the field of one channel on its training split alone.
No validation data or early stopping is used; the network after $n_{\mathrm{iter}}$
iterations is returned (\cref{app:implementation-details}).}
\label{alg:fit}
\begin{algorithmic}[1]
\Require training data $\mathcal{D}_\ell^{\mathrm{tr}}$; representation $\mathrm{rep}$; network $f_\theta(s, x)$; training density $\mu_{\mathrm{tr}}$ on $\mathcal{I}_\tau$; iterations $n_{\mathrm{iter}}$; batch size $B$
\Ensure fitted field $\widehat{f}^{\ell}$, estimating $\eta_s^\ell$ if $\mathrm{rep} = \mathrm{den}$ and $u_s^\ell$ if $\mathrm{rep} = \mathrm{vel}$
\For{$\mathrm{iter} = 1, \dots, n_{\mathrm{iter}}$}
    \State Draw $x^{(1)}, \dots, x^{(B)}$ from $\mathcal{D}_\ell^{\mathrm{tr}}$, and independent $s_j \sim \mu_{\mathrm{tr}}$ and $z_j \sim \mathcal{N}(0, \mathbb{I}_d)$
    \For{$j = 1, \dots, B$}
        \State $x_{s_j} \gets a_{s_j} x^{(j)} + \sigma_{s_j} z_j$ \Comment{Gaussian channel~\eqref{eq:channels}}
        \State $y_j \gets \Call{Target}{x^{(j)}, z_j, s_j}$
    \EndFor
    \State Update $\theta$ by a gradient step on $\frac{1}{B} \sum_{j=1}^{B} \big|f_\theta(s_j, x_{s_j}) - y_j\big|^2$
\EndFor
\State \Return $\widehat{f}^{\ell} = f_\theta$
\Statex
\Function{Target}{$x, z, s$}
    \If{$\mathrm{rep} = \mathrm{den}$}
        \State \Return $x$ \Comment{denoiser target}
    \Else
        \State \Return $a_s' x + \sigma_s' z$ \Comment{outward velocity target}
    \EndIf
\EndFunction
\end{algorithmic}
\end{algorithm}

\begin{algorithm}[t]
\caption{\textsc{Score}: Monte Carlo estimate of the observation-level score.}
\label{alg:score}
\begin{algorithmic}[1]
\Require held-out observation $\omega$; fitted fields $\widehat{f}^0, \widehat{f}^1$; representation $\mathrm{rep}$; path weight $w$; sampling density $\mu$ on $\mathcal{I}_\tau$; Monte Carlo draws $K$
\Ensure score $\widehat{\Phi}_w^{\mathrm{rep}}(\omega)$
\For{$k = 1, \dots, K$}
    \State Draw $s_k \sim \mu$ and $z_k \sim \mathcal{N}(0, \mathbb{I}_d)$ \Comment{fresh draws for each observation}
    \State $\omega_{s_k} \gets a_{s_k} \omega + \sigma_{s_k} z_k$
    \State $y_k \gets \Call{Target}{\omega, z_k, s_k}$ \Comment{\cref{alg:fit}}
    \State $\psi_k \gets \big|y_k - \widehat{f}^1(s_k, \omega_{s_k})\big|^2 - \big|y_k - \widehat{f}^0(s_k, \omega_{s_k})\big|^2$ \Comment{local contrast}
\EndFor
\State \Return $\dfrac{1}{K} \displaystyle\sum_{k=1}^{K} \frac{w(s_k)}{\mu(s_k)}\, \psi_k$ \Comment{importance weighting~\eqref{eq:mc_witness}}
\end{algorithmic}
\end{algorithm}

\begin{table}[t]
\centering
\caption{The four PBI variants, each combining a regression representation with a path
weight and run with \cref{alg:pbi,alg:fit,alg:score}. The two Jeffreys weights give
the same population target, the Jeffreys divergence of \cref{prop:jeffreys} up to path
truncation. All weights are positive and bounded on $I_\tau$.}
\label{tab:variants}
\small
\begin{tabular}{llll}
\toprule
Variant & Fitted fields & Regression target $y$ & Path weight $w(s)$ \\
\midrule
PBI-Denoiser\_U & $\widehat{\eta}_s^0, \widehat{\eta}_s^1$ & $\omega$ & $1$ \\
PBI-Denoiser\_J & $\widehat{\eta}_s^0, \widehat{\eta}_s^1$ & $\omega$ & $\lambda_s a_s / \sigma_s^2$ \\
PBI-Velocity\_U & $\widehat{u}_s^0, \widehat{u}_s^1$ & $a_s' \omega + \sigma_s' z$ & $1$ \\
PBI-Velocity\_J & $\widehat{u}_s^0, \widehat{u}_s^1$ & $a_s' \omega + \sigma_s' z$ & $a_s / (\lambda_s \sigma_s^2)$ \\
\bottomrule
\end{tabular}
\end{table}

\subsection{Relationship between the Four Variants}
\label{sec:four_variants}

The method yields four closely related test statistics,
\begin{equation}
\label{eq:four_tests}
\begin{array}{c|cc}
    & \text{Uniform} & \text{Jeffreys weighted} \\
    \hline
    \text{Denoising risk} & \widehat{T}_{\mathrm{U}}^{\mathrm{den}} & \widehat{T}_{\mathrm{J}}^{\mathrm{den}} \\[1mm]
    \text{Velocity risk} & \widehat{T}_{\mathrm{U}}^{\mathrm{vel}} & \widehat{T}_{\mathrm{J}}^{\mathrm{vel}} .
\end{array}
\end{equation}
The rows differ in the regression representation used to probe the path, and the columns differ in how evidence is aggregated across radii.
By \cref{lem:risk_contrasts}, $w_{\mathrm{J}}^{\mathrm{den}} D_s^{\mathrm{den}} = D_s^{\mathrm{vel}} / \chi_s$ at every radius, so the two Jeffreys variants share the population target identified in \cref{prop:jeffreys}.
Their finite-sample behaviour can nevertheless differ because the denoiser and the velocity regression have different noise and approximation properties.
The uniform variants omit the information-theoretic reweighting.
They provide a useful alternative when the Jeffreys weights amplify the difficult low-noise region near $s = \tfrac{1}{2}$, where $\sigma_s \to 0$.

\subsection{Monte Carlo Approximation over the Path}
\label{sec:path_mc}

In practice, neither the latent expectation nor the path integral in~\eqref{eq:aggregated_witness} is evaluated analytically.
For each held-out observation $\omega^{(i)}$, we draw independent
\begin{equation}
    s_{ik} \sim \mu,
    \qquad
    z_{ik} \sim \mathcal{N}(0, \mathbb{I}_d),
    \qquad
    k = 1, \dots, K,
\end{equation}
where $\mu$ is a sampling density that is strictly positive on $\mathcal{I}_\tau$, and use the importance-weighted estimate
\begin{equation}
\label{eq:mc_witness}
    \widehat{\Phi}_w^{\mathrm{rep}}\big(\omega^{(i)}\big)
    =
    \frac{1}{K} \sum_{k=1}^{K} \frac{w(s_{ik})}{\mu(s_{ik})}\, \psi_{s_{ik}}^{\mathrm{rep}}\big(\omega^{(i)}, z_{ik}\big).
\end{equation}
Uniform sampling corresponds to constant $\mu$, and the factor $w/\mu$ permits nonuniform sampling in $s$ without changing the population target~\eqref{eq:population_target}.
Increasing $K$ reduces only the Monte Carlo component of the variance.
The $K$ corrupted copies of a single observation are therefore treated as one statistical unit and averaged before testing.

\section{Experiments}
\label{sec:experiments}

We evaluate the path-based inference (PBI) tests of \cref{sec:method} on one synthetic benchmark and three image benchmarks.
Our study addresses two questions: (i) whether the proposed tests achieve higher power than existing two-sample tests as the sample size increases, and (ii) how the regression representation (denoiser or velocity) and the path weighting (uniform or Jeffreys) affect performance across data modalities.
We consider the four variants \texttt{PBI-Denoiser\_U}, \texttt{PBI-Denoiser\_J}, \texttt{PBI-Velocity\_U} and \texttt{PBI-Velocity\_J} (\cref{tab:variants}); the complete procedure is given in \cref{alg:pbi,alg:fit,alg:score}.
The suffixes \texttt{U} and \texttt{J} denote the uniform and Jeffreys path weights of \cref{sec:aggregation}.
Each variant is a valid test in its own right by \cref{prop:validity}, and we report all four to study the effects of these design choices.
We compare against AutoTST \citep{kubler2022autotst}, C2ST-L \citep{cheng2022c2st}, C2ST-S \citep{lopezpaz2017revisiting}, MMD-D and MMD-O \citep{liu2020learning}, and MMDAgg \citep{jmlr_mmdagg}.

\paragraph{Evaluation protocol:}
All methods are compared under the same total budget of $2n$ observations per population, irrespective of whether they require sample splitting.
Methods that separate learning from testing (PBI, AutoTST, C2ST-L, C2ST-S, MMD-D and MMD-O) split each dataset into $n^{\mathrm{tr}}_\ell = n^{\mathrm{te}}_\ell = n$ observations; their learned components are fitted only on the training split and frozen before the test is evaluated on the independent held-out split.
MMDAgg requires no training split and is instead evaluated on all $2n$ observations per population.
Thus, $n$ in all figures and tables indexes the split-based allocation, and every method receives $4n$ observations across the two populations, so the comparison accounts for both the data used to learn a statistic and the data used to evaluate it.
For each of five seeds, the learned components are fitted once, and the test is then applied to 50 independently drawn held-out test sets of $n$ observations per population.
The power for a seed is the fraction of these 50 tests that reject $\mathcal{H}_0$ at level $\kappa = 0.05$, and we report the mean over the five seeds (250 tests in total), quoting gains in percentage points (pp) over the strongest baseline at each $n$.

\paragraph{Implementation:}
Throughout this section, the two fields are fitted by separate networks, one per population, each trained on its own training split alone (the \emph{Separate} architecture of \cref{sec:ablation}); this applies to both the denoiser and the velocity representations.
No validation set or early stopping is used, and we evaluate the final checkpoint after a fixed training budget.
The training-time variable is restricted to $s \in [0.01, 0.49]$, and the time-embedding dimension is 64.
Image inputs are normalised to $[-1, 1]$, whereas HDGM samples are left unnormalised.
Dataset-specific backbones and optimisation settings, together with the baseline implementations, are given in \cref{app:implementation-details}.

\begin{figure}[!t]
  \centering
  \includegraphics[width=\textwidth]{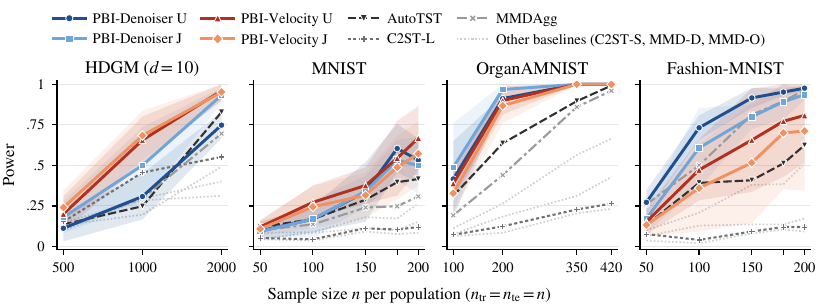}
  \caption{Test power at level $\kappa = 0.05$ as a function of the per-population
  sample size $n$ ($n_{\mathrm{tr}} = n_{\mathrm{te}} = n$; every method receives $4n$
  observations in total). MNIST, OrganAMNIST and Fashion-MNIST use class contamination
  at a fraction $0.1$ (3$\to$8, left$\to$right lung, sneaker$\to$ankle boot). Lines show the mean over five training seeds (50 test sets each); shaded bands show $\pm 1$ standard deviation across seeds for the PBI variants. Highlighted baselines are those that are strongest at some
  sample size; the remaining baselines are shown in light grey. The HDGM axis is
  logarithmic.}
  \label{fig:power-main}
\end{figure}

\paragraph{Reporting:}
\Cref{fig:power-main} summarises power on all four benchmarks.
For each benchmark, a first table reports power at each sample size as the mean over five seeds $\pm$ the unbiased sample variance across seeds; the bands in \cref{fig:power-main} instead show $\pm 1$ standard deviation.
A second table averages power over the sample sizes of the benchmark and compares each PBI variant with every baseline: each cell reports the absolute difference in percentage points and, in parentheses, the difference relative to that baseline's power.
Relative differences become large when a baseline's power is close to the level $\kappa$, so they should be read alongside the absolute differences.

\subsection{High-Dimensional Gaussian Mixtures}
\label{sec:exp-hdgm}

\textbf{Data:}
The high-dimensional Gaussian mixture (HDGM) benchmark of \citet{liu2020learning} consists of equal-weight, two-component Gaussian mixtures in $d = 10$ dimensions.
Let $c \sim \mathrm{Bernoulli}(\tfrac{1}{2})$ denote the component indicator.
Under both $\rho_0$ and $\rho_1$, the components have means $\mathbf{0}$ and $0.5\, \mathbf{1}_d$, for $c = 0$ and $c = 1$ respectively, and unit marginal variances.
Under $\rho_0$, the coordinates are independent within each component.
Under $\rho_1$, the first two coordinates are correlated, with correlation $+0.5$ in the first component and $-0.5$ in the second.
Concretely, we draw $v \sim \mathcal{N}(\mathbf{0}, \mathbb{I}_d)$, replace $v_2$ by $\pm \tfrac{1}{2} v_1 + \tfrac{\sqrt{3}}{2} v_2$ under $\rho_1$, with the sign $+$ for $c = 0$ and $-$ for $c = 1$, and add the component mean $0.5\, c\, \mathbf{1}_d$.
The two densities therefore share their component weights, means and coordinate-wise marginal variances, and differ only in a localised, component-dependent covariance structure.
We evaluate $n \in \{500, 1000, 2000\}$.

\textbf{Results:}
\Cref{fig:power-main,tab:hdgm-power-mean-variance,tab:hdgm-power-comparison} show that the velocity representation is particularly effective on HDGM.
At $n = 500$, \texttt{PBI-Velocity\_J} reaches power $0.240$, exceeding the strongest baseline, C2ST-L, by $8.8$~pp.
The advantage widens at $n = 1000$, where \texttt{PBI-Velocity\_J} attains $0.684$ against $0.456$ for C2ST-L, a gain of $22.8$~pp.
At $n = 2000$, \texttt{PBI-Velocity\_U} performs best, reaching $0.960$ and exceeding the strongest baseline, AutoTST, by $13.2$~pp.
The Jeffreys weighting markedly improves the denoiser, raising its power from $0.112$, $0.308$ and $0.748$ to $0.180$, $0.500$ and $0.932$ across the three sample sizes.
As a result, \texttt{PBI-Denoiser\_J} exceeds the strongest baseline at every sample size, whereas \texttt{PBI-Denoiser\_U} remains below it.
These results indicate that the velocity representation is well suited to the localised covariance shift in this benchmark, and that the Jeffreys weighting is beneficial in this low-dimensional setting.

\begin{table}[t]
\centering
\caption{Test power on HDGM ($d = 10$; localised covariance shift, \cref{app:implementation-details})
at level $\kappa = 0.05$ with $n$ training and $n$ test observations per population.
Entries report the mean over five training seeds (50 test sets each; \cref{app:implementation-details}) $\pm$ the unbiased sample variance across seeds (not the standard deviation). Bold marks the best method and underlining the best baseline in each column.}
\label{tab:hdgm-power-mean-variance}
\small
\begin{adjustbox}{max width=\textwidth}
\begin{tabular}{lccc}
\toprule
& \multicolumn{3}{c}{Sample size $n$} \\
\cmidrule(l){2-4}
Method & $500$ & $1000$ & $2000$ \\
\midrule
PBI-Denoiser\_U & \res{0.112}{0.0065} & \res{0.308}{0.0201} & \res{0.748}{0.0065} \\
PBI-Denoiser\_J & \res{0.180}{0.0098} & \res{0.500}{0.0552} & \res{0.932}{0.0023} \\
PBI-Velocity\_U & \res{0.200}{0.0120} & \res{0.656}{0.0211} & \resb{0.960}{0.0034} \\
PBI-Velocity\_J & \resb{0.240}{0.0110} & \resb{0.684}{0.0231} & \res{0.952}{0.0031} \\
\midrule
AutoTST & \res{0.136}{0.0045} & \res{0.248}{0.0433} & \resu{0.828}{0.0117} \\
C2ST-L & \resu{0.152}{0.0055} & \resu{0.456}{0.0433} & \res{0.552}{0.0129} \\
C2ST-S & \res{0.096}{0.0035} & \res{0.296}{0.0287} & \res{0.400}{0.0158} \\
MMD-D & \res{0.148}{0.0209} & \res{0.284}{0.1175} & \res{0.312}{0.1263} \\
MMD-O & \res{0.132}{0.0045} & \res{0.196}{0.0003} & \res{0.492}{0.0021} \\
MMDAgg & \res{0.120}{0.0018} & \res{0.296}{0.0045} & \res{0.696}{0.0037} \\
\bottomrule
\end{tabular}
\end{adjustbox}
\end{table}
\begin{table}[t]
\centering
\caption{Power differences between each PBI variant (rows) and each baseline (columns) on HDGM, averaged over $n \in \{500, 1000, 2000\}$.
With $A_m$ the mean power of method $m$ averaged over seeds and sample sizes, each cell reports the absolute difference $100(A_{\mathrm{PBI}} - A_b)$ in percentage points, followed in parentheses by the relative difference $100(A_{\mathrm{PBI}} - A_b)/A_b$ in percent.
Positive values favour PBI, and the largest gain in each column is in bold.}
\label{tab:hdgm-power-comparison}
\small
\begin{adjustbox}{max width=\textwidth}
\begin{tabular}{lcccccc}
\toprule
& \multicolumn{6}{c}{Baseline $b$} \\
\cmidrule(l){2-7}
PBI variant & AutoTST & C2ST-L & C2ST-S & MMD-D & MMD-O & MMDAgg \\
\midrule
PBI-Denoiser\_U & \gain{-1.5}{-3.6} & \gain{+0.3}{+0.7} & \gain{+12.5}{+47.5} & \gain{+14.1}{+57.0} & \gain{+11.6}{+42.4} & \gain{+1.9}{+5.0} \\
PBI-Denoiser\_J & \gain{+13.3}{+33.0} & \gain{+15.1}{+39.0} & \gain{+27.3}{+103.5} & \gain{+28.9}{+116.7} & \gain{+26.4}{+96.6} & \gain{+16.7}{+45.0} \\
PBI-Velocity\_U & \gain{+20.1}{+49.8} & \gain{+21.9}{+56.6} & \gain{+34.1}{+129.3} & \gain{+35.7}{+144.1} & \gain{+33.2}{+121.5} & \gain{+23.5}{+63.3} \\
PBI-Velocity\_J & \gainb{+22.1}{+54.8} & \gainb{+23.9}{+61.7} & \gainb{+36.1}{+136.9} & \gainb{+37.7}{+152.2} & \gainb{+35.2}{+128.8} & \gainb{+25.5}{+68.7} \\
\bottomrule
\end{tabular}
\end{adjustbox}
\end{table}

\subsection{MNIST}
\label{sec:exp-mnist}

\textbf{Data:}
MNIST contains greyscale images of handwritten digits \citep{lecun1998mnist}.
We construct a controlled class-contamination alternative from a source digit and a target digit.
Under $\rho_0$, all images are drawn from the source digit $3$.
Under $\rho_1$, a fraction $0.1$ of the images are drawn from the target digit $8$, and the remainder from the source digit.
Each sample of $n$ observations from $\rho_1$ contains exactly $\lfloor 0.1\, n + \tfrac{1}{2} \rfloor$ target-digit images, that is, $0.1\, n$ rounded half up, with the remaining images drawn independently from the source digit.
We vary $n \in \{50, 100, 150, 180, 200\}$.

\textbf{Results:}
\Cref{tab:mnist-power-mean-variance,tab:mnist-power-comparison} show that AutoTST is the strongest baseline at every sample size, and that the PBI variants become substantially more powerful as more data are available.
At $n = 50$, \texttt{PBI-Velocity\_U} ties AutoTST at power $0.120$, while the other variants remain slightly below it.
At $n = 100$, \texttt{PBI-Velocity\_U} reaches $0.272$, improving over AutoTST by $10.4$~pp ($61.9\%$ relative), and both denoiser variants match AutoTST at $0.168$.
From $n = 150$ onwards, all four PBI variants exceed the strongest baseline.
The largest gain at $n = 180$ is obtained by \texttt{PBI-Denoiser\_U}, which reaches $0.604$ and improves over AutoTST by $20.8$~pp.
At $n = 200$, \texttt{PBI-Velocity\_U} is strongest, with power $0.668$ and a gain of $25.2$~pp.
The denoiser variants dip slightly between $n = 180$ and $n = 200$, a change within the seed-to-seed variability reported in \cref{tab:mnist-power-mean-variance}.
Overall, uniform weighting is more effective than Jeffreys weighting on this task, particularly for the velocity representation, which performs better under uniform weighting across all sample sizes.

\begin{table}[t]
\centering
\caption{Test power on MNIST with source digit $3$ and target digit $8$ at a contamination fraction of $0.1$, at level $\kappa = 0.05$ and with $n$ training and $n$ test observations per population.
Entries report the mean over five training seeds (50 test sets each; \cref{app:implementation-details}) $\pm$ the unbiased sample variance across seeds (not the standard deviation).
The best result in each column is in bold, and the best baseline is underlined; tied entries share the formatting.}
\label{tab:mnist-power-mean-variance}
\small
\begin{adjustbox}{max width=\textwidth}
\begin{tabular}{lccccc}
\toprule
& \multicolumn{5}{c}{Sample size $n$} \\
\cmidrule(l){2-6}
Method & $50$ & $100$ & $150$ & $180$ & $200$ \\
\midrule
PBI-Denoiser\_U & \res{0.096}{0.0033} & \res{0.168}{0.0065} & \res{0.336}{0.0339} & \resb{0.604}{0.0227} & \res{0.532}{0.0181} \\
PBI-Denoiser\_J & \res{0.092}{0.0009} & \res{0.168}{0.0129} & \res{0.340}{0.0092} & \res{0.532}{0.0157} & \res{0.500}{0.0306} \\
PBI-Velocity\_U & \resb{0.120}{0.0018} & \resb{0.272}{0.0105} & \resb{0.376}{0.0125} & \res{0.544}{0.0529} & \resb{0.668}{0.0387} \\
PBI-Velocity\_J & \res{0.108}{0.0023} & \res{0.244}{0.0169} & \res{0.316}{0.0195} & \res{0.488}{0.0385} & \res{0.572}{0.0393} \\
\midrule
AutoTST & \resbu{0.120}{0.0048} & \resu{0.168}{0.0117} & \resu{0.288}{0.0071} & \resu{0.396}{0.0245} & \resu{0.416}{0.0361} \\
C2ST-L & \res{0.052}{0.0001} & \res{0.044}{0.0019} & \res{0.112}{0.0023} & \res{0.104}{0.0055} & \res{0.120}{0.0074} \\
C2ST-S & \res{0.044}{0.0003} & \res{0.036}{0.0011} & \res{0.092}{0.0041} & \res{0.076}{0.0025} & \res{0.084}{0.0043} \\
MMD-D & \res{0.048}{0.0009} & \res{0.084}{0.0021} & \res{0.112}{0.0035} & \res{0.100}{0.0002} & \res{0.152}{0.0017} \\
MMD-O & \res{0.084}{0.0011} & \res{0.088}{0.0009} & \res{0.180}{0.0016} & \res{0.172}{0.0001} & \res{0.264}{0.0061} \\
MMDAgg & \res{0.100}{0.0014} & \res{0.136}{0.0007} & \res{0.240}{0.0108} & \res{0.248}{0.0011} & \res{0.308}{0.0041} \\
\bottomrule
\end{tabular}
\end{adjustbox}
\end{table}
\begin{table}[t]
\centering
\caption{Power differences between each PBI variant (rows) and each baseline (columns) on MNIST, averaged over $n \in \{50, 100, 150, 180, 200\}$.
Cells are computed as in \cref{tab:hdgm-power-comparison}.
Positive values favour PBI, and the largest gain in each column is in bold.}
\label{tab:mnist-power-comparison}
\small
\begin{adjustbox}{max width=\textwidth}
\begin{tabular}{lcccccc}
\toprule
& \multicolumn{6}{c}{Baseline $b$} \\
\cmidrule(l){2-7}
PBI variant & AutoTST & C2ST-L & C2ST-S & MMD-D & MMD-O & MMDAgg \\
\midrule
PBI-Denoiser\_U & \gain{+7.0}{+25.1} & \gain{+26.1}{+301.9} & \gain{+28.1}{+422.9} & \gain{+24.8}{+250.0} & \gain{+19.0}{+120.3} & \gain{+14.1}{+68.2} \\
PBI-Denoiser\_J & \gain{+4.9}{+17.6} & \gain{+24.0}{+277.8} & \gain{+26.0}{+391.6} & \gain{+22.7}{+229.0} & \gain{+16.9}{+107.1} & \gain{+12.0}{+58.1} \\
PBI-Velocity\_U & \gainb{+11.8}{+42.7} & \gainb{+31.0}{+358.3} & \gainb{+33.0}{+496.4} & \gainb{+29.7}{+299.2} & \gainb{+23.8}{+151.3} & \gainb{+19.0}{+91.9} \\
PBI-Velocity\_J & \gain{+6.8}{+24.5} & \gain{+25.9}{+300.0} & \gain{+27.9}{+420.5} & \gain{+24.6}{+248.4} & \gain{+18.8}{+119.3} & \gain{+13.9}{+67.4} \\
\bottomrule
\end{tabular}
\end{adjustbox}
\end{table}

\subsection{OrganAMNIST}
\label{sec:exp-organamnist}

\textbf{Data:}
OrganAMNIST is a MedMNIST subset consisting of grayscale axial CT slices derived from three-dimensional abdominal CT volumes and centred on annotated organ bounding boxes \citep{medmnistv2}.
We use the $64\times64$ dimensionality and apply the class-contamination construction with Left Lung as the source class and Right Lung as the target class.
Under $\rho_0$, all images are drawn from the source class, whereas each sample of $n$ observations from $\rho_1$ contains $\lfloor 0.1\,n + \tfrac{1}{2} \rfloor$ target-class images, with the remainder drawn from the source class.
We vary $n\in\{100,200,350,420\}$.

\textbf{Results:}
\Cref{tab:organamnist-power-mean-variance,tab:organamnist-power-comparison} show that AutoTST is the strongest baseline at every sample size, yet every PBI variant exceeds it throughout the experiment.
In the lower-sample regime, \texttt{PBI-Denoiser\_J} is the strongest variant.
At $n=100$, it reaches $0.488$ power compared with $0.312$ for AutoTST, a gain of $17.6$~pp ($56.4\%$ relative).
At $n=200$, the corresponding powers are $0.968$ and $0.636$, yielding a gain of $33.2$~pp ($52.2\%$ relative).
All four PBI variants attain power $1.000$ at $n=350$ and $n=420$, with zero variance across five training seeds, whereas AutoTST reaches $0.896$ and $0.992$, respectively.
Averaged over all sample sizes, \texttt{PBI-Denoiser\_J} has the largest gain against every baseline, including $15.5$~pp ($21.9\%$ relative) over AutoTST and $25.1$~pp ($40.9\%$ relative) over MMDAgg.
Thus, OrganAMNIST favours the denoiser parameterisation in the lower-sample regime, while both parameterisations saturate once $n\geq350$.

\begin{table}[t]
\centering
\caption{Test power on OrganAMNIST with source class Left Lung and target class Right Lung at a contamination fraction of $0.1$, at level $\kappa = 0.05$ and with $n$ training and $n$ test observations per population.
Entries report the mean over five training seeds (50 test sets each; \cref{app:implementation-details}) $\pm$ the unbiased sample variance across seeds (not the standard deviation).
The best result in each column is in bold, and the best baseline is underlined.}
\label{tab:organamnist-power-mean-variance}
\small
\begin{adjustbox}{max width=\textwidth}
\begin{tabular}{lcccc}
\toprule
& \multicolumn{4}{c}{Sample size $n$} \\
\cmidrule(l){2-5}
Method & $100$ & $200$ & $350$ & $420$ \\
\midrule
PBI-Denoiser\_U & \res{0.416}{0.0543} & \res{0.912}{0.0003} & \resb{1.000}{0.0000} & \resb{1.000}{0.0000} \\
PBI-Denoiser\_J & \resb{0.488}{0.0737} & \resb{0.968}{0.0013} & \resb{1.000}{0.0000} & \resb{1.000}{0.0000} \\
PBI-Velocity\_U & \res{0.388}{0.0097} & \res{0.900}{0.0076} & \resb{1.000}{0.0000} & \resb{1.000}{0.0000} \\
PBI-Velocity\_J & \res{0.328}{0.0141} & \res{0.868}{0.0117} & \resb{1.000}{0.0000} & \resb{1.000}{0.0000} \\
\midrule
AutoTST & \resu{0.312}{0.0291} & \resu{0.636}{0.0199} & \resu{0.896}{0.0123} & \resu{0.992}{0.0003} \\
C2ST-L & \res{0.072}{0.0019} & \res{0.124}{0.0047} & \res{0.228}{0.0027} & \res{0.264}{0.0037} \\
C2ST-S & \res{0.064}{0.0007} & \res{0.084}{0.0011} & \res{0.204}{0.0043} & \res{0.232}{0.0045} \\
MMD-D & \res{0.076}{0.0011} & \res{0.184}{0.0083} & \res{0.312}{0.0069} & \res{0.428}{0.0053} \\
MMD-O & \res{0.112}{0.0029} & \res{0.264}{0.0011} & \res{0.564}{0.0005} & \res{0.664}{0.0015} \\
MMDAgg & \res{0.192}{0.0023} & \res{0.440}{0.0138} & \res{0.860}{0.0030} & \res{0.960}{0.0008} \\
\bottomrule
\end{tabular}
\end{adjustbox}
\end{table}

\begin{table}[t]
\centering
\caption{Power differences between each PBI variant (rows) and each baseline (columns) on OrganAMNIST, averaged over $n \in \{100, 200, 350, 420\}$.
Cells are computed as in \cref{tab:hdgm-power-comparison}.
Positive values favour PBI, and the largest gain in each column is in bold.}
\label{tab:organamnist-power-comparison}
\small
\begin{adjustbox}{max width=\textwidth}
\begin{tabular}{lcccccc}
\toprule
& \multicolumn{6}{c}{Baseline $b$} \\
\cmidrule(l){2-7}
PBI variant & AutoTST & C2ST-L & C2ST-S & MMD-D & MMD-O & MMDAgg \\
\midrule
PBI-Denoiser\_U & \gain{+12.3}{+17.3} & \gain{+66.0}{+383.7} & \gain{+68.6}{+469.9} & \gain{+58.2}{+232.8} & \gain{+43.1}{+107.5} & \gain{+21.9}{+35.7} \\
PBI-Denoiser\_J & \gainb{+15.5}{+21.9} & \gainb{+69.2}{+402.3} & \gainb{+71.8}{+491.8} & \gainb{+61.4}{+245.6} & \gainb{+46.3}{+115.5} & \gainb{+25.1}{+40.9} \\
PBI-Velocity\_U & \gain{+11.3}{+15.9} & \gain{+65.0}{+377.9} & \gain{+67.6}{+463.0} & \gain{+57.2}{+228.8} & \gain{+42.1}{+105.0} & \gain{+20.9}{+34.1} \\
PBI-Velocity\_J & \gain{+9.0}{+12.7} & \gain{+62.7}{+364.5} & \gain{+65.3}{+447.3} & \gain{+54.9}{+219.6} & \gain{+39.8}{+99.3} & \gain{+18.6}{+30.3} \\
\bottomrule
\end{tabular}
\end{adjustbox}
\end{table}

\subsection{Fashion-MNIST}
\label{sec:exp-fashion-mnist}

\textbf{Data:}
Fashion-MNIST contains greyscale images of apparel with the same dimensions and the same training and test splits as MNIST \citep{xiao2017_fashionmnist}.
We apply the class-contamination construction, with Sneaker as the source class and Ankle Boot as the target class.
Under $\rho_0$, all images are drawn from the source class, whereas each sample of $n$ observations from $\rho_1$ contains $\lfloor 0.1\, n + \tfrac{1}{2} \rfloor$ target-class images, with the remainder drawn from the source class.
We again vary $n \in \{50, 100, 150, 180, 200\}$.

\textbf{Results:}
\Cref{tab:fashion-mnist-power-mean-variance,tab:fashion-mnist-power-comparison} reveal a different ordering of the PBI variants.
MMDAgg is the strongest baseline at every sample size, and \texttt{PBI-Denoiser\_U} is the only variant that exceeds it at every sample size.
Its largest advantage occurs at $n = 100$, where it reaches $0.732$ against $0.500$ for MMDAgg, a gain of $23.2$~pp ($46.4\%$ relative).
It remains stronger at $n = 150$ and $n = 180$, with gains of $10.4$ and $5.6$~pp respectively, and approaches ceiling power at $n = 200$ ($0.976$ against $0.964$).
\texttt{PBI-Denoiser\_J} exceeds MMDAgg at $n = 100$ and is within $1.2$~pp of it at $n = 150$ and $n = 180$.
Both velocity variants trail MMDAgg across all sample sizes and exhibit markedly greater variability across seeds, suggesting that velocity regression is less stable on this dataset.
Thus, unlike HDGM and MNIST, Fashion-MNIST clearly favours the denoiser representation with uniform weighting.

\begin{table}[t]
\centering
\caption{Test power on Fashion-MNIST with source class Sneaker and target class Ankle Boot at a contamination fraction of $0.1$, at level $\kappa = 0.05$ and with $n$ training and $n$ test observations per population.
Entries report the mean over five training seeds (50 test sets each; \cref{app:implementation-details}) $\pm$ the unbiased sample variance across seeds (not the standard deviation).
The best result in each column is in bold, and the best baseline is underlined.}
\label{tab:fashion-mnist-power-mean-variance}
\small
\begin{adjustbox}{max width=\textwidth}
\begin{tabular}{lccccc}
\toprule
& \multicolumn{5}{c}{Sample size $n$} \\
\cmidrule(l){2-6}
Method & $50$ & $100$ & $150$ & $180$ & $200$ \\
\midrule
PBI-Denoiser\_U & \resb{0.272}{0.0057} & \resb{0.732}{0.0057} & \resb{0.916}{0.0037} & \resb{0.952}{0.0005} & \resb{0.976}{0.0003} \\
PBI-Denoiser\_J & \res{0.172}{0.0073} & \res{0.608}{0.0605} & \res{0.800}{0.0056} & \res{0.892}{0.0117} & \res{0.936}{0.0011} \\
PBI-Velocity\_U & \res{0.152}{0.0061} & \res{0.472}{0.0355} & \res{0.656}{0.1015} & \res{0.772}{0.1177} & \res{0.808}{0.0841} \\
PBI-Velocity\_J & \res{0.132}{0.0047} & \res{0.360}{0.0724} & \res{0.516}{0.1379} & \res{0.700}{0.1208} & \res{0.712}{0.1341} \\
\midrule
AutoTST & \res{0.132}{0.0011} & \res{0.392}{0.0211} & \res{0.408}{0.0157} & \res{0.512}{0.0139} & \res{0.624}{0.0117} \\
C2ST-L & \res{0.076}{0.0015} & \res{0.040}{0.0014} & \res{0.092}{0.0023} & \res{0.120}{0.0038} & \res{0.120}{0.0008} \\
C2ST-S & \res{0.036}{0.0003} & \res{0.020}{0.0002} & \res{0.080}{0.0010} & \res{0.104}{0.0029} & \res{0.092}{0.0007} \\
MMD-D & \res{0.064}{0.0015} & \res{0.128}{0.0031} & \res{0.136}{0.0059} & \res{0.136}{0.0039} & \res{0.172}{0.0041} \\
MMD-O & \res{0.100}{0.0006} & \res{0.208}{0.0061} & \res{0.380}{0.0040} & \res{0.384}{0.0021} & \res{0.512}{0.0063} \\
MMDAgg & \resu{0.260}{0.0030} & \resu{0.500}{0.0004} & \resu{0.812}{0.0041} & \resu{0.896}{0.0049} & \resu{0.964}{0.0015} \\
\bottomrule
\end{tabular}
\end{adjustbox}
\end{table}
\begin{table}[ht]
\centering
\caption{Power differences between each PBI variant (rows) and each baseline (columns) on Fashion-MNIST, averaged over $n \in \{50, 100, 150, 180, 200\}$.
Cells are computed as in \cref{tab:hdgm-power-comparison}.
Positive values favour PBI, and the largest gain in each column is in bold.}
\label{tab:fashion-mnist-power-comparison}
\small
\begin{adjustbox}{max width=\textwidth}
\begin{tabular}{lcccccc}
\toprule
& \multicolumn{6}{c}{Baseline $b$} \\
\cmidrule(l){2-7}
PBI variant & AutoTST & C2ST-L & C2ST-S & MMD-D & MMD-O & MMDAgg \\
\midrule
PBI-Denoiser\_U & \gainb{+35.6}{+86.1} & \gainb{+68.0}{+758.9} & \gainb{+70.3}{+1059.0} & \gainb{+64.2}{+505.0} & \gainb{+45.3}{+142.9} & \gainb{+8.3}{+12.1} \\
PBI-Denoiser\_J & \gain{+26.8}{+64.8} & \gain{+59.2}{+660.7} & \gain{+61.5}{+926.5} & \gain{+55.4}{+435.8} & \gain{+36.5}{+115.2} & \gain{-0.5}{-0.7} \\
PBI-Velocity\_U & \gain{+15.8}{+38.3} & \gain{+48.2}{+538.4} & \gain{+50.6}{+761.4} & \gain{+44.5}{+349.7} & \gain{+25.5}{+80.6} & \gain{-11.4}{-16.7} \\
PBI-Velocity\_J & \gain{+7.0}{+17.0} & \gain{+39.4}{+440.2} & \gain{+41.8}{+628.9} & \gain{+35.7}{+280.5} & \gain{+16.7}{+52.8} & \gain{-20.2}{-29.5} \\
\bottomrule
\end{tabular}
\end{adjustbox}
\end{table}

\subsection{Type-I Error}
\label{sec:exp-type1}

\Cref{fig:type1} reports the empirical type-I error on all four benchmarks, with both samples drawn from $\rho_0$.
The rejection rates of all four PBI variants remain close to the nominal level $\kappa = 0.05$ at every sample size, with deviations comparable to the seed-to-seed variation shown by the shaded bands, as guaranteed by \cref{prop:validity}.

\begin{figure}[!t]
  \centering
  \includegraphics[width=\textwidth]{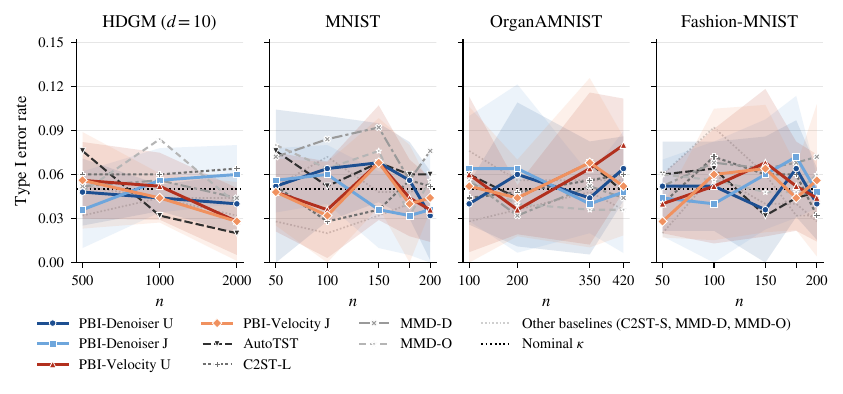}
  \caption{Empirical Type-I error at level $\kappa = 0.05$, with both samples drawn from
$\rho_0$, as a function of the per-population sample size $n$
($n_{\mathrm{tr}} = n_{\mathrm{te}} = n$). Lines show the mean rejection rate over five
training seeds (50 test sets each); shaded bands show $\pm 1$ standard deviation across
seeds for the PBI variants, and the dotted line marks $\kappa$. Baselines are styled as in
\cref{fig:power-main}, and the HDGM axis is logarithmic. The PBI variants remain near the
nominal level on all four benchmarks, as guaranteed by \cref{prop:validity}.}
  \label{fig:type1}
\end{figure}

\subsection{Samples from the Fitted Models}
\label{sec:exp-samples}

\Cref{fig:mnist-generation-n50,fig:organ-amnist-generation-n100,fig:fashion-mnist-generation-n50} show samples generated by the separately trained denoiser and velocity models at the smallest sample size of each image benchmark, alongside held-out observations from both populations.
The $\rho_0$ models are trained on the source class alone (digit 3, Left Lung and Sneaker), and the $\rho_1$ models on samples containing $10\%$ of the target class (digit 8, Right Lung and Ankle Boot).
Although the samples are visibly noisy, the tests built from these fitted fields attain the powers reported in \cref{sec:exp-mnist,sec:exp-organamnist,sec:exp-fashion-mnist} at the same sample sizes.

\providecommand{\mnistfigdir}{experiments/mnist}
\providecommand{\organfigdir}{experiments/organamnist}
\providecommand{\fashionfigdir}{experiments/fashionmnist}

\begin{figure}[!t]
  \centering
  \begin{minipage}[t]{0.48\textwidth}
    \centering
    \includegraphics[width=0.88\linewidth]{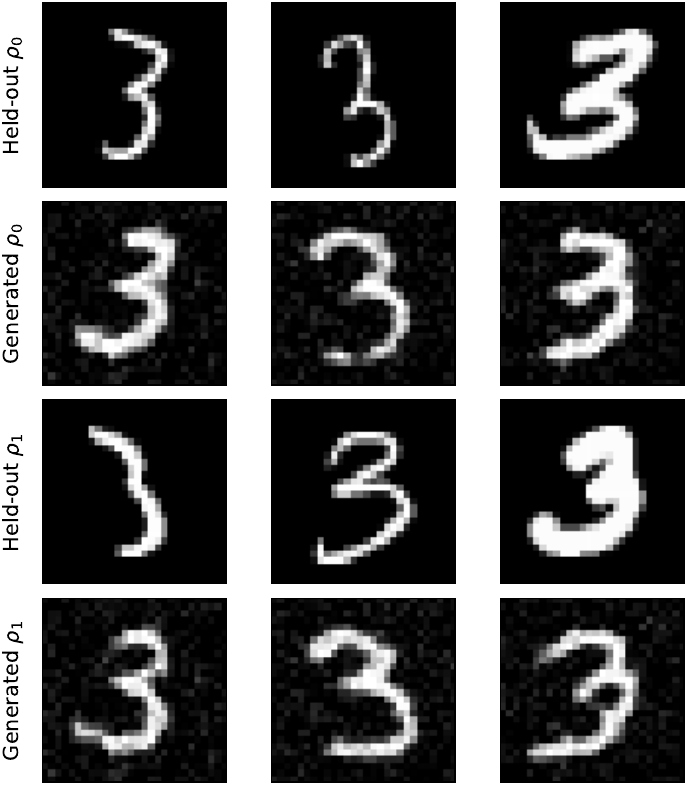}\\[-2pt]
    \textbf{(a) Denoiser}
  \end{minipage}\hfill
  \begin{minipage}[t]{0.48\textwidth}
    \centering
    \includegraphics[width=0.88\linewidth]{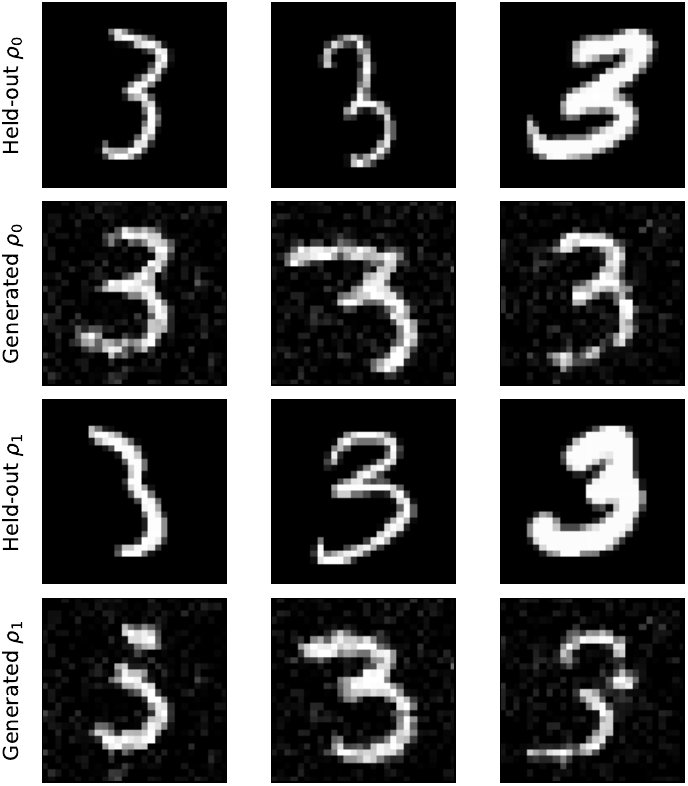}\\[-2pt]
    \textbf{(b) Velocity}
  \end{minipage}
  \caption{Samples from the fitted MNIST models at $n_{\mathrm{tr}} = n_{\mathrm{te}} = 50$
per population, for (a)~the denoiser and (b)~the velocity representation. Rows show, from
top to bottom, held-out observations from $\rho_0$ (digit 3), samples from the $\rho_0$
model, held-out observations from $\rho_1$ (90\% digit 3, 10\% digit 8) and samples from
the $\rho_1$ model; columns are unpaired examples, so a row need not contain a
target-class image. Samples are generated with the
probability flow ODE with $150$ steps. Although the samples are visibly noisy, models
trained at this sample size already achieve the power in
\cref{tab:mnist-power-mean-variance}.}
  \label{fig:mnist-generation-n50}
\end{figure}

\begin{figure}[!t]
  \centering
  \begin{minipage}[t]{0.48\textwidth}
    \centering
    \includegraphics[width=0.88\linewidth]{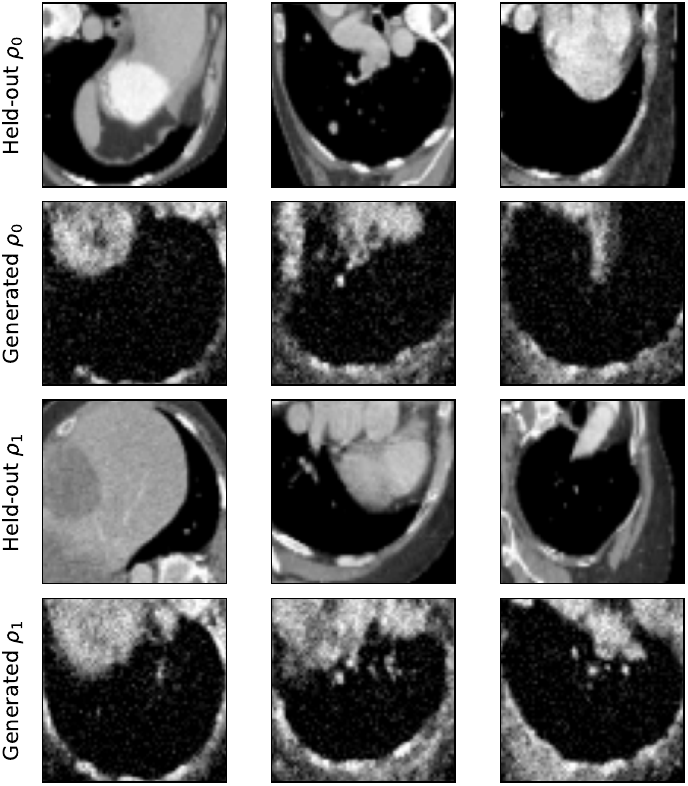}\\[-2pt]
    \textbf{(a) Denoiser}
  \end{minipage}\hfill
  \begin{minipage}[t]{0.48\textwidth}
    \centering
    \includegraphics[width=0.88\linewidth]{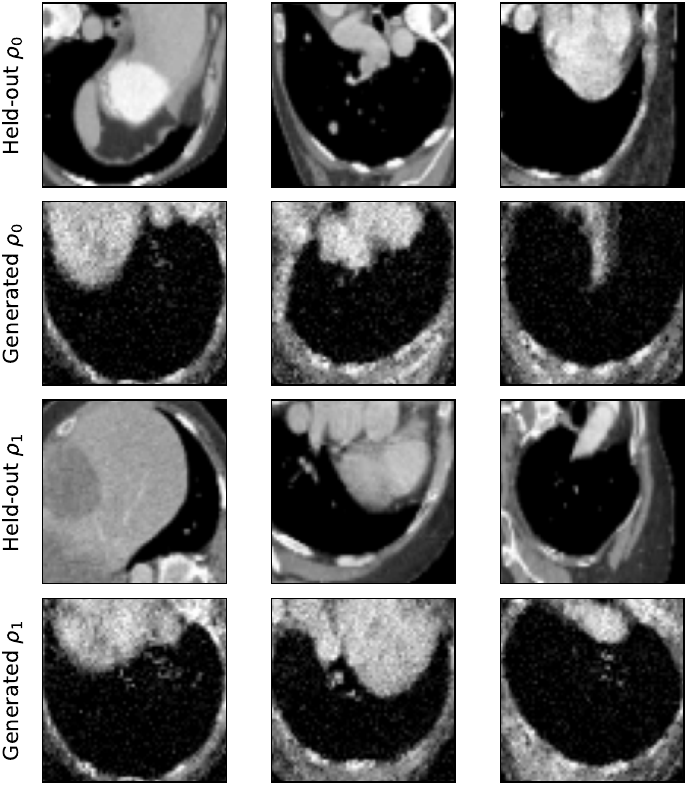}\\[-2pt]
    \textbf{(b) Velocity}
  \end{minipage}
  \caption{Samples from the fitted OrganAMNIST models at $n_{\mathrm{tr}} = n_{\mathrm{te}} = 100$ per population, for (a)~the denoiser and (b)~the velocity representation. Rows show, from top to bottom, held-out observations from $\rho_0$ (Left Lung), samples from the $\rho_0$ model, held-out observations from $\rho_1$ (90\% Left Lung, 10\% Right Lung) and samples from the $\rho_1$ model; columns are unpaired examples, so a row need not contain a target-class image. Samples are generated with the probability flow ODE with $150$ steps. Although the samples are visibly noisy, models trained at this sample size already achieve the power in \cref{tab:organamnist-power-mean-variance}.}
  \label{fig:organ-amnist-generation-n100}
\end{figure}

\begin{figure}[!t]
  \centering
  \begin{minipage}[t]{0.48\textwidth}
    \centering
    \includegraphics[width=0.88\linewidth]{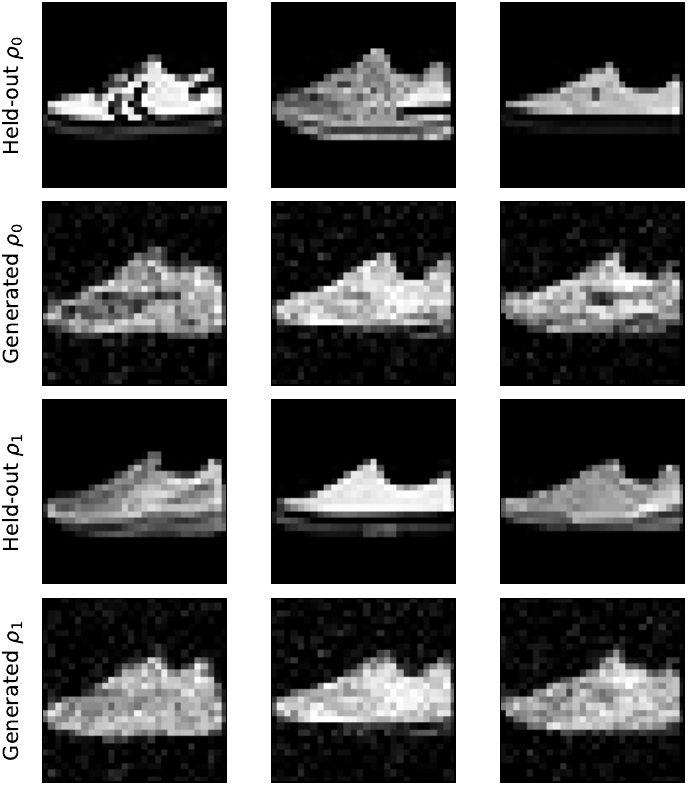}\\[-2pt]
    \textbf{(a) Denoiser}
  \end{minipage}\hfill
  \begin{minipage}[t]{0.48\textwidth}
    \centering
    \includegraphics[width=0.88\linewidth]{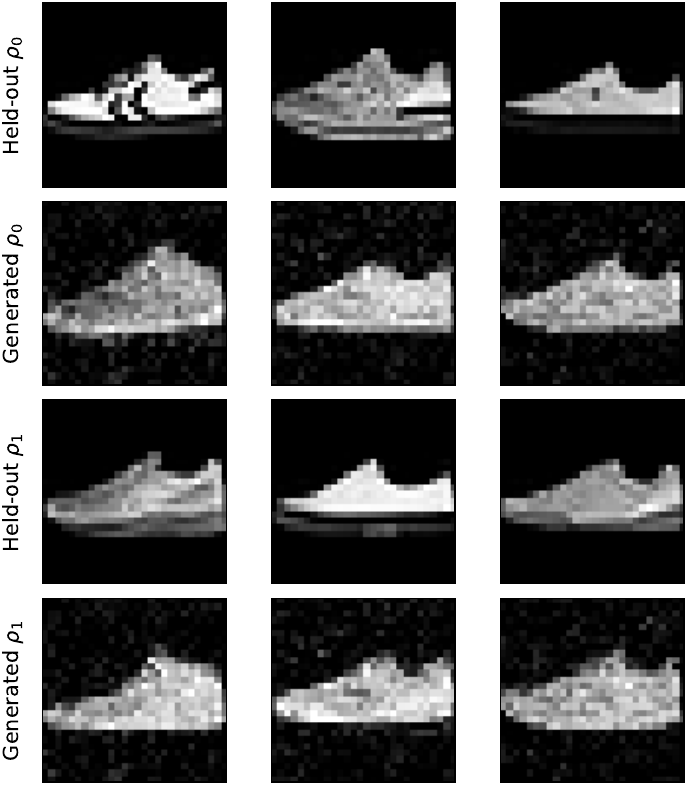}\\[-2pt]
    \textbf{(b) Velocity}
  \end{minipage}
  \caption{Samples from the fitted Fashion-MNIST models at $n_{\mathrm{tr}} = n_{\mathrm{te}} = 50$
per population, for (a)~the denoiser and (b)~the velocity representation. Rows show, from
top to bottom, held-out observations from $\rho_0$ (Sneaker), samples from the $\rho_0$
model, held-out observations from $\rho_1$ (90\% Sneaker, 10\% Ankle Boot) and samples
from the $\rho_1$ model; columns are unpaired examples, so a row need not contain a
target-class image. Samples are generated with the
probability flow ODE with $150$ steps. Although the samples are visibly noisy, models
trained at this sample size already achieve the power in
\cref{tab:fashion-mnist-power-mean-variance}.}
  \label{fig:fashion-mnist-generation-n50}
\end{figure}

\subsection{Ablation on Model Architecture}
\label{sec:ablation}
We compare four model architectures on MNIST and Fashion-MNIST: \texttt{Conditioned}, \texttt{Separate}, \texttt{Common Contrast}, and \texttt{Restricted Common Contrast}. All four architectures estimate the two fields $\widehat f^{\,0}$ and $\widehat f^{\,1}$,
where $\widehat f^{\,\ell}(s,x)$ estimates the denoiser $\eta^\ell_s(x)$ or the velocity
$u^\ell_s(x)$, depending on the representation. They differ only in how parameters are
shared between the two populations; the interpolant, the regression targets of
\cref{alg:fit} and the test are unchanged.
\begin{itemize}[leftmargin=*, itemsep=2pt]
  \item \texttt{Separate}: two independent networks,
        $\widehat f^{\,\ell}(s,x) = f_{\theta_\ell}(s,x)$, each fitted on
        $\mathcal{D}^{\mathrm{tr}}_\ell$ alone. This is the architecture used in all other
        experiments.
  \item \texttt{Conditioned}: a single network with a binary population input,
        $\widehat f^{\,\ell}(s,x) = f_\theta(s,x,\ell)$, which shares features between the
        populations while allowing the two fields to differ.
  \item \texttt{Common Contrast}: a common field $h_\theta$ and a contrast field $d_\phi$,
        both of full capacity, with
        $\widehat f^{\,0} = h_\theta - \tfrac12 d_\phi$ and
        $\widehat f^{\,1} = h_\theta + \tfrac12 d_\phi$. This reparameterises rather than
        restricts the pair of fields, since any pair is recovered by
        $h = (f^0+f^1)/2$ and $d = f^1 - f^0$; by \cref{prop:invariance}, the population
        contrast vanishes at any single radius if and only if $\mathcal{H}_0$ holds.
  \item \texttt{Restricted Common Contrast}: the same decomposition with a lower-capacity
        contrast network. The zero
        contrast of $\mathcal{H}_0$ remains representable, but under $\mathcal{H}_1$, the
        restricted $d_\phi$ may capture only part of the population contrast, which can
        reduce variance at the cost of attenuating the signal.
\end{itemize}
The shared architectures are fitted by minimising the sum of the two regression losses of
\cref{alg:fit}, each on its own training split. All four use the training splits only,
so \cref{prop:validity} applies unchanged: the architecture affects power but not validity.
For the two common--contrast architectures, the local contrast reduces to
$\psi_s = -2\, d_\phi(s,\omega_s)\cdot\bigl(y - h_\theta(s,\omega_s)\bigr)$, where $y$ is
the regression target, so the two-sample signal passes through $d_\phi$ alone.

All experiments use $(n_{\mathrm{train}} = n_{\mathrm{test}}=50)$ observations per population and five training seeds (50 test sets each). We report mean empirical power and unbiased sample variance across seeds. We use $\epsilon = 0.3$ for the contamination, whereas the main experiment uses $\epsilon = 0.1$.

\begin{table}[t]
\centering
\caption{Architecture ablation on MNIST (source digit 3, target digit 8) at contamination
fraction $\epsilon = 0.3$, higher than the $0.1$ used in the main experiments, with
$n_{\mathrm{tr}} = n_{\mathrm{te}} = 50$ per population. Columns are the four
architectures described in \cref{sec:ablation}; \texttt{Separate} is the architecture
used in all other experiments. Entries report the mean over five training seeds (50 test sets each; \cref{app:implementation-details}) $\pm$ the
unbiased sample variance across seeds; bold marks the best variant for each
architecture.}
\label{tab:mnist-power-architecture-ablation}
\small
\begin{adjustbox}{max width=\textwidth}
\begin{tabular}{lcccc}
\toprule
Method & Conditioned & Separate & Common Contrast & Restricted Common Contrast \\
\midrule
PBI-Denoiser\_U & {\bfseries\boldmath\res{0.796}{0.0289}} & {\bfseries\boldmath\res{0.820}{0.0114}} & {\bfseries\boldmath\res{0.792}{0.0297}} & {\bfseries\boldmath\res{0.784}{0.0247}} \\
PBI-Denoiser\_J & \res{0.736}{0.0075} & \res{0.712}{0.0021} & \res{0.680}{0.0422} & \res{0.680}{0.0150} \\
PBI-Velocity\_U & \res{0.784}{0.0549} & \res{0.768}{0.0225} & \res{0.744}{0.0043} & \res{0.668}{0.0825} \\
PBI-Velocity\_J & \res{0.668}{0.0801} & \res{0.712}{0.0267} & \res{0.504}{0.0087} & \res{0.520}{0.0988} \\
\bottomrule
\end{tabular}
\end{adjustbox}
\end{table}

On MNIST, the \texttt{Conditioned} and \texttt{Separate} architectures achieve the highest mean power for all four PBI variants, as shown in \cref{tab:mnist-power-architecture-ablation}. 
The strongest configuration is \texttt{PBI-Denoiser\_U} with Separate models, which reaches power $0.820$, compared with $0.796$, $0.792$, and $0.784$ for Conditioned, Common Contrast, and Restricted Common Contrast, respectively. 
Architectural differences are more pronounced for \texttt{PBI-Velocity\_J}: 
Separate models achieve $0.712$, exceeding Common Contrast and Restricted Common Contrast by $20.8$ and $19.2$ percentage points. Separate models also exhibit lower seed-to-seed variance than Conditioned models across all PBI variants. 

\begin{table}[t]
\centering
\caption{Architecture ablation on Fashion-MNIST (source digit 3, target digit 8) at contamination
fraction $\epsilon = 0.3$, higher than the $0.1$ used in the main experiments, with
$n_{\mathrm{tr}} = n_{\mathrm{te}} = 50$ per population. Columns are the four
architectures described in \cref{sec:ablation}; \texttt{Separate} is the architecture
used in all other experiments. Entries report the mean over five training seeds (50 test sets each; \cref{app:implementation-details}) $\pm$ the
unbiased sample variance across seeds; bold marks the best variant for each
architecture.}
\label{tab:fashion-mnist-power-architecture-ablation}
\small
\begin{adjustbox}{max width=\textwidth}
\begin{tabular}{lcccc}
\toprule
Method & Conditioned & Separate & Common Contrast & Restricted Common Contrast \\
\midrule
PBI-Denoiser\_U & {\bfseries\boldmath\res{1.000}{0.0000}} & {\bfseries\boldmath\res{1.000}{0.0000}} & {\bfseries\boldmath\res{0.996}{0.0001}} & {\bfseries\boldmath\res{1.000}{0.0000}} \\
PBI-Denoiser\_J & \res{0.972}{0.0011} & \res{0.968}{0.0037} & \res{0.892}{0.0107} & \res{0.956}{0.0031} \\
PBI-Velocity\_U & \res{0.912}{0.0077} & \res{0.868}{0.0177} & \res{0.896}{0.0059} & \res{0.952}{0.0023} \\
PBI-Velocity\_J & \res{0.792}{0.0377} & \res{0.740}{0.0590} & \res{0.728}{0.0181} & \res{0.916}{0.0043} \\
\bottomrule
\end{tabular}
\end{adjustbox}
\end{table}
\Cref{tab:fashion-mnist-power-architecture-ablation} shows ablation results on Fashion-MNIST. \texttt{PBI-Denoiser\_U} achieves nearly saturated power across all architectures, ranging from $0.996$ to $1.000$, leaving little scope to distinguish their performance. 
Architectural choices have a larger effect on the velocity variants. 
Restricted Common Contrast achieves the highest mean power for both \texttt{PBI-Velocity\_U} ($0.952$) and \texttt{PBI-Velocity\_J} ($0.916$), improving over Separate models by $8.4$ and $17.6$ percentage points, respectively. 
It also yields the lowest sample variance among the four architectures for both velocity variants.

Across both datasets, uniform weighting yields higher mean power than Jeffreys weighting for every matched architecture and parameterisation. 
\texttt{PBI-Denoiser\_U} is the strongest variant within each architecture, whereas the preferred architecture depends on the dataset and parameterisation. 
These results indicate that architectural constraints can materially affect empirical power, particularly for the velocity variants. 
The comparisons are descriptive and do not establish statistical significance.

\section{Observation-Level Asymmetry under Corruption}
\label{sec:interpretability}

The statistics of \cref{sec:method} are differences of group means of observation-level scores $\widehat\Phi^{\mathrm{rep}}_w(\omega)$, so each held-out observation carries its own share of the evidence against $\mathcal{H}_0$. This appendix checks, on a controlled example, that these shares localise the discrepancy: when only part of $\rho_1$ departs from $\rho_0$, the observations responsible for the departure receive the largest scores.

\paragraph{Setup:}
We reuse the two-dimensional construction of \cref{fig:path_invariance}. The reference distribution $\rho_0$ is an equal-weight mixture of three Gaussians with means $(-1.6,-1)$, $(0,1.75)$ and $(1.6,-1)$ and covariance $0.1\,\mathbb{I}_2$. The second distribution is a contaminated copy,
\begin{equation}
\rho_1 = 0.9\,\rho_0 + 0.1\,C,
\end{equation}
where the corruption $C$ is one of three structural perturbations of increasing difficulty:
(i) a \emph{new cluster} $C=\mathcal{N}\big((0,-0.1),\,0.1\,\mathbb{I}_2\big)$ between the modes of $\rho_0$, where $\rho_0$ has almost no mass;
(ii) a \emph{shifted mode} $C=\mathcal{N}\big((2.2,-0.4),\,0.1\,\mathbb{I}_2\big)$, a copy of the right-hand mode translated by $(0.6,0.6)$ that partially overlaps it;
(iii) a \emph{deformed mode} $C=\mathcal{N}\big((0,1.75),\,\mathrm{diag}(0.9,\,0.02)\big)$, which shares its centre with the top mode of $\rho_0$ but is stretched horizontally.
Each held-out sample contains $500$ observations per group, so that about $50$ observations of $\mathcal{D}_1^{\mathrm{te}}$ are corrupted. The corruption labels are used only for evaluation and are never available during testing.

\paragraph{Fields:}
We consider two sets of fields. The \emph{exact} fields are the posterior-mean denoisers of $\rho_0$ and $\rho_1$, available in closed form for Gaussian mixtures. The \emph{fitted} fields mimic the pipeline of \cref{alg:pbi}: for each group, we draw an independent training sample of $500$ observations, fit a Gaussian mixture by expectation--maximisation with the number of components selected by BIC over $\{1,\dots,6\}$, and use the posterior-mean denoiser of the fitted mixture. The fitted fields are therefore misspecified whenever the selected mixture does not resolve the corruption.

\paragraph{Scores:}
For each held-out observation we compute the Monte Carlo score $\widehat\Phi^{\mathrm{den}}_w(\omega)$ of \cref{eq:mc_witness} with $K=64$ draws, truncation $\tau=0.05$ and $\mu$ uniform on $\mathcal{I}_\tau$, for the uniform ($w_{\mathrm{U}}$) and Jeffreys ($w^{\mathrm{den}}_{\mathrm{J}}$) weights, and report the \emph{asymmetry}
\begin{equation}
A(\omega) = -\widehat\Phi^{\mathrm{den}}_w(\omega).
\end{equation}
By \cref{lem:risk_contrasts}, $A$ is large when the channel-$1$ denoiser reconstructs $\omega$ much better than the channel-$0$ denoiser, and it has mean zero on the part of $\rho_1$ that coincides with $\rho_0$ up to reweighting. As a geometric counterpart of \cref{fig:path_invariance}, we also compute the \emph{mirror mismatch} $|X_0-\omega|$ of an observation $\omega\sim\rho_1$: we integrate the probability flow ODE from $t=1-\epsilon$ to the bottleneck with the velocity of the $\rho_1$ half, continue to $t=\epsilon$ with the velocity of the $\rho_0$ half, and record the endpoint $X_0$ ($\epsilon=2\times10^{-3}$, adaptive Runge--Kutta with relative tolerance $10^{-6}$). Under $\mathcal{H}_0$ the path is reflection symmetric and $X_0=\omega$ (\cref{prop:invariance}). Because it requires two ODE solves per observation, the mirror mismatch is computed on a subset of $30$ corrupted and $90$ uncorrupted observations per seed.

\paragraph{Metrics:}
We rank the held-out observations from $\rho_1$ by $A$ and report the area under the ROC curve (AUROC) for separating corrupted from uncorrupted observations, and the precision among the $k$ highest-ranked observations (P@$k$), where $k$ is the number of corrupted observations; a random ranking gives P@$k=0.10$. We also report the rejection rate of the two-sided permutation test of \cref{sec:permutation} with the uniform denoiser statistic, $R=999$ permutations and $\kappa=0.05$. All numbers are averaged over $10$ independent seeds.

\begin{table}[t]
\caption{Observation-level localisation under structural corruption,
$\rho_1 = 0.9\,\rho_0 + 0.1\,C$, with 500 held-out observations per group, averaged over
10 seeds. $\bar A$ is the mean asymmetry (denoiser, uniform weight) of observations from
$\rho_0$ and of uncorrupted and corrupted observations from $\rho_1$. AUROC (uniform and
Jeffreys weights) and P@$k$ (uniform weight) measure how well $A$ ranks the corrupted
observations within the $\rho_1$ sample; random ranking gives AUROC $0.5$ and
P@$k = 0.10$. Mirror is the AUROC of the geometric mismatch $|X_0 - \omega|$, computed on
30 corrupted and 90 uncorrupted observations per seed. Rej.\ is the rejection rate of the
test at $\kappa = 0.05$ with $R = 999$ permutations.}
\label{tab:interpretability}
\centering\small\setlength{\tabcolsep}{3.5pt}
\begin{tabular}{llcccccccc}
\toprule
& & \multicolumn{3}{c}{$\bar A$} & \multicolumn{2}{c}{AUROC} & & & \\
\cmidrule(lr){3-5}\cmidrule(lr){6-7}
Fields & Corruption $C$ & $\rho_0$ & $\rho_1$ matched & $\rho_1$ corrupted & U & J & P@$k$ & Mirror & Rej. \\
\midrule
Exact  & new cluster   & $-0.005$ & $-0.005$ & $0.134$ & 0.999 & 0.998 & 0.99 & 1.000 & 1.0 \\
       & shifted mode  & $-0.012$ & $-0.008$ & $0.207$ & 0.979 & 0.981 & 0.78 & 0.942 & 1.0 \\
       & deformed mode & $-0.005$ & $-0.006$ & $0.138$ & 0.904 & 0.922 & 0.44 & 0.796 & 0.7 \\
\midrule
Fitted & new cluster   & $-0.006$ & $-0.005$ & $0.139$ & 0.982 & 1.000 & 0.89 & 0.990 & 0.9 \\
       & shifted mode  & $-0.010$ & $-0.010$ & $0.204$ & 0.982 & 0.981 & 0.79 & 0.925 & 0.9 \\
       & deformed mode & $-0.006$ & $-0.006$ & $0.134$ & 0.828 & 0.867 & 0.38 & 0.750 & 0.4 \\
\bottomrule
\end{tabular}
\end{table}

\paragraph{Results:}
\Cref{tab:interpretability} shows that the asymmetry is concentrated on the corrupted observations. Observations from $\rho_0$ and uncorrupted observations from $\rho_1$ have a mean asymmetry close to zero, whereas corrupted observations score between $20$ and $40$ times higher, the observation-level counterpart of the broken reflection symmetry in the right panel of \cref{fig:path_invariance}. The new cluster and the shifted mode are ranked almost perfectly (AUROC between $0.98$ and $1.00$), and the corresponding tests reject in every or nearly every seed. The deformed mode is harder: its observations near the shared centre are individually indistinguishable from the uncorrupted top mode, so only its tails can be localised (P@$k\approx0.4$), and the global test is also less powerful. Fitted fields reproduce these conclusions, with a loss concentrated on the deformed mode, where the selected mixture does not always resolve the corruption; this is the observation-level counterpart of the dependence of power on field accuracy in \cref{prop:consistency}. The geometric mirror mismatch orders the corruptions in the same way. Still, it is less discriminative than $A$ and requires ODE solves, thereby supporting the use of the regression-risk score for localisation.

\paragraph{Scope:}
The scores rank observations but do not provide per-observation error control: the permutation $p$-value is a statement about the two distributions, and flagging individual observations with guarantees would require an additional multiple-testing or conformal procedure. We omit the velocity representation from \cref{tab:interpretability} because, for fields of the affine form of \cref{lem:relations}, the velocity witness is a reweighting of the denoiser witness. Indeed, $U_s(\omega,z)-u^\ell_s(\omega_s) = a'_s\omega + \sigma'_s z - \lambda_s\eta^\ell_s(\omega_s) - (\sigma'_s/\sigma_s)(a_s\omega+\sigma_s z) = \lambda_s\big(\omega-\eta^\ell_s(\omega_s)\big)$, so $\psi^{\mathrm{vel}}_s = \lambda_s^2\,\psi^{\mathrm{den}}_s$ pointwise; in particular the Jeffreys-weighted velocity and denoiser scores coincide. The two representations differ in practice only because the velocity and denoiser networks are trained separately.

\section{Related Work}
\label{sec:related}

\paragraph{Two-sample tests with learned statistics:}
Kernel tests based on the maximum mean discrepancy (MMD) compare mean embeddings of the two distributions in a reproducing kernel Hilbert space \citep{gretton2012kernel}, and their power depends strongly on the kernel and on the representation of the data. Most adaptive tests, therefore, learn part of the statistics from data, and they differ in how this learning is separated from calibration. Sample-splitting methods fit the statistic on one part of the data and evaluate it on the other, so that the learned statistic can be treated as fixed during calibration: kernels chosen to maximise a power criterion \citep{gretton_optimal_2012, sutherland_generative_2021}, deep kernels \citep{liu2020learning}, classifiers whose held-out accuracy or logits serve as the statistic \citep{lopezpaz2017revisiting, cheng2022c2st, kim_classification_2020}, and real-valued witness functions whose mean difference is tested on held-out data \citep{kubler_witness_2022}, fitted with automated machine learning in AutoTST \citep{kubler2022autotst}. Other methods avoid the split by learning representations on auxiliary data \citep{kirchler_two-sample_2020}, by accounting for kernel selection through selective inference \citep{kubler_learning_2020}, or by aggregating or fusing tests over a collection of kernels with jointly calibrated thresholds \citep{jmlr_mmdagg, biggs_mmd-fuse_2023}. Our tests follow the sample-splitting template, but the witness is not trained to discriminate between the samples: it is obtained from two arcs of a generative model, each fitted to a single population, and its population mean is characterised exactly in terms of the fields of a stochastic interpolant. Since the statistic is a difference of means of observation-level scores, it also localises the discrepancy in data space, which complements tests that learn interpretable test locations \citep{chwialkowski_fast_2015, jitkrittum_interpretable_2016}.

\paragraph{Scores, denoisers and divergences:}
Score functions have long been used for model criticism, most prominently through the kernel Stein discrepancy for goodness-of-fit testing against a model with a known score \citep{liu_kernelized_2016, chwialkowski_kernel_2016}, and more recently for simple-versus-simple hypothesis testing and change-point detection when the scores of both hypotheses are known \citep{moushegian_diffusion-based_2025}. In our setting, neither distribution is known, and both fields are learned from finite samples. The link between Gaussian smoothing, score differences and the Kullback--Leibler divergence goes back to de Bruijn's identity and its extension to relative entropy \citep{lyu_interpretation_2012}, and to the I-MMSE relation and its mismatched-estimation counterpart, which expresses the relative entropy between two Gaussian-smoothed distributions as an integral of excess denoising error \citep{guo_mutual_2004, verdu_mismatched_2010}. The same identities underlie likelihood bounds for weighted denoising score matching \citep{song_maximum_2021, kingma_variational_2021} and information-theoretic views of diffusion models \citep{kong_information-theoretic_2023}.\cref{prop:jeffreys} is the symmetrised form of the mismatched-estimation identity, written along the Gaussian-bottleneck path in terms of the two fields that the test estimates; its role here is to identify the population target of the Jeffreys-weighted statistics.

\paragraph{Generative models as tools for inference:}
Beyond synthesis, diffusion and flow models have been used for likelihood evaluation through the probability flow ODE \citep{song_score-based_2020}, for out-of-distribution detection \citep{graham_denoising_2023, liu_unsupervised_2023}, and for zero-shot classification by comparing denoising errors across class-conditional models \citep{li_your_2023, clark_text--image_2023}. The last line of work is closest in spirit to our denoising-risk contrast, which compares the held-out regression risks of two models. Still, it uses these comparisons for \emph{prediction} rather than for calibrated inference. Paths of intermediate distributions have also been used to estimate density ratios, by telescoping over a sequence of bridges \citep{rhodes_telescoping_2020} or by integrating infinitesimal classifiers along a continuum \citep{choi_density_2022}; our statistics likewise accumulate local evidence along a path, but we target a statistical test with finite-sample validity rather than a density-ratio estimate in this line of work. Finally, independently trained diffusion models and flow-based models learn remarkably similar fields \citep{kadkhodaie_generalization_2024, zhang_emergence_2024, english_path_2026, briq_amazing_2025}, which supports treating these fields as stable statistical objects.

\paragraph{Midpoint inference from flow models:}
Closest to our work is a recent line of research built on the zero-flow criterion \citep{wang_zero-flow_2026}: for rectified flow with independently coupled endpoints, the velocity field vanishes at $t=\tfrac12$ if and only if the two endpoint distributions coincide. The criterion was introduced to certify conditional sufficiency in representation learning \citep{wang_zero-flow_2026}, has been turned into a midpoint divergence for training one-step generators, including a generalisation to symmetric stochastic interpolants \citep{shlenskii_midpoint_2026}, and, concurrent to our work, into a two-sample test that learns the midpoint velocity on a training split and evaluates its alignment with the displacements of held-out pairs, calibrated by a studentised asymptotic approximation \citep{wang_zero-flow_2026}. These approaches and ours rest on the same symmetry: when the endpoints of a symmetric interpolant are exchangeable, its law is invariant under the time reflection $t\mapsto 1-t$. They exploit it differently. The zero-flow tests read the symmetry off a single field at a single time, which must be learned from pairs of observations drawn from both populations. For the Gaussian-bottleneck interpolant, the midpoint velocity vanishes identically under both hypotheses, and the evidence is instead spread across noise levels. Each half of the path is a Gaussian channel acting on one population, so the symmetry becomes an equality between two separately trained fields at every noise level (\cref{prop:invariance}). Each field is, after rescaling its input, the denoiser of a standard diffusion model trained on a single population, which lets the test reuse diffusion training, yields observation-level scores, and admits permutation calibration with finite-sample validity.
\section{Discussion and Limitations}
\label{sec:conclusion}

\begin{figure}[!t]
  \centering
  \includegraphics[width=0.9\textwidth]{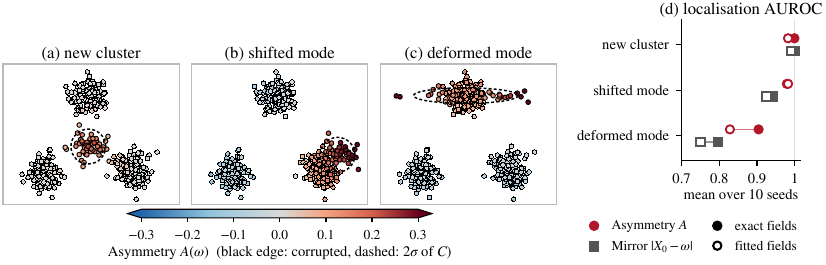}
  \caption{Observation-level localisation ($\rho_1 = 0.9\,\rho_0 + 0.1\,C$, 500 held-out
  observations per group.)
  (a)--(c)~Held-out $\rho_1$ observations coloured by the asymmetry
  $A(\omega) = -\widehat{\Phi}^{\mathrm{den}}_{U}(\omega)$, computed with the exact
  denoisers for one seed; positive values mean the $\rho_1$ denoiser reconstructs
  $\omega$ better than the $\rho_0$ denoiser. Black edges mark corrupted observations,
  and dashed ellipses show the $2\sigma$ region of $C$. For the deformed mode,
  uncorrupted points near the shared centre also score highly, so only the tails are
  localised. (d)~AUROC for ranking the corrupted observations within the $\rho_1$
  sample, averaged over 10 seeds, for $A$ and for the geometric mirror mismatch
  $|X_0-\omega|$, with exact (filled) and fitted (open) fields. Setup and full results: \cref{sec:interpretability,tab:interpretability}.}
  \label{fig:localisation}
\end{figure}

We showed that the generative dynamics learned by transport-based models are useful
objects of statistical inference, not only tools for generation. Routing both
distributions through a shared Gaussian bottleneck turns the null hypothesis into a
reflection symmetry of the path, which held-out regression risks convert into
observation-level scores. Permutation calibration makes the resulting tests valid in
finite samples for any trained network. Across four benchmarks, they improve on
the strongest baseline by up to 33 percentage points at an equal total sample budget.

\textbf{Sample splitting:} The main cost of our approach is that the fields must be learned, which requires splitting each dataset into training and testing sets. This cost is shared by the learned baselines we compare against and by classical kernel tests when the bandwidth is selected from the data. A tempting way to avoid the split is to sample from the Gaussian midpoint and check the symmetry of the learned paths on generated data. Such a test is not calibrated: its type-I error is controlled only if the modelling error is negligible. Moreover, \citet{english_is_2026} shows that inferences drawn from a learned path agree across training and held-out data, but not with data newly sampled from the same model. Their observation that path lengths on real and generated data align when generation is accurate suggests a criterion for early stopping and model selection. Because it can be computed on the training split alone, it would leave the validity of our tests intact.

\textbf{Low-sample regimes and Model quality:} Deep generative models might be expected to overfit when few observations are available. Recent work identifies several sources of implicit regularisation in diffusion and flow models, arising from the training dynamics, the architecture and the regression objective itself \citep{bonnaire_why_2025,farghly_implicit_2025,bertrand_closed-form_2025}. Consistent with this, our tests match or exceed the strongest baseline, even at the smallest sample size of every image benchmark.
Overfitting cannot inflate the type-I error rate in any case, since calibration uses only held-out data; empirically, it stays close to $\kappa$
(\cref{fig:type1}).
For modelling, we used off-the-shelf architectures and training settings without any hyperparameter search. Our goal was two-sample testing rather than generation, and we wanted to show that the tests remain powerful under modest computational budgets. Samples from our trained models (\cref{fig:mnist-generation-n50,fig:organ-amnist-generation-n100,fig:fashion-mnist-generation-n50}) show that a perfect generative model is not needed for a powerful test. Conversely, since power is governed by how accurately the fields are learned (\cref{sec:consistency}), better-trained models, for instance, fine-tuned pretrained ones, should yield more powerful tests. Architecture also matters, mainly for the velocity variants (\cref{sec:ablation}). Power, however, is limited to differences that the fitted fields can represent. If both models converge to the same misspecified approximation, the test has no power even though $\mathcal{H}_1$ holds.

\textbf{Choice of variant:}
No single combination of representation and weighting dominates across benchmarks, so the variant must be fixed in advance. Alternatively, several variants can be combined into a single valid test, for example, with a Bonferroni correction or by permuting the maximum of the standardised statistics under shared permutations. We have not yet evaluated such aggregated tests and leave them, together with the use of pretrained models, to future work.


\acks{We extend our gratitude to Prof. Masato Taki for helpful discussions.  We gratefully acknowledge funding by the Deutsche Forschungsgemeinschaft (DFG, German Research Foundation) – project number
459422098, JST SPRING, Japan Grant Number JPMJSP2104, JST CREST (PMJCR2015), JSPS KAKENHI (24K02905, 25H01107), JST ERATO (JPMJER2601), the National Research Foundation, Singapore and the Ministry of Digital Development and Information under the AI Visiting Professorship Programme (award number AIVP-2024-004). Any opinions, findings and conclusions or recommendations expressed in this material are those of the author(s) and do not reflect the views of the National Research Foundation, Singapore and the Ministry of Digital Development and Information.}


\newpage

\appendix
\section{Experimental Details and Notations}
\label{app:implementation-details}

\cref{tab:notation} summarises the notation used throughout the paper.

\begin{table}[ht]
\centering
\caption{Summary of notation.}
\label{tab:notation}
\begin{tabular}{ll}
\toprule
Symbol & Meaning \\
\midrule
$\rho_0, \rho_1$ & Unknown data densities; $\mathcal{H}_0 \colon \rho_0 = \rho_1$ \\
$\mathcal{D}_\ell = \{x_\ell^{(i)}\}_{i=1}^{n_\ell}$ & Dataset drawn i.i.d.\ from $\rho_\ell$, $\ell \in \{0, 1\}$ \\
$\kappa$ & Significance level (often denoted $\alpha$ in the testing literature) \\
$x_t, \rho_t$ & Stochastic interpolant and its density at time $t \in [0,1]$ \\
$z \sim \mathcal{N}(0, \mathbb{I}_d)$ & Gaussian latent variable \\
$s = |t - \tfrac{1}{2}|$ & Radial coordinate, the distance from the Gaussian bottleneck \\
$t_s^0, t_s^1$ & Mirrored times $\tfrac{1}{2} - s$ and $\tfrac{1}{2} + s$ \\
$a_s, \sigma_s$ & Radial signal and noise schedules \\
$\zeta_s^\ell$ & Density of channel $\ell$ at radius $s$, $\zeta_s^\ell = \rho_{t_s^\ell}$; $\zeta_0^\ell = \mathcal{N}(0, \tfrac{1}{2}\mathbb{I}_d)$ and $\zeta_{1/2}^\ell = \rho_\ell$ \\
$\eta_s^\ell, u_s^\ell$ & Population denoiser and outward velocity of channel $\ell$ \\
$\widehat{\eta}_s^\ell, \widehat{u}_s^\ell$ & Learned counterparts, fitted on $\mathcal{D}_\ell^{\mathrm{tr}}$ \\
$\lambda_s, \chi_s$ & Known positive scalars relating denoiser, velocity and score contrasts \\
$\omega, \omega_s$ & Held-out observation and its corrupted version at radius $s$ \\
$\psi_s^{\mathrm{den}}, \psi_s^{\mathrm{vel}}$ & Local risk contrasts at radius $s$ \\
$D_s^{\mathrm{rep}}$ & Population signal at radius $s$, $\mathrm{rep} \in \{\mathrm{den}, \mathrm{vel}\}$ \\
$\mathcal{I}_\tau$ & Truncated path interval $[\tau, \tfrac{1}{2} - \tau]$ \\
$w_{\mathrm{U}}, w_{\mathrm{J}}^{\mathrm{rep}}$ & Uniform and Jeffreys path weights \\
$\widehat{T}_w^{\mathrm{rep}}$ & Test statistic with representation $\mathrm{rep}$ and weight $w$ \\
$\varepsilon_{\ell|k}(s)$ & Error of learned field $\ell$ measured under channel $k$ \\
\bottomrule
\end{tabular}
\end{table}

\paragraph{FC-MLP backbone:}
For vector-valued observations $x\in\mathbb{R}^{d}$, the network input is the concatenation $[t,x]$ of the scalar time variable and the observation.  
The implementation supports appending a binary population indicator for a conditioned model, but the experiments reported here use separate networks for $\rho_0$ and $\rho_1$ and therefore use only $[t,x]$.  
The HDGM backbone has four hidden layers of width 128 with SiLU activations and a linear $10$ dimensional output.

\paragraph{U-Net backbone:}
For image-valued observations, we use a three-level encoder--decoder U-Net.
The encoder applies double-convolution blocks at channel widths $c,2c,4c,8c$, with factor-two max pooling between levels.
Conditional time embeddings are introduced at widths $c$, $2c$, and $4c$.  
The decoder uses transposed convolutions for upsampling and concatenates the corresponding encoder features, giving double-convolution inputs of widths $12c$, $6c$, and $3c$.  
A final $1\times1$ convolution maps the $c$-channel representation to the output.  
We set $c=16$ for MNIST and Fashion-MNIST and $c=32$ for OrganAMNIST.

For HDGM, the fixed-budget implementation produces 1300, 1300, 2600 recorded optimisation steps for $n=500$, 1000, and 2000, respectively.  
MNIST and OrganAMNIST use 1,600 at each sample size, while Fashion-MNIST uses 2,500 recorded steps.
All timing records correspond to the \texttt{final\_fixed\_budget} checkpoint-selection rule.

\paragraph{Baseline implementation details:}
We use the official NumPy implementation of MMDAgg \citep{jmlr_mmdagg}, considering ten bandwidths and retaining its default Monte Carlo calibration parameters. 
AutoTST is implemented using the official codebase of \citet{kubler2022autotst}; we use the AutoGluon tabular backend with the \texttt{best\_quality} preset and a time budget of 150 seconds. Image observations are flattened before being passed to AutoTST.

We adapt MMD-D, MMD-O, C2ST-L, and C2ST-S from the implementation of \citet{liu2020learning}. 
C2ST-S computes its statistics from the classifier's hard class predictions, whereas C2ST-L uses its predicted class probabilities. 
MMD-O optimises the bandwidth of a Gaussian kernel applied directly to the observations, while MMD-D jointly learns a neural feature representation and the parameters of a deep kernel by maximising a variance-normalised empirical MMD objective. 
We use convolutional feature networks for image data and fully connected networks for vector-valued data.


\section{Per-Benchmark Power and Type-I Error Curves}
\label{app:power-curves}

For better readability, this appendix shows the power curves of \cref{fig:power-main} and the type-I error curves of \cref{fig:type1} for each benchmark separately, with all ten methods shown individually.
\Cref{fig:hdgm-power-appendix,fig:mnist-power-appendix,fig:organamnist-power-appendix,fig:fashion-mnist-power-appendix} report power, and \cref{fig:hdgm-type1-appendix,fig:mnist-type1-appendix,fig:organamnist-type1-appendix,fig:fashion-mnist-type1-appendix} report the empirical type-I error at level $\kappa = 0.05$; the numerical power values are in the tables of \cref{sec:experiments}.

\providecommand{\hdgmfigdir}{experiments/hdgm}
\providecommand{\mnistfigdir}{experiments/mnist}
\providecommand{\organfigdir}{experiments/organamnist}
\providecommand{\fashionfigdir}{experiments/fashionmnist}

\begin{figure}
  \centering
  \includegraphics[width=0.82\textwidth]{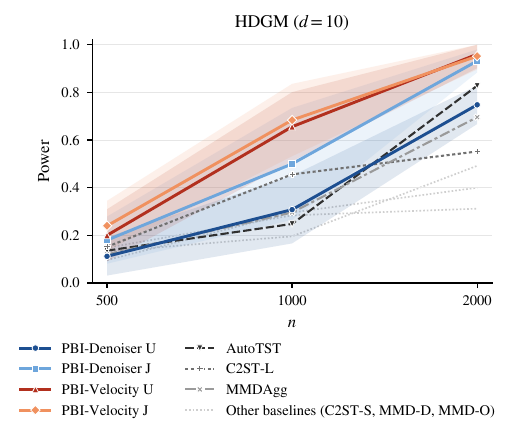}
  \caption{Test power on HDGM at level $\kappa = 0.05$ against the per-population sample size
$n$ ($n_{\mathrm{tr}} = n_{\mathrm{te}} = n$), with all ten methods shown individually; a
condensed version appears in \cref{fig:power-main}. Lines show the mean over five training seeds (50 test sets each); shaded bands show
$\pm 1$ standard deviation across seeds; numerical values are in
\cref{tab:hdgm-power-mean-variance}.}
  \label{fig:hdgm-power-appendix}
\end{figure}
\begin{figure}
  \centering
  \includegraphics[width=0.82\textwidth]{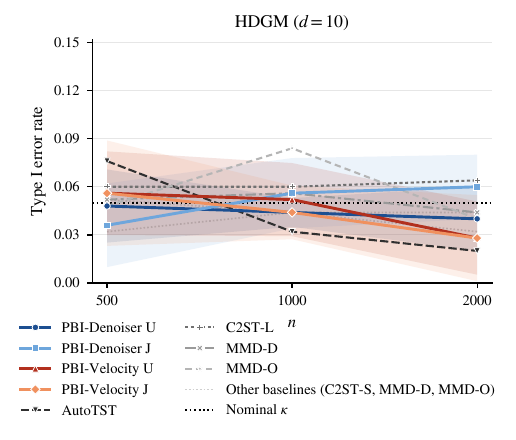}
  \caption{Empirical Type-I error on HDGM at level $\kappa = 0.05$, with both samples drawn
    from $\rho_0$, against the per-population
    sample size $n$ ($n_{\mathrm{tr}} = n_{\mathrm{te}} = n$). Lines show the mean over five training seeds (50 test sets each); shaded bands show
$\pm 1$ standard deviation across seeds; the dotted line marks $\kappa$. The PBI variants stay close to the nominal level, as guaranteed by \cref{prop:validity}.}
  \label{fig:hdgm-type1-appendix}
\end{figure}

\begin{figure}
  \centering
  \includegraphics[width=0.82\textwidth]{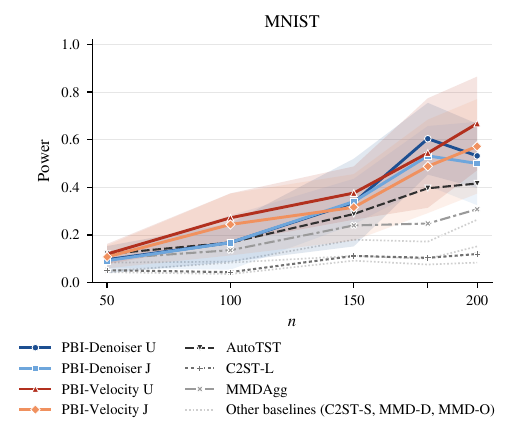}
  \caption{Test power on MNIST at level $\kappa = 0.05$ against the per-population sample size
$n$ ($n_{\mathrm{tr}} = n_{\mathrm{te}} = n$), with all ten methods shown individually; a
condensed version appears in \cref{fig:power-main}. Lines show the mean over five training seeds (50 test sets each); shaded bands show
$\pm 1$ standard deviation across seeds; numerical values are in
\cref{tab:mnist-power-mean-variance}.}
  \label{fig:mnist-power-appendix}
\end{figure}
\begin{figure}
  \centering
  \includegraphics[width=0.82\textwidth]{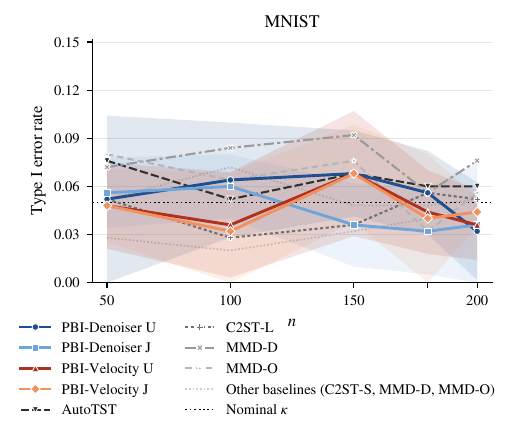}
  \caption{Empirical Type-I error on MNIST at level $\kappa = 0.05$, with both samples drawn
from $\rho_0$, against the per-population
sample size $n$ ($n_{\mathrm{tr}} = n_{\mathrm{te}} = n$). Lines show the mean over five training seeds (50 test sets each); shaded bands show
$\pm 1$ standard deviation across seeds; the dotted
line marks $\kappa$. The PBI variants stay close to the nominal level, as guaranteed by
\cref{prop:validity}.}
  \label{fig:mnist-type1-appendix}
\end{figure}

\begin{figure}
  \centering
  \includegraphics[width=0.82\textwidth]{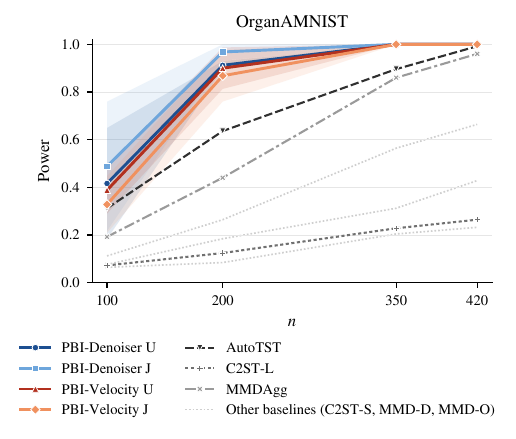}
  \caption{Test power on OrganAMNIST at level $\kappa = 0.05$ against the per-population sample size
$n$ ($n_{\mathrm{tr}} = n_{\mathrm{te}} = n$), with all ten methods shown individually; a
condensed version appears in \cref{fig:power-main}. Lines show the mean over five training seeds (50 test sets each); shaded bands show
$\pm 1$ standard deviation across seeds; numerical values are in
\cref{tab:organamnist-power-mean-variance}.}
  \label{fig:organamnist-power-appendix}
\end{figure}
\begin{figure}
  \centering
  \includegraphics[width=0.82\textwidth]{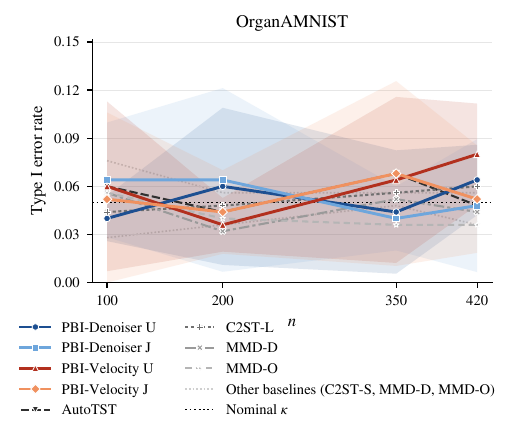}
  \caption{Empirical Type-I error on OrganAMNIST at level $\kappa = 0.05$, with both samples drawn
from $\rho_0$, against the per-population
sample size $n$ ($n_{\mathrm{tr}} = n_{\mathrm{te}} = n$). Lines show the mean over five training seeds (50 test sets each); shaded bands show
$\pm 1$ standard deviation across seeds; the dotted
line marks $\kappa$. The PBI variants stay close to the nominal level, as guaranteed by
\cref{prop:validity}.}
  \label{fig:organamnist-type1-appendix}
\end{figure}

\begin{figure}
  \centering
  \includegraphics[width=0.82\textwidth]{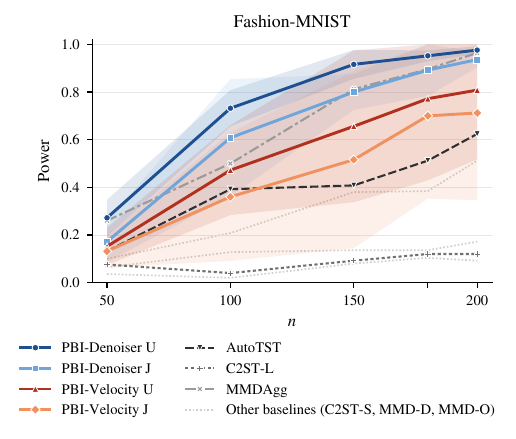}
  \caption{Test power on Fashion-MNIST at level $\kappa = 0.05$ against the per-population sample size
$n$ ($n_{\mathrm{tr}} = n_{\mathrm{te}} = n$), with all ten methods shown individually; a
condensed version appears in \cref{fig:power-main}. Lines show the mean over five training seeds (50 test sets each) and
shaded bands $\pm 1$ standard deviation across seeds; numerical values are in
\cref{tab:fashion-mnist-power-mean-variance}.}
  \label{fig:fashion-mnist-power-appendix}
\end{figure}
\begin{figure}[ht]
  \centering
  \includegraphics[width=0.82\textwidth]{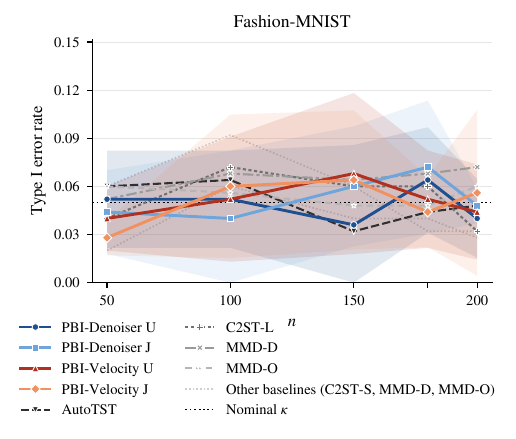}
  \caption{Empirical Type-I error on Fashion-MNIST at level $\kappa = 0.05$, with both samples drawn
from $\rho_0$, against the per-population
sample size $n$ ($n_{\mathrm{tr}} = n_{\mathrm{te}} = n$). Lines show the mean over five training seeds (50 test sets each) and shaded bands $\pm 1$ standard deviation across seeds; the dotted
line marks $\kappa$. The PBI variants stay close to the nominal level, as guaranteed by
\cref{prop:validity}.}
  \label{fig:fashion-mnist-type1-appendix}
\end{figure}

\newpage

\vskip 0.2in
\bibliography{iclr2027_conference}

\begin{thebibliography}{62}
\providecommand{\natexlab}[1]{#1}
\providecommand{\url}[1]{\texttt{#1}}
\expandafter\ifx\csname urlstyle\endcsname\relax
  \providecommand{\doi}[1]{doi: #1}\else
  \providecommand{\doi}{doi: \begingroup \urlstyle{rm}\Url}\fi

\bibitem[Albergo et~al.(2025)Albergo, Boffi, and Vanden-Eijnden]{albergo_stochastic_2025}
M.~Albergo, N.~M. Boffi, and E.~Vanden-Eijnden.
\newblock Stochastic {Interpolants}: {A} {Unifying} {Framework} for {Flows} and {Diffusions}.
\newblock \emph{Journal of Machine Learning Research}, 26\penalty0 (209):\penalty0 1--80, 2025.
\newblock ISSN 1533-7928.
\newblock URL \url{http://jmlr.org/papers/v26/23-1605.html}.

\bibitem[Albergo and Vanden-Eijnden(2022)]{albergo_building_2022}
M.~S. Albergo and E.~Vanden-Eijnden.
\newblock Building {Normalizing} {Flows} with {Stochastic} {Interpolants}.
\newblock Sept. 2022.
\newblock URL \url{https://openreview.net/forum?id=li7qeBbCR1t}.

\bibitem[Anderson(1991)]{anderson1991hotelling}
T.~Anderson.
\newblock Hotelling (1931) the generalization of.
\newblock \emph{Breakthroughs in Statistics}, 1:\penalty0 45, 1991.

\bibitem[Barocas et~al.(2023)Barocas, Hardt, and Narayanan]{barocas_fairness_2023}
S.~Barocas, M.~Hardt, and A.~Narayanan.
\newblock \emph{Fairness and {Machine} {Learning}: {Limitations} and {Opportunities}}.
\newblock MIT Press, Dec. 2023.
\newblock ISBN 9780262376525.
\newblock Google-Books-ID: HuGwEAAAQBAJ.

\bibitem[Bertrand et~al.(2025)Bertrand, Gagneux, Massias, and Emonet]{bertrand_closed-form_2025}
Q.~Bertrand, A.~Gagneux, M.~Massias, and R.~Emonet.
\newblock On the {Closed}-{Form} of {Flow} {Matching}: {Generalization} {Does} {Not} {Arise} from {Target} {Stochasticity}, Dec. 2025.
\newblock URL \url{http://arxiv.org/abs/2506.03719}.
\newblock arXiv:2506.03719 [cs].

\bibitem[Biggs et~al.(2023)Biggs, Schrab, and Gretton]{biggs_mmd-fuse_2023}
F.~Biggs, A.~Schrab, and A.~Gretton.
\newblock {MMD}-{Fuse}: {Learning} and {Combining} {Kernels} for {Two}-{Sample} {Testing} {Without} {Data} {Splitting}.
\newblock Nov. 2023.
\newblock URL \url{https://openreview.net/forum?id=JOkgEY9os2}.

\bibitem[Bonnaire et~al.(2025)Bonnaire, Urfin, Biroli, and Mezard]{bonnaire_why_2025}
T.~Bonnaire, R.~Urfin, G.~Biroli, and M.~Mezard.
\newblock Why {Diffusion} {Models} {Don}’t {Memorize}: {The} {Role} of {Implicit} {Dynamical} {Regularization} in {Training}.
\newblock Oct. 2025.
\newblock URL \url{https://openreview.net/forum?id=BSZqpqgqM0}.

\bibitem[Bortoli(2022)]{bortoli_convergence_2022}
V.~D. Bortoli.
\newblock Convergence of denoising diffusion models under the manifold hypothesis.
\newblock \emph{Transactions on Machine Learning Research}, Aug. 2022.
\newblock ISSN 2835-8856.
\newblock URL \url{https://openreview.net/forum?id=MhK5aXo3gB}.

\bibitem[Briq et~al.(2025)Briq, Kamp, Fried, Cohen, and Kesselheim]{briq_amazing_2025}
R.~Briq, M.~Kamp, O.~Fried, S.~Cohen, and S.~Kesselheim.
\newblock The {Amazing} {Stability} of {Flow} {Matching}.
\newblock Nov. 2025.
\newblock URL \url{https://openreview.net/forum?id=K6ptzIhoKy}.

\bibitem[Bunne et~al.(2023)Bunne, Hsieh, Cuturi, and Krause]{bunne_schrodinger_2023}
C.~Bunne, Y.-P. Hsieh, M.~Cuturi, and A.~Krause.
\newblock The {Schrödinger} {Bridge} between {Gaussian} {Measures} has a {Closed} {Form}.
\newblock In \emph{Proceedings of {The} 26th {International} {Conference} on {Artificial} {Intelligence} and {Statistics}}, pages 5802--5833. PMLR, Apr. 2023.
\newblock URL \url{https://proceedings.mlr.press/v206/bunne23a.html}.

\bibitem[Cheng and Cloninger(2022)]{cheng2022c2st}
X.~Cheng and A.~Cloninger.
\newblock Classification logit two-sample testing by neural networks for differentiating near manifold densities.
\newblock \emph{IEEE Transactions on Information Theory}, 68\penalty0 (10):\penalty0 6631--6662, 2022.
\newblock \doi{10.1109/TIT.2022.3175691}.

\bibitem[Choi et~al.(2022)Choi, Meng, Song, and Ermon]{choi_density_2022}
K.~Choi, C.~Meng, Y.~Song, and S.~Ermon.
\newblock Density {Ratio} {Estimation} via {Infinitesimal} {Classification}, Mar. 2022.
\newblock URL \url{http://arxiv.org/abs/2111.11010}.
\newblock arXiv:2111.11010 [cs.LG].

\bibitem[Chwialkowski et~al.(2016)Chwialkowski, Strathmann, and Gretton]{chwialkowski_kernel_2016}
K.~Chwialkowski, H.~Strathmann, and A.~Gretton.
\newblock A {Kernel} {Test} of {Goodness} of {Fit}.
\newblock In \emph{Proceedings of {The} 33rd {International} {Conference} on {Machine} {Learning}}, pages 2606--2615. PMLR, June 2016.
\newblock URL \url{https://proceedings.mlr.press/v48/chwialkowski16.html}.

\bibitem[Chwialkowski et~al.(2015)Chwialkowski, Ramdas, Sejdinovic, and Gretton]{chwialkowski_fast_2015}
K.~P. Chwialkowski, A.~Ramdas, D.~Sejdinovic, and A.~Gretton.
\newblock Fast {Two}-{Sample} {Testing} with {Analytic} {Representations} of {Probability} {Measures}.
\newblock In \emph{Advances in {Neural} {Information} {Processing} {Systems}}, volume~28. Curran Associates, Inc., 2015.
\newblock URL \url{https://papers.nips.cc/paper_files/paper/2015/hash/b571ecea16a9824023ee1af16897a582-Abstract.html}.

\bibitem[Clark and Jaini(2023)]{clark_text--image_2023}
K.~Clark and P.~Jaini.
\newblock Text-to-{Image} {Diffusion} {Models} are {Zero}-{Shot} {Classifiers}, Sept. 2023.
\newblock URL \url{http://arxiv.org/abs/2303.15233}.
\newblock arXiv:2303.15233 [cs.CV].

\bibitem[English and Suzuki(2026)]{english_path_2026}
E.~English and T.~Suzuki.
\newblock Path {Invariance} and the {Robustness} of {Flow} {Matching}: {Beyond} {Architectural} and {Data} {Perturbations}.
\newblock Mar. 2026.
\newblock URL \url{https://openreview.net/forum?id=7XH9TUEzCI}.

\bibitem[English et~al.(2026)English, Huang, and Suzuki]{english_is_2026}
E.~English, W.~Huang, and T.~Suzuki.
\newblock Is your {Flow} {Matching} {Model} {Really} {Generalising}? {A} {Path}-{Length} {Diagnostic}.
\newblock May 2026.
\newblock URL \url{https://openreview.net/forum?id=mG9tp027lG}.

\bibitem[Farghly et~al.(2025)Farghly, Rebeschini, Deligiannidis, and Doucet]{farghly_implicit_2025}
T.~Farghly, P.~Rebeschini, G.~Deligiannidis, and A.~Doucet.
\newblock Implicit {Regularisation} in {Diffusion} {Models}: {An} {Algorithm}-{Dependent} {Generalisation} {Analysis}.
\newblock Oct. 2025.
\newblock URL \url{https://openreview.net/forum?id=3lSqgESdPu}.

\bibitem[Graham et~al.(2023)Graham, Pinaya, Tudosiu, Nachev, Ourselin, and Cardoso]{graham_denoising_2023}
M.~S. Graham, W.~H.~L. Pinaya, P.-D. Tudosiu, P.~Nachev, S.~Ourselin, and M.~J. Cardoso.
\newblock Denoising diffusion models for out-of-distribution detection, Apr. 2023.
\newblock URL \url{http://arxiv.org/abs/2211.07740}.
\newblock arXiv:2211.07740 [cs.LG].

\bibitem[Grathwohl et~al.(2018)Grathwohl, Chen, Bettencourt, Sutskever, and Duvenaud]{grathwohl_ffjord_2018}
W.~Grathwohl, R.~T.~Q. Chen, J.~Bettencourt, I.~Sutskever, and D.~Duvenaud.
\newblock {FFJORD}: {Free}-{Form} {Continuous} {Dynamics} for {Scalable} {Reversible} {Generative} {Models}.
\newblock Sept. 2018.
\newblock URL \url{https://openreview.net/forum?id=rJxgknCcK7}.

\bibitem[Gretton et~al.(2012{\natexlab{a}})Gretton, Borgwardt, Rasch, Schölkopf, and Smola]{gretton2012kernel}
A.~Gretton, K.~M. Borgwardt, M.~J. Rasch, B.~Schölkopf, and A.~Smola.
\newblock A {Kernel} {Two}-{Sample} {Test}.
\newblock \emph{Journal of Machine Learning Research}, 13\penalty0 (25):\penalty0 723--773, 2012{\natexlab{a}}.
\newblock ISSN 1533-7928.
\newblock URL \url{http://jmlr.org/papers/v13/gretton12a.html}.

\bibitem[Gretton et~al.(2012{\natexlab{b}})Gretton, Sejdinovic, Strathmann, Balakrishnan, Pontil, Fukumizu, and Sriperumbudur]{gretton_optimal_2012}
A.~Gretton, D.~Sejdinovic, H.~Strathmann, S.~Balakrishnan, M.~Pontil, K.~Fukumizu, and B.~K. Sriperumbudur.
\newblock Optimal kernel choice for large-scale two-sample tests.
\newblock In \emph{Advances in {Neural} {Information} {Processing} {Systems}}, volume~25. Curran Associates, Inc., 2012{\natexlab{b}}.
\newblock URL \url{https://proceedings.neurips.cc/paper/2012/hash/dbe272bab69f8e13f14b405e038deb64-Abstract.html}.

\bibitem[Guo et~al.(2004)Guo, Shamai, and Verdu]{guo_mutual_2004}
D.~Guo, S.~Shamai, and S.~Verdu.
\newblock Mutual {Information} and {Minimum} {Mean}-square {Error} in {Gaussian} {Channels}, Dec. 2004.
\newblock URL \url{http://arxiv.org/abs/cs/0412108}.
\newblock arXiv:cs/0412108.

\bibitem[Jitkrittum et~al.(2016)Jitkrittum, Szabó, Chwialkowski, and Gretton]{jitkrittum_interpretable_2016}
W.~Jitkrittum, Z.~Szabó, K.~P. Chwialkowski, and A.~Gretton.
\newblock Interpretable {Distribution} {Features} with {Maximum} {Testing} {Power}.
\newblock In \emph{Advances in {Neural} {Information} {Processing} {Systems}}, volume~29. Curran Associates, Inc., 2016.
\newblock URL \url{https://papers.nips.cc/paper_files/paper/2016/hash/0a09c8844ba8f0936c20bd791130d6b6-Abstract.html}.

\bibitem[Kadkhodaie et~al.(2024)Kadkhodaie, Guth, Simoncelli, and Mallat]{kadkhodaie_generalization_2024}
Z.~Kadkhodaie, F.~Guth, E.~Simoncelli, and S.~Mallat.
\newblock Generalization in diffusion models arises from geometry-adaptive harmonic representations.
\newblock \emph{International Conference on Learning Representations}, 2024:\penalty0 46543--46567, May 2024.
\newblock URL \url{https://proceedings.iclr.cc/paper_files/paper/2024/hash/cbaf319a4712385b5ba8a414808b5713-Abstract-Conference.html}.

\bibitem[Kim et~al.(2020)Kim, Ramdas, Singh, and Wasserman]{kim_classification_2020}
I.~Kim, A.~Ramdas, A.~Singh, and L.~Wasserman.
\newblock Classification accuracy as a proxy for two sample testing, Feb. 2020.
\newblock URL \url{http://arxiv.org/abs/1602.02210}.
\newblock arXiv:1602.02210 [cs.LG].

\bibitem[Kingma et~al.(2021)Kingma, Salimans, Poole, and Ho]{kingma_variational_2021}
D.~P. Kingma, T.~Salimans, B.~Poole, and J.~Ho.
\newblock Variational {Diffusion} {Models}.
\newblock Nov. 2021.
\newblock URL \url{https://openreview.net/forum?id=2LdBqxc1Yv}.

\bibitem[Kirchler et~al.(2020)Kirchler, Khorasani, Kloft, and Lippert]{kirchler_two-sample_2020}
M.~Kirchler, S.~Khorasani, M.~Kloft, and C.~Lippert.
\newblock Two-sample {Testing} {Using} {Deep} {Learning}.
\newblock In \emph{Proceedings of the {Twenty} {Third} {International} {Conference} on {Artificial} {Intelligence} and {Statistics}}, pages 1387--1398. PMLR, June 2020.
\newblock URL \url{https://proceedings.mlr.press/v108/kirchler20a.html}.

\bibitem[Kong et~al.(2023)Kong, Brekelmans, and Steeg]{kong_information-theoretic_2023}
X.~Kong, R.~Brekelmans, and G.~V. Steeg.
\newblock Information-{Theoretic} {Diffusion}, Feb. 2023.
\newblock URL \url{http://arxiv.org/abs/2302.03792}.
\newblock arXiv:2302.03792 [cs.LG].

\bibitem[Kong et~al.(2020)Kong, Ping, Huang, Zhao, and Catanzaro]{kong_diffwave_2020}
Z.~Kong, W.~Ping, J.~Huang, K.~Zhao, and B.~Catanzaro.
\newblock {DiffWave}: {A} {Versatile} {Diffusion} {Model} for {Audio} {Synthesis}.
\newblock Oct. 2020.
\newblock URL \url{https://openreview.net/forum?id=a-xFK8Ymz5J}.

\bibitem[Kübler et~al.(2020)Kübler, Jitkrittum, Schölkopf, and Muandet]{kubler_learning_2020}
J.~Kübler, W.~Jitkrittum, B.~Schölkopf, and K.~Muandet.
\newblock Learning {Kernel} {Tests} {Without} {Data} {Splitting}.
\newblock In \emph{Advances in {Neural} {Information} {Processing} {Systems}}, volume~33, pages 6245--6255. Curran Associates, Inc., 2020.
\newblock URL \url{https://proceedings.neurips.cc/paper/2020/hash/44f683a84163b3523afe57c2e008bc8c-Abstract.html}.

\bibitem[Kübler et~al.(2022{\natexlab{a}})Kübler, Jitkrittum, Schölkopf, and Muandet]{kubler_witness_2022}
J.~M. Kübler, W.~Jitkrittum, B.~Schölkopf, and K.~Muandet.
\newblock A {Witness} {Two}-{Sample} {Test}, Feb. 2022{\natexlab{a}}.
\newblock URL \url{http://arxiv.org/abs/2102.05573}.
\newblock arXiv:2102.05573 [cs.LG].

\bibitem[Kübler et~al.(2022{\natexlab{b}})Kübler, Stimper, Buchholz, Muandet, and Schölkopf]{kubler2022autotst}
J.~M. Kübler, V.~Stimper, S.~Buchholz, K.~Muandet, and B.~Schölkopf.
\newblock {AutoML} {Two}-{Sample} {Test}.
\newblock Oct. 2022{\natexlab{b}}.
\newblock URL \url{https://openreview.net/forum?id=JLWOTZpWZzY}.

\bibitem[LeCun et~al.(1998)LeCun, Cortes, and Burges]{lecun1998mnist}
Y.~LeCun, C.~Cortes, and C.~J. Burges.
\newblock The mnist database of handwritten digits.
\newblock \url{http://yann.lecun.com/exdb/mnist/}, 1998.

\bibitem[Lehmann and Romano(2005)]{lehmann_invariance_2005}
E.~L. Lehmann and J.~P. Romano.
\newblock Invariance.
\newblock In \emph{Testing {Statistical} {Hypotheses}}, pages 212--276. Springer, New York, NY, 2005.
\newblock ISBN 9780387276052.
\newblock \doi{10.1007/0-387-27605-X_6}.
\newblock URL \url{https://doi.org/10.1007/0-387-27605-X_6}.

\bibitem[Li et~al.(2023)Li, Prabhudesai, Duggal, Brown, and Pathak]{li_your_2023}
A.~C. Li, M.~Prabhudesai, S.~Duggal, E.~Brown, and D.~Pathak.
\newblock Your {Diffusion} {Model} is {Secretly} a {Zero}-{Shot} {Classifier}, Sept. 2023.
\newblock URL \url{http://arxiv.org/abs/2303.16203}.
\newblock arXiv:2303.16203 [cs.LG].

\bibitem[Lipman et~al.(2022)Lipman, Chen, Ben-Hamu, Nickel, and Le]{lipman_flow_2022}
Y.~Lipman, R.~T.~Q. Chen, H.~Ben-Hamu, M.~Nickel, and M.~Le.
\newblock Flow {Matching} for {Generative} {Modeling}.
\newblock Sept. 2022.
\newblock URL \url{https://openreview.net/forum?id=PqvMRDCJT9t}.

\bibitem[Liu et~al.(2020)Liu, Xu, Lu, Zhang, Gretton, and Sutherland]{liu2020learning}
F.~Liu, W.~Xu, J.~Lu, G.~Zhang, A.~Gretton, and D.~J. Sutherland.
\newblock Learning deep kernels for non-parametric two-sample tests.
\newblock In H.~D. III and A.~Singh, editors, \emph{Proceedings of the 37th International Conference on Machine Learning}, volume 119 of \emph{Proceedings of Machine Learning Research}, pages 6316--6326. PMLR, 13--18 Jul 2020.
\newblock URL \url{https://proceedings.mlr.press/v119/liu20m.html}.

\bibitem[Liu et~al.(2016)Liu, Lee, and Jordan]{liu_kernelized_2016}
Q.~Liu, J.~Lee, and M.~Jordan.
\newblock A {Kernelized} {Stein} {Discrepancy} for {Goodness}-of-fit {Tests}.
\newblock In \emph{Proceedings of {The} 33rd {International} {Conference} on {Machine} {Learning}}, pages 276--284. PMLR, June 2016.
\newblock URL \url{https://proceedings.mlr.press/v48/liub16.html}.

\bibitem[Liu et~al.(2022)Liu, Gong, and Liu]{liu_flow_2022}
X.~Liu, C.~Gong, and Q.~Liu.
\newblock Flow {Straight} and {Fast}: {Learning} to {Generate} and {Transfer} {Data} with {Rectified} {Flow}.
\newblock Sept. 2022.
\newblock URL \url{https://openreview.net/forum?id=XVjTT1nw5z}.

\bibitem[Liu et~al.(2023)Liu, Zhou, Wang, and Weinberger]{liu_unsupervised_2023}
Z.~Liu, J.~P. Zhou, Y.~Wang, and K.~Q. Weinberger.
\newblock Unsupervised {Out}-of-{Distribution} {Detection} with {Diffusion} {Inpainting}, Aug. 2023.
\newblock URL \url{http://arxiv.org/abs/2302.10326}.
\newblock arXiv:2302.10326 [cs.CV].

\bibitem[Lopez-Paz and Oquab(2017)]{lopezpaz2017revisiting}
D.~Lopez-Paz and M.~Oquab.
\newblock Revisiting {Classifier} {Two}-{Sample} {Tests}.
\newblock Feb. 2017.
\newblock URL \url{https://openreview.net/forum?id=SJkXfE5xx}.

\bibitem[Lyu(2012)]{lyu_interpretation_2012}
S.~Lyu.
\newblock Interpretation and {Generalization} of {Score} {Matching}, May 2012.
\newblock URL \url{http://arxiv.org/abs/1205.2629}.
\newblock arXiv:1205.2629 [cs.LG].

\bibitem[Mlodozeniec et~al.(2025)Mlodozeniec, Eschenhagen, Bae, Immer, Krueger, and Turner]{mlodozeniec_influence_2025}
B.~Mlodozeniec, R.~Eschenhagen, J.~Bae, A.~Immer, D.~Krueger, and R.~Turner.
\newblock Influence {Functions} for {Scalable} {Data} {Attribution} in {Diffusion} {Models}, May 2025.
\newblock URL \url{http://arxiv.org/abs/2410.13850}.
\newblock arXiv:2410.13850 [cs.LG].

\bibitem[Moushegian et~al.(2025)Moushegian, Banerjee, and Tarokh]{moushegian_diffusion-based_2025}
S.~Moushegian, T.~Banerjee, and V.~Tarokh.
\newblock Diffusion-{Based} {Hypothesis} {Testing} and {Change}-{Point} {Detection}, June 2025.
\newblock URL \url{http://arxiv.org/abs/2506.16089}.
\newblock arXiv:2506.16089 [stat.ML].

\bibitem[Pidstrigach(2022)]{pidstrigach_score-based_2022}
J.~Pidstrigach.
\newblock Score-{Based} {Generative} {Models} {Detect} {Manifolds}.
\newblock Oct. 2022.
\newblock URL \url{https://openreview.net/forum?id=AiNrnIrDfD9}.

\bibitem[Rabanser et~al.(2019)Rabanser, Günnemann, and Lipton]{rabanser_failing_2019}
S.~Rabanser, S.~Günnemann, and Z.~Lipton.
\newblock Failing {Loudly}: {An} {Empirical} {Study} of {Methods} for {Detecting} {Dataset} {Shift}.
\newblock In \emph{Advances in {Neural} {Information} {Processing} {Systems}}, volume~32. Curran Associates, Inc., 2019.
\newblock URL \url{https://proceedings.neurips.cc/paper_files/paper/2019/hash/846c260d715e5b854ffad5f70a516c88-Abstract.html}.

\bibitem[Ramdas et~al.(2015)Ramdas, Reddi, Poczos, Singh, and Wasserman]{ramdas_decreasing_2015}
A.~Ramdas, S.~J. Reddi, B.~Poczos, A.~Singh, and L.~Wasserman.
\newblock On the {Decreasing} {Power} of {Kernel} and {Distance} {Based} {Nonparametric} {Hypothesis} {Tests} in {High} {Dimensions}.
\newblock \emph{Proceedings of the AAAI Conference on Artificial Intelligence}, 29\penalty0 (1), Mar. 2015.
\newblock ISSN 2374-3468.
\newblock \doi{10.1609/aaai.v29i1.9692}.
\newblock URL \url{https://ojs.aaai.org/index.php/AAAI/article/view/9692}.

\bibitem[Rhodes et~al.(2020)Rhodes, Xu, and Gutmann]{rhodes_telescoping_2020}
B.~Rhodes, K.~Xu, and M.~Gutmann.
\newblock Telescoping {Density}-{Ratio} {Estimation}.
\newblock In \emph{Advances in {Neural} {Information} {Processing} {Systems}}, volume~33, pages 4905--4916. Curran Associates, Inc., 2020.
\newblock URL \url{https://proceedings.neurips.cc/paper/2020/hash/33d3b157ddc0896addfb22fa2a519097-Abstract.html}.

\bibitem[Rombach et~al.(2022)Rombach, Blattmann, Lorenz, Esser, and Ommer]{rombach_high-resolution_2022}
R.~Rombach, A.~Blattmann, D.~Lorenz, P.~Esser, and B.~Ommer.
\newblock High-{Resolution} {Image} {Synthesis} {With} {Latent} {Diffusion} {Models}.
\newblock pages 10684--10695, 2022.
\newblock URL \url{https://openaccess.thecvf.com/content/CVPR2022/html/Rombach_High-Resolution_Image_Synthesis_With_Latent_Diffusion_Models_CVPR_2022_paper.html}.

\bibitem[Salimans et~al.(2016)Salimans, Goodfellow, Zaremba, Cheung, Radford, Chen, and Chen]{salimans_improved_2016}
T.~Salimans, I.~Goodfellow, W.~Zaremba, V.~Cheung, A.~Radford, X.~Chen, and X.~Chen.
\newblock Improved {Techniques} for {Training} {GANs}.
\newblock In \emph{Advances in {Neural} {Information} {Processing} {Systems}}, volume~29. Curran Associates, Inc., 2016.
\newblock URL \url{https://proceedings.neurips.cc/paper_files/paper/2016/hash/8a3363abe792db2d8761d6403605aeb7-Abstract.html}.

\bibitem[Schrab et~al.(2023)Schrab, Kim, Albert, Laurent, Guedj, and Gretton]{jmlr_mmdagg}
A.~Schrab, I.~Kim, M.~Albert, B.~Laurent, B.~Guedj, and A.~Gretton.
\newblock Mmd aggregated two-sample test.
\newblock \emph{Journal of Machine Learning Research}, 24\penalty0 (194):\penalty0 1--81, 2023.
\newblock URL \url{http://jmlr.org/papers/v24/21-1289.html}.

\bibitem[Shlenskii et~al.(2026)Shlenskii, Gushchin, Novitskiy, Dylov, and Korotin]{shlenskii_midpoint_2026}
D.~Shlenskii, N.~Gushchin, L.~Novitskiy, D.~V. Dylov, and A.~Korotin.
\newblock Midpoint {Generative} {Models}, June 2026.
\newblock URL \url{http://arxiv.org/abs/2605.29920}.
\newblock arXiv:2605.29920 [cs.LG].

\bibitem[Song et~al.(2020)Song, Sohl-Dickstein, Kingma, Kumar, Ermon, and Poole]{song_score-based_2020}
Y.~Song, J.~Sohl-Dickstein, D.~P. Kingma, A.~Kumar, S.~Ermon, and B.~Poole.
\newblock Score-{Based} {Generative} {Modeling} through {Stochastic} {Differential} {Equations}.
\newblock Oct. 2020.
\newblock URL \url{https://openreview.net/forum?id=PxTIG12RRHS&utm_campaign=NLP%20News&utm_medium=email&utm_source=Revue%20newsletter}.

\bibitem[Song et~al.(2021)Song, Durkan, Murray, and Ermon]{song_maximum_2021}
Y.~Song, C.~Durkan, I.~Murray, and S.~Ermon.
\newblock Maximum {Likelihood} {Training} of {Score}-{Based} {Diffusion} {Models}.
\newblock In \emph{Advances in {Neural} {Information} {Processing} {Systems}}, volume~34, pages 1415--1428. Curran Associates, Inc., 2021.
\newblock URL \url{https://proceedings.neurips.cc/paper/2021/hash/0a9fdbb17feb6ccb7ec405cfb85222c4-Abstract.html}.

\bibitem[Sutherland et~al.(2021)Sutherland, Tung, Strathmann, De, Ramdas, Smola, and Gretton]{sutherland_generative_2021}
D.~J. Sutherland, H.-Y. Tung, H.~Strathmann, S.~De, A.~Ramdas, A.~Smola, and A.~Gretton.
\newblock Generative {Models} and {Model} {Criticism} via {Optimized} {Maximum} {Mean} {Discrepancy}.
\newblock Mar. 2021.
\newblock URL \url{https://openreview.net/forum?id=HJWHIKqgl}.

\bibitem[Verdú(2010)]{verdu_mismatched_2010}
S.~Verdú.
\newblock Mismatched {Estimation} and {Relative} {Entropy}.
\newblock \emph{IEEE Transactions on Information Theory}, 56\penalty0 (8):\penalty0 3712--3720, Aug. 2010.
\newblock ISSN 1557-9654.
\newblock \doi{10.1109/TIT.2010.2050800}.
\newblock URL \url{https://ieeexplore.ieee.org/document/5508632}.

\bibitem[Wang et~al.(2026)Wang, Wang, Liu, and Suzuki]{wang_zero-flow_2026}
Y.~Wang, L.~Wang, S.~Liu, and T.~Suzuki.
\newblock Zero-{Flow} {Two}-{Sample} {Tests}, July 2026.
\newblock URL \url{http://arxiv.org/abs/2607.21542}.
\newblock arXiv:2607.21542 [cs.LG] version: 1.

\bibitem[Watson et~al.(2023)Watson, Juergens, Bennett, Trippe, Yim, Eisenach, Ahern, Borst, Ragotte, Milles, Wicky, Hanikel, Pellock, Courbet, Sheffler, Wang, Venkatesh, Sappington, Torres, Lauko, De~Bortoli, Mathieu, Ovchinnikov, Barzilay, Jaakkola, DiMaio, Baek, and Baker]{watson_novo_2023}
J.~L. Watson, D.~Juergens, N.~R. Bennett, B.~L. Trippe, J.~Yim, H.~E. Eisenach, W.~Ahern, A.~J. Borst, R.~J. Ragotte, L.~F. Milles, B.~I.~M. Wicky, N.~Hanikel, S.~J. Pellock, A.~Courbet, W.~Sheffler, J.~Wang, P.~Venkatesh, I.~Sappington, S.~V. Torres, A.~Lauko, V.~De~Bortoli, E.~Mathieu, S.~Ovchinnikov, R.~Barzilay, T.~S. Jaakkola, F.~DiMaio, M.~Baek, and D.~Baker.
\newblock De novo design of protein structure and function with {RFdiffusion}.
\newblock \emph{Nature}, 620\penalty0 (7976):\penalty0 1089--1100, Aug. 2023.
\newblock ISSN 1476-4687.
\newblock \doi{10.1038/s41586-023-06415-8}.

\bibitem[Xiao et~al.(2017)Xiao, Rasul, and Vollgraf]{xiao2017_fashionmnist}
H.~Xiao, K.~Rasul, and R.~Vollgraf.
\newblock Fashion-mnist: a novel image dataset for benchmarking machine learning algorithms, 2017.

\bibitem[Yang et~al.(2023)Yang, Shi, Wei, Liu, Zhao, Ke, Pfister, and Ni]{medmnistv2}
J.~Yang, R.~Shi, D.~Wei, Z.~Liu, L.~Zhao, B.~Ke, H.~Pfister, and B.~Ni.
\newblock Medmnist v2-a large-scale lightweight benchmark for 2d and 3d biomedical image classification.
\newblock \emph{Scientific Data}, 10\penalty0 (1):\penalty0 41, 2023.

\bibitem[Zhang et~al.(2024)Zhang, Zhou, Lu, Guo, Wang, Shen, and Qu]{zhang_emergence_2024}
H.~Zhang, J.~Zhou, Y.~Lu, M.~Guo, P.~Wang, L.~Shen, and Q.~Qu.
\newblock The {Emergence} of {Reproducibility} and {Consistency} in {Diffusion} {Models}.
\newblock In \emph{Proceedings of the 41st {International} {Conference} on {Machine} {Learning}}, pages 60558--60590. PMLR, July 2024.
\newblock URL \url{https://proceedings.mlr.press/v235/zhang24cn.html}.

\end{thebibliography}

\end{document}